\documentclass[11pt,letter]{article}
\usepackage[margin = 2.54cm]{geometry}

\usepackage[numbers,compress]{natbib}

\setcitestyle{numbers}

\usepackage[utf8]{inputenc} %
\usepackage[T1]{fontenc}    %

\usepackage{url}            %
\usepackage{booktabs}       %
\usepackage{amsfonts}       %
\usepackage{nicefrac}       %
\usepackage{microtype}      %
\usepackage{xcolor}         %
\usepackage{subfigure}

\usepackage{amsmath,amssymb,amsthm}
\usepackage[linesnumbered,vlined,ruled]{algorithm2e}

\usepackage{tikz}
\usetikzlibrary{trees}
\usetikzlibrary{positioning}
\usepackage{tikzscale}

\usepackage{graphicx} %

\usepackage{wrapfig}
\usepackage{graphicx}
\usepackage{multirow}
\usepackage{makecell}
\usepackage{mathrsfs,ifpdf}

\usepackage{indentfirst,relsize}
\usepackage{setspace}
\usepackage{enumerate}
\usepackage{latexsym}
\usepackage{stackrel}
\usepackage[all]{xy}
\usepackage[usenames,dvipsnames]{pstricks}
\usepackage{pst-grad} %
\usepackage{pst-plot} %
\usepackage{xspace}
\usepackage{bbm}

\usepackage[super]{nth}

\usepackage{tablefootnote}

\usepackage{hyperref} %
\usepackage{cleveref} %

\newcommand{\E}{\mathop{\mathbb{E}}}
\newcommand{\remove}[1]{}

\newcommand{\R}{\mathbb{R}}
\newcommand{\N}{\mathbb{N}}

\newcommand{\cS}{\mathcal{S}}
\newcommand{\cT}{\mathcal{T}}
\newcommand{\cL}{\mathcal{L}}
\newcommand{\cA}{\mathcal{A}}

\newcommand{\cC}{\mathcal{C}}
\newcommand{\cD}{\mathcal{D}}
\newcommand{\cE}{\mathcal{E}}

\newcommand{\cP}{\mathcal{P}}
\newcommand{\cN}{\mathcal{N}}
\newcommand{\cO}{\mathcal{O}}

\newcommand{\cG}{\mathcal{G}}

\newcommand{\cX}{\mathcal{X}}

\newcommand{\eps}{\varepsilon}

\newcommand{\wt}{\texttt{wt}}

\newcommand{\wh}{\widehat}
\newcommand{\KL}{\mathrm{KL}}
\newcommand{\TV}{\mathsf{d_{TV}}}
\newcommand{\Reg}{\mathsf{Reg}}
\newcommand{\EWA}{\textsf{EWA}}
\newcommand{\RWM}{\textsf{RWM}}
\newcommand{\Clip}{\mathsf{Clip}}
\newcommand{\Bucket}{\mathsf{Bucket}}

\newcommand{\pa}{\mathsf{pa}} %

\newcommand{\ignore}[1]{}

\newtheorem{theo}{Theorem}[section]
\newtheorem{lem}[theo]{Lemma}

\newtheorem{coro}[theo]{Corollary}

\newtheorem{cl}[theo]{Claim}

\theoremstyle{definition}
\newtheorem{defi}[theo]{Definition}
\newtheorem{rem}[theo]{Remark}

\newtheorem{fact}[theo]{Fact}

\newcounter{mynotes}
\newcommand{\nbr}{\mathsf{Nbr}}

\newcommand{\kl}{\mathsf{D_{KL}}}
\newcommand{\paren}[1]{\left( #1 \right)}
\newcommand{\bracket}[1]{\left[ #1 \right]}

\newcommand{\abs}[1]{\left| #1 \right|}
\newcommand{\sampchordal}{$\mathsf{SamplingChordalDist}$ }

\newcommand{\countchordal}{$\mathsf{CountChordalDist}$ }

\newcommand{\learnchordal}{$\mathsf{LearnChordalDist}$ }

\newcommand{\ao}{\mathsf{AO}}

\newcommand{\tg}{\mathcal{T}^G}

\newcommand{\lnk}{\mathsf{Link}}

\newcommand{\sep}{\mathsf{Sep}}

\newcommand{\stk}{\mathsf{Stack}}

\newcommand{\ldag}{\mathcal{DAG}^{(\ell)}}
\newcommand{\tab}{\mathsf{Table}}

\newcommand{\hel}{\mathsf{H}}

\usepackage[makeroom]{cancel}
\newcommand{\adds}[1]{{#1}}
\newcommand{\dels}[1]{}
\newcommand{\aodist}[2]{\ensuremath P^{+}_{#1,#2}} %
\newcommand{\aonodedist}[3]{\ensuremath P^{+}_{#1,#2,#3}} %
\newcommand{\aoestdsamp}[2]{\ensuremath \widehat{P}^{+}_{#1,#2}} %
\newcommand\numrelop[2]%
{\stackrel{\scriptscriptstyle(\mkern-1.5mu#1\mkern-1.5mu)}{#2}}

\newenvironment{ack}
    {\section*{Acknowledgments}
    }
    {%
    }

\title{Distribution Learning Meets Graph Structure Sampling}

\author{%
 Arnab Bhattacharyya\\
 University of Warwick\\
 \texttt{arnab.bhattacharyya@warwick.ac.uk}
 \and
 Sutanu Gayen\\
 IIT Kanpur\\
 \texttt{sutanugayen@gmail.com}
 \and
 Philips {George John}\\
 CNRS@CREATE \& Dept of Computer Science\\
 National University of Singapore\\
 \texttt{philips.george.john@u.nus.edu}
 \and
 Sayantan Sen\\
 Centre for Quantum Technologies\\
 National University of Singapore\\
 \texttt{sayantan789@gmail.com}
 \and
 N. V. Vinodchandran\\
 University of Nebraska-Lincoln\\
 \texttt{vinod@cse.unl.edu}
}

\date{}

\begin{document}

\maketitle

\begin{abstract}
    This work establishes a novel link between the problem of PAC-learning high-dimensional graphical models and the task of (efficient) counting and sampling of graph structures, using an online learning framework. The problem of efficiently counting and sampling graphical structures, such as spanning trees and acyclic orientations, has been a vibrant area of research in algorithms. We show that this rich algorithmic foundation can be leveraged to develop new algorithms for learning high-dimensional graphical models.

We present the first efficient algorithm for (both realizable and agnostic) learning of Bayes nets with a chordal skeleton. In particular, we present an algorithm that, given integers $k,d > 0$, error parameter $\varepsilon > 0$, an undirected chordal graph $G$ on $n$ vertices, and sample access to a distribution $P^*$ on $[k]^n$; (1) returns a Bayes net $\widehat{P}$ with skeleton $G$ and indegree $d$, whose KL-divergence from $P^*$ is at most $\varepsilon$ more than the optimal KL-divergence between $P^*$ and any Bayes net with skeleton $G$ and indegree $d$, (2) uses $\widetilde{O}(n^3k^{d+1}/\varepsilon^2)$ samples from $P^*$ and runs in time $\mathrm{poly}(n,k,\varepsilon^{-1})$ for constant $d$. Prior results in this spirit were for only for \adds{trees ($d=1$, tree skeleton) via Chow-Liu}, and in the realizable setting for polytrees (arbitrary $d$ but tree skeleton). Thus, our result significantly extends the state-of-the-art in learning Bayes net distributions. We also establish new results for learning tree and polytree distributions.

\end{abstract}

\section{Introduction}

High-dimensional distributions are pivotal in contemporary machine learning, with widespread applications across various domains such as gene regulation networks~\citep{friedman2000using,castelo2006robust,castelo2009reverse,edwards2010selecting}, protein signaling networks~\citep{durbin1998biological,spirin2003protein,thomas2005graphical}, brain connectivity networks~\citep{huang2010learning,varoquaux2010brain}, and psychiatric symptom networks~\citep{boschloo2015network,peralta2020symptom,xie2022conditional}. Probabilistic graphical models provide succinct representations of high-dimensional distributions over an exponentially large sample space such as $\mathbb{R}^n$ or $\{0,1\}^n$. 
These models leverage the limited dependence between component variables, encoded by a dependency graph, to describe joint probability distributions over a large set of variables in a succinct and interpretable manner.
Probabilistic graphical models such as Bayesian networks, Ising models, and Gaussian graphical models are extensively utilized to model a wide range of data generation processes in practice (refer to \cite{lauritzen1996graphical,wainwright2008graphical,koller2009probabilistic} and the references therein). Learning distributions represented by these graphical models is a central challenge with significant theoretical and practical implications.

{\em The present work focuses on learning an unknown Bayesian network from sample data.} A {Bayesian network (Bayes net)} with \( n \) variables and alphabet size \( k \) is a probability distribution over \([k]^n\) defined by a directed acyclic graph (DAG) \( G \) on \( [n] \). Each node represents a random variable, which is conditionally independent of non-descendants given its parents. By {Bayes rule}, the distribution factorizes into \( n \) conditional probabilities. If \( G \) has in-degree at most \( d \), it requires at most \( n k^{d+1} \) parameters to describe the distribution, significantly reducing the descriptional complexity from \( k^n \) parameters required for an arbitrary distribution.

Learning Bayesian network distributions  involves two steps: structure learning (identifying the dependency graph) and parameter learning (estimating conditional probability tables). Structure learning methods fall into two categories: constraint-based, which iteratively removes edges by testing for conditional independence, and score-based, which assigns scores to DAGs and frames structure recovery as an optimization problem, often solved using heuristics like greedy hill climbing.
The current work broadly fits in the framework of  score-based approach. However, instead of optimizing the score directly, we use the framework of {\em online learning} to reduce the problem to {\em sampling} from a family of high-dimensional structures. %

In the online learning framework, the goal is to design a forecaster that observes a sequence of examples $x^{(1)}, x^{(2)}, \dots, x^{(T)}$, and at each time (or round) $t$, outputs a prediction $\wh{p}_t$ based only on $x^{(1)}, \dots, x^{(t-1)}$. After predicting $\wh{p}_t$, it observes $x^{(t)}$, and it incurs a loss $\ell(\wh{p}_t, x^{(t)})$ for a {\em loss function} $\ell$. The cumulative loss of the forecaster is benchmarked against that of a fixed and known set of {\em experts}. The goal of the algorithm is to minimize the {\em regret}, defined as the difference between the cumulative loss over all rounds and the loss it would incur if it were to follow the best expert. Online learning is a well-established field with a wide range of applications in theoretical computer science, including game theory, approximation algorithms, and complexity theory (see \cite{freund1996game,freund1999adaptive,DBLP:conf/soda/BeagleholeHKLL23,daskalakis2009complexity,rubinstein2017settling,behnezhad2019optimal,kovenock2012coalitional} and the references therein). %

In distribution learning, the experts are all the possible candidate Bayesian networks (up to a sufficient discretization). The observations are random samples from the unknown distribution, and we set the loss function to be the negative log-likelihood $\ell(\wh{p},x) = -\log \wh{p}(x)$. The primary obstacle in applying the online approach to distribution learning lies in ensuring computational efficiency. All the standard forecasting algorithms have running time at least linear in the number of experts. In our case, the experts are all the discretization of all candidate Bayes nets, which is exponentially many.  A key insight of our work is the discovery of the close relationship between this computational challenge and the task of efficient counting and sampling of DAGs from a class of DAGs. This link allows us to transfer techniques and algorithms from the counting and sampling literature to the realm of distribution learning, leading to significant new results in learning Bayes net distributions.

\section{Our Results}
We first set up the framework of PAC-learning \citep{valiant1984theory} for distributions; formal definitions appear in \Cref{sec:prelim}. We use KL-divergence (denoted as $\kl$) as the notion of similarity between probability distributions, and we will  work with distributions on $[k]^n = \{1,\dots, k\}^n$. Let $\cC$ be a class of such distributions; in our applications, $\cC$ will correspond to some family of Bayes nets.

For $\eps>0, A\geq 1$ and two distributions $P$ and $\wh{P}$ over $[k]^n$, we say $\wh{P}$ is an {\em $(\eps, A)$-approximation for $P$ with respect to $\cC$} if $\kl(P \| \wh{P}) \leq  A \cdot \min_{Q \in \cC} \kl(P \| Q) + \eps$. When $A = 1$, we simply say $\wh{P}$ is an {\em $\eps$-approximation} of $P$. An algorithm is said to be an {\em agnostic PAC-learner for $\cC$} if for any $\eps, \delta>0$ and access to i.i.d.~samples from an input distribution $P^*$, it outputs a distribution $\wh{P}$ which is an $\eps$-approximation for $P^*$ with probability at least $1-\delta$. If $\wh{P}$ is not necessarily in $\cC$, the algorithm is called {\em improper}; otherwise, it is called {\em proper}. Also, the {\em realizable setting} corresponds to the case when the input $P^*$ is guaranteed to be in $\cC$.

\begin{table}
\setlength{\tabcolsep}{8pt}
\renewcommand{\arraystretch}{1.5}
    \centering
    \begin{tabular}{llcc}
    \toprule
       & ~ & \makecell{Chordal graph with indegree $\leq d$\\ and known skeleton}  & Tree with unknown skeleton\\   
    \midrule
\multirow{2}{*}{Realizable} & Proper  & $\widetilde{O}\left(\max\left\{\frac{n^3}{\eps^2\delta^2}, \frac{n k^{d+1}}{\eps}\right\}\right)$ & $\widetilde{O}\left(\max\left\{\frac{n^3}{\eps^2\delta^2}, \frac{n k^{2}}{\eps}\right\}\right)$\\
   & Improper  & $\widetilde{O}\left(\frac{n k^{d+1}}{\eps \delta}\right)$ &  $\widetilde{O}\left(\frac{nk^2}{\eps\delta}\right)$\\
    \midrule
\multirow{2}{*}{Agnostic}   & Proper & $\widetilde{O}\left(\max\left\{\frac{n^3}{\eps^2\delta^2}, \frac{n k^{d+1}}{\eps}\right\}\right)$ & $\widetilde{O}\left(\max\left\{\frac{n^3}{\eps^2\delta^2}, \frac{n k^{2}}{\eps}\right\}\right)$\\
   & Improper & $\widetilde{O}\left(\max\left\{\frac{ n^4 }{\eps^{4}}, \frac{n k^{d+1}}{\eps}\right\}\right)$\tablefootnote{There is a $\mathrm{polylog}(1/\delta)$ dependency here (as opposed to $1/\delta^2$ for the proper learner) hidden in $\tilde{O}(\cdot)$.} & $\widetilde{O}\left(\max\left\{\frac{ n^4 }{\eps^{4}}, \frac{n k^{2}}{\eps}\right\}\right)$\\
    \bottomrule
    \end{tabular}
    \caption{Our results: Sample complexities for $(\eps,\delta)$-PAC learning (the $\tilde{O}(\cdot)$ notation hides polylog factors)}
    \label{tab:result}
\end{table}

It is well-known (e.g., \cite{bhattacharyya2023near}) that given a DAG $G$, the minimum KL divergence between $P^*$ and a Bayes net over $G$ can be written as $J_{P^*} - \sum_{i \in V(G)} I(X_i; X_{\pa_G(i)})$, where $X \sim P^*$, $I$ is the mutual information, and $J_{P^*}$ is a constant independent of $G$.  Hence, if $\cC$ is the class of Bayes nets over DAGs of in-degree $d$, a natural strategy for designing agnostic learning for $\cC$ is the following: First approximate the mutual information between any node and any set of $d$ other nodes up to a suitable additive error. Next, maximize the sum of mutual informations between a node and its $d$ parents, over all possible DAGs with in-degree $d$. Iterating over all possible DAG structures would then lead to an algorithm with sample complexity $\widetilde{O}(n^2k^{d+1}\eps^{-2})$. However, this algorithm has exponential time complexity and the sample complexity is also suboptimal compared to known lower bounds.

In this work, we give an improper agnostic learning algorithm for Bayes nets with indegree $d$ with sample complexity $\widetilde{O}(nk^{d+1}\eps^{-1})$, which is sample-optimal upto polylogarithmic factors. The algorithm is computationally inefficient. Our main contribution is the design of sample and computational-efficient algorithms for new natural classes of Bayes nets, extending the state of the art.  In particular, modifying our algorithm for general bounded-indegree Bayes nets, we give computational and time efficient algorithms for learning {\em chordal-structured Bayes nets with a known skeleton}. Efficient algorithms are currently known only for learning tree-structured distributions (\cite{bhattacharyya2023near, DBLP:conf/stoc/DaskalakisP21,DBLP:journals/tit/ChowL68}) and for learning polytree-structured distributions with a given skeleton (\cite{choo2023learning}).   

An undirected graph is chordal if every cycle of length four or more has a chord (an edge connecting two non-adjacent vertices in the cycle). Chordal graphs form a significantly broader class than trees and encompass several well-studied graph families, including interval graphs and $k$-trees. Consequently, our results represent a major advancement in the state of the art for learning Bayesian network distributions. Beyond their structural significance, chordal graphs play a crucial role in the study of Bayesian networks particularly in causal Bayesian networks~\citep{andersson1997characterization,koller2009probabilistic}.
We describe our results next. The sample  complexities of our results are summarized in \Cref{tab:result}.

\paragraph{Learning with Known Chordal Skeleton}
The {\em skeleton} of a DAG refers to its underlying undirected graph. We consider Bayes nets having a known {\em chordal} skeleton with bounded indegree and present an efficient algorithm for learning such distributions.  
\adds{
\begin{theo}\label{thm:intro_chordal}
Let $G$ be an undirected chordal graph on $n$ nodes, and suppose  $d$ is a fixed constant. Consider the problem of agnostically learning a distribution w.r.t the class of Bayes nets having skeleton $G$ with indegree $\leq d$. 
There exist (i) an agnostic improper PAC-learner for this problem using {$\widetilde{O}\left(\frac{n^4}{\eps^4} + \frac{nk^{d+1}}{\eps}\right)$ samples} that returns an efficiently-samplable mixture of such Bayes nets, and 
(ii) an agnostic proper PAC-learner using $\widetilde{O}\left(\frac{n^3}{\eps^2 \delta^2} + \frac{nk^{d+1}}{\eps}\right)$ samples that returns a single Bayes net. Both algorithms have $\mathrm{poly}(n,k,1/\eps,1/\delta)$ running time.
\end{theo}
}
This is the first result yielding efficient algorithms for agnostic learning Bayes nets on {\em non-tree} skeletons without further distributional assumptions; see Section \ref{sec:related} for discussion of previous work.

For efficiently learning chordal and polytree distributions, we need to know the correct skeleton (underlying undirected graph). To the best of our knowledge, there is currently no computational hardness result regarding this. Additionally, there have been several works with the known skeleton assumption, even in the context beyond PAC distribution learning. \cite{sagar24} designed an FPT algorithm (in terms of total degree and treewidth) for counting the Markov Equivalence Classes with a given skeleton. On the practical side, several works for Bayes net structure learning first learn a skeleton from the data and then fix the orientations (e.g., see the survey \cite{daly2011surveybayes}, section 4.9.1). However, the approach of first learning the skeleton and then performing the distribution learning does not have sound theoretical guarantees, since the distance measures in these two contexts are different.

A well-investigated class of Bayes nets is the class of {\em polytree} distributions: whose DAGs have tree (acyclic) skeletons. %
Polytrees are especially interesting because they admit fast exact inference \citep{Pea14}. \cite{dasgupta1999learning} investigated the problem of learning polytree distributions in terms of the negative log-likelihood cost, and showed that it is NP-hard to get a 2-polytree (where indegree is $\leq 2$) whose cost is at most $c$ times that of the optimal 2-polytree for some constant $c>1$, even if we have oracle access to the true entropies (equivalently, infinite samples). %
Our distribution learning algorithms in contrast achieve an {\em additive} approximation in the reverse-KL cost.
As a direct corollary of the above theorem, we have the following result for bounded indegree polytrees with known skeleton.
\adds{
\begin{coro}
Let $d>0$ be a fixed constant and $G$ be a given undirected tree. Consider the problem of agnostically learning a distribution w.r.t the class of Bayes nets having skeleton $G$ with indegree $\leq d$, i.e. $d$-polytrees with skeleton $G$. There exist (i) an agnostic improper PAC-learner for this problem using {$\widetilde{O}\left(\frac{n^4}{\eps^4} + \frac{nk^{d+1}}{\eps}\right)$ samples} that returns an efficiently-samplable mixture of such polytrees, and 
(ii) an agnostic proper PAC-learner using $\widetilde{O}\left(\frac{n^3}{\eps^2 \delta^2} + \frac{nk^{d+1}}{\eps}\right)$ samples that returns a single polytree. Both algorithms have $\mathrm{poly}(n,k,1/\eps,1/\delta)$ running time.
\end{coro}
}

The closest related result is that of \cite{choo2023learning} who designed a PAC-learner in the realizable setting for polytrees with optimal\footnote{Although not stated in the corollary above, in the realizable setting, our techniques also yield sample complexity with the optimal dependence on $n, k, $ and $\varepsilon$.} sample complexity $\widetilde{O}(nk^{d+1}\eps^{-1}\log \delta^{-1})$. However, their analysis crucially uses the realizability assumption, and it was left as an open question in that work to find an efficient agnostic learner for polytrees. The above corollary answers this question.

{
\begin{rem}
We can bound the running time of our learning algorithms for chordal-structured Bayes nets (with known skeleton) as follows: At the outset, for every node $i \in [n]$ and for every choice of the $\leq d$ parents $\pa(i)$, we learn the conditional distribution associated with node $i$ given that choice of parents $\pa(i)$. These are add-one conditional distributions computed from a sufficiently-large ($\widetilde{O}(nk^d/\eps)$) set of samples. Subsequently, the learning algorithm focuses only on the combinatorial problem of learning an acyclic orientation. The running time contribution from the node-distribution learning part is $\widetilde{O}((\Delta k)^d d n^2/\varepsilon)$, where $\Delta$ is the maximum (undirected) degree of the skeleton. Here, $n {\Delta \choose d} \leq n \Delta^d$ (for $d \ll n$) bounds the number of all possible (node, parent-set) pairings and $\widetilde{O}(d n k^d/\varepsilon)$ is the time for computing a ``good'' add-one conditional distribution for a given node and parent-set. Note that the runtime is polynomial even if both $\Delta$ and $d$ are $O(\log n)$. If $d$ is unbounded, then the runtime can grow at an exponential $2^{d \log (\Delta k)}$ rate. Note that, for unbounded $d$, an exponential dependence on the runtime and sample complexity is inevitable since chordal-structured indegree-$(n-1)$ Bayes nets with a fixed skeleton (the complete graph on $[n]$) can capture arbitrary $n$-dimensional distributions (we do not use faithfulness or similar assumptions for distribution learning).
\end{rem}
}

\paragraph{Learning Tree-structured Distributions}
By {\em tree-structured distributions} (or simply, trees when the meaning is clear), we mean Bayes nets whose underlying DAG has in-degree 1. They can be equivalently defined as undirected Markov models over (undirected) trees. The celebrated work of  \cite{DBLP:journals/tit/ChowL68} developed a polynomial time algorithm for learning tree-structured distributions, if the algorithm is provided the exact mutual information between each pair of variables. PAC-learning guarantees with sample complexity bounds came later \citep{DBLP:conf/stoc/DaskalakisP21, bhattacharyya2023near},  In particular, the highlight of these works was establishing that in the realizable setting, i.e., when the samples are being generated by a tree-structured distribution on $[k]^n$, the Chow-Liu algorithm is a PAC-learner with sample complexity $\widetilde{O}(nk^3/\eps)$. While the dependence on $n$ and $\eps$ is tight, it was left as an open question whether a better dependence on $k$ is possible.

Our work answers this in the affirmative:
\begin{theo}\label{thm:intro_tree}
Let $\cC$ be the family of tree-structured distributions over $[k]^n$. There exists an algorithm that for any $\eps>0$, given sample access to a distribution $P^* \in \cC$, returns an $\eps$-approximation $\wh{P}$ of $P^*$ with probability at least $2/3$ and uses $m=\widetilde{O}(nk^2\eps^{-1})$ samples and $\mathrm{poly}(m)$ running time. The distribution $\wh{P}$ is a mixture of distributions from $\mathcal{C}$ and is samplable in polynomial time. 
\end{theo}
Note that in contrast to Theorem \ref{thm:intro_chordal}, here, the algorithm does not know the true skeleton a priori.
The output distribution $\wh{P}$ is a mixture of exponentially many trees but can nevertheless be sampled in polynomial time by using the matrix-tree theorem, as we explain later. {We note that the dependence of $k^2$ on the sample complexity is tight. This follows from \cite[Theorem 13]{BCD20} (see also \cite{devroye2001density}), which proves that learning a Bayes net with in-degree $d$ requires $\Omega(n k^{d+1})$ samples, and for tree Bayes nets, the in-degree being $d=1$, requiring $\Omega(n k^2)$ samples}. Learning with mixtures of trees has been studied before  (\cite{meila2000learning, kumar2009learning, anandkumar2012learning,selvam2023mixtures}) but in other contexts.

We also note that going beyond trees, the same approach allows us to develop polynomial sample and time algorithms for learning Bayes nets on an unknown DAG whose moralization is promised to have constant vertex cover size. Here, instead of sampling using the matrix-tree theorem, we can utilize a recent result of \cite{harviainen2023revisiting} to sample such DAGs. Details appear in \Cref{sec:learnbayesnetboundedvc}.

\paragraph{Why KL divergence?} 
We briefly discuss why learning in KL divergence is relevant. The study of agnostic learning of distributions in KL divergence goes back to at least three decades ago by the works of \cite{abe1991polynomial} and \cite{dasgupta1997sample,dasgupta1999learning}. The authors argued that given random samples from an unknown distribution $P$, minimizing the KL divergence is the same as maximizing the log-likelihood in expectation, due to the following equation: %
$\kl(P||Q)=-H(P)-\E_{x\sim P} [\log Q(x)],$
where $H(P)$ is the entropy of $P$. KL divergence also appears in the study of density estimation such as Yang-Barron's construction and covering complexity (\cite{yang1999information,wu2017lecture}). \cite{DBLP:conf/stoc/KearnsMRRSS94} also studied the complexity of distribution learning in terms of KL divergence and gave a coding-theoretic interpretation to choosing KL divergence as the distance function. %
Finally, several recent works have investigated the problem of learning high-dimensional distributions in this stronger KL divergence guarantee, such as~\cite{bhattacharyya2023near,daskalakis2019testing,choo2023learning,bhattacharyya2022learning,wang2024optimal}.

\section{Our Techniques}

\paragraph{Online Learning Framework to Learning in reverse KL} Given i.i.d.~samples from a distribution $P^*$, we are trying to learn it. Roughly, for a random $x\sim P^*$, a good approximate Bayes net $P$ should maximize the probability $P(x)$, or equivalently, minimize the expected log-likelihood $\mathbb{E}_{x\sim P^*} \bracket{\log \frac{1}{P(x)}}$. Keeping this in mind, we define the loss of any Bayes net $\wh{P}$ predicted by the algorithm to be $\log \frac{1}{\wh{P}(x)}$ for a sample $x$.

We follow the online learning framework. Here the algorithm $\cA$ observes a set of samples $x^{(1)},\dots, x^{(T)}\sim P^*$ over $T$ rounds from an unknown Bayes net $P^*$. The goal of the algorithm is to learn a distribution $\wh{P}$ which is close to $P^*$. After observing each sample $x^{(t)}$, $\cA$ predicts a Bayes net $\wh{P}_t$ and incurs a loss of $\log \frac{1}{\wh{P}_t(x^{(t)})}$ for this round. However, there is a set of experts $\cE$ to help $\cA$. For simplicity, we can assume each expert $E \in \cE$ is simply one Bayes net among all possible Bayes nets. Had $\cA$ stuck to any particular expert $E\in \cE$, it would incur a total loss $\sum_{t=1}^T \log \frac{1}{E(x^{(t)})}$ over all the $T$ rounds. The algorithm can change the experts in between or do some randomized strategy for choosing the expert among $\cE$. Let $\wh{P}_t$ be its prediction after round $t$. The {regret} is defined to be the difference between the loss of the algorithm and that of the best expert:
 $\sum_{t=1}^T \log \frac{1}{E(x^{(t)})}-\min_{E\in\cE} \sum_{t=1}^T \log \frac{1}{E(x^{(t)})}$.

In our setting, the expert set consists of one Bayes net per DAG from the class of DAGs under consideration (e.g., acyclic orientations of a given skeleton). To associate a Bayes net with a DAG, we approximately learn the conditional distributions at each node using the {\em add-one } or {\em Laplace} estimator on a separate set of samples.  We have the guarantee that these finitely many Bayes nets form a ``cover'' for the class of Bayes nets we wish to learn.

We relate the regret mentioned above to the expected average of the  KL divergence over the rounds: $\mathbb{E}[\frac{1}{T}\kl(P^*||\widehat{P}_t)]$. Once we control the average KL divergence, using the convexity of KL, we can show that the mixture distribution $\frac{1}{T} \sum_{t=1}^T \widehat{P}_t$ is also close to $P^*$. Finally, we translate the above bounds from expectation to high probability using McDiarmid's bounded difference inequality.

We use known bounds on the regret for two classic online learning algorithms: the Exponential Weighted Average (\EWA) algorithm and the Randomized Weighted Majority (\RWM) algorithm. \EWA\ returns us the mixture $\frac{1}{T} \sum_{t=1}^T \widehat{P}_t$ which improperly learns $P^*$ in (reverse) KL. \RWM\ returns a random Bayes net $\wh{P}$ which properly learns $P^*$ in expected KL. The pseudocode for these algorithms is given in \Cref{alg:meta_improper_intro} and \Cref{alg:meta_proper_intro}.

\begin{minipage}{0.42\textwidth}
\begin{algorithm}[H]
\setcounter{AlgoLine}{0}
	\caption{\EWA-based learning for Bayes nets}\label{alg:meta_improper_intro}
	\SetKwInOut{Input}{Input}
	\SetKwInOut{Output}{Output}
	\SetKwProg{myfn}{function}{}{}
    \Input{\,\, ${\cN} = \{P_1,\ldots,P_N\}, T$, hyperparameter $\eta > 0$.}
	  \Output{\,\, Sampler for $\widehat{P}$.}
	$w_{i,0} \gets 1 \mbox{ for each } i \in [N].$ \
    
	  \For{$t \gets 1$ to $T$}{%
		  $\mbox{Observe sample } x^{(t)} \sim P^\ast$. \
        
		Update $w_{i,t} \gets w_{i,t-1} \cdot P_i(x^{(t)})^\eta$ for each $i \in [N]$.
	}
    
	\myfn{\rm\textsc{EWA-Sampler}()}{%
        Sample $t \gets [T]$ uniformly, then sample $i \sim [N]$ with probability $\frac{w_{i,t-1}}{\sum_{j \in [N]} w_{j,t-1}}$.\
		
        \KwRet $x \sim P_i$. \
	}
	
    \KwRet \textsc{EWA-Sampler}
    \tcc{This is a sampler for $\wh{P}$.}
\end{algorithm}
\end{minipage}\hfill\begin{minipage}{0.46\textwidth}
\begin{algorithm}[H]
\setcounter{AlgoLine}{0}
	\caption{\RWM-based learning for Bayes nets}\label{alg:meta_proper_intro}
	\SetKwInOut{Input}{Input}
	\SetKwInOut{Output}{Output}
	\Input{\,\, ${\cN} = \{P_1,\ldots,P_N\}, T$, hyperparameter $\eta > 0$.}
	\Output{\,\, $\wh{P} \in \cN$.}
	$w_{i,0} \gets 1 \mbox{ for each } i \in [N].$ \
	
    \For{$t \gets 1$ to $T$}{
        Sample $i_t \mbox{ from } [N]$ with $\Pr(i_t = i) = \frac{w_{i,t-1}}{\sum_{j \in [N]} w_{j,t-1}}$. \
		
        Observe sample $x^{(t)} \sim P^\ast$. \
		
        \For{$i \in [N]$}{
			$w_{i,t} \gets w_{i,t-1} \cdot P_i(x^{(t)})^\eta$.
		}
	}
    
    Sample $t$ uniformly from $[T]$. \
	
    \KwRet $\wh{P} \gets P_{i_t}$. \
\end{algorithm}
\end{minipage}

\textbf{Efficient Learning of Restricted Classes of Bayes Nets} Our learning algorithm for Bayes nets mentioned above is sample-optimal but not time-efficient in general since the number of experts to be maintained is of exponential size. However, we observe that for special cases of Bayes nets, we can efficiently sample from the experts according to the randomized strategy of the algorithm. As a remark, the idea that the computational barrier of \RWM\ or \EWA\ may be side-stepped by developing efficient sampling schemes was also used in a recent work on fast equilibrium computation in structured games \citep{DBLP:conf/soda/BeagleholeHKLL23} and partly motivated our work.

To see the simplest example of this idea, suppose $\cP = \{P_1, \dots, P_N\}$ is a set of  distributions over $[k]$, and let $\cP^{\otimes n} = \cP \times \cP \times \cdots \times \cP$ be a set of product distributions over $[k]^n$. Each element of $\cP^{\otimes n}$ is indexed as $P_{\mathbf{i}}$ for $\mathbf{i}=(i_1,\dots, i_n)$, so that $P_{\mathbf{i}}(x) = \prod_{j=1}^n P_{i_j}(x_j)$.
The size of $\cP^{\otimes n}$ is clearly $N^n$, so it is infeasible to work with it directly. The \RWM\ algorithm maintains a distribution over $\cP^{\otimes n}$, so that the probability that \RWM\ picks $P_{\mathbf{i}}$ for its prediction $\wh{P}_t$  at time $t$ is proportional to $\prod_{s=1}^{t-1} P_{\mathbf{i}}(x^{(s)})^\eta$, where $x^{(s)}$ is the observed sample at time $s$ and $\eta>0$ is a parameter. Therefore:
\[
\Pr_{\RWM}[\wh{P}_t = P_{\mathbf{i}}] = \frac{\prod_{s=1}^{t-1} P_{\mathbf{i}}(x^{(s)})^\eta}{\sum_{\mathbf{i}'} \prod_{s=1}^{t-1} P_{\mathbf{i}'}(x^{(s)})^\eta} = \prod_{j=1}^n \frac{\prod_{s=1}^{t-1} P_{i_j}(x^{(s)}_j)^\eta}{\sum_{i'_j} \prod_{s=1}^{t-1} P_{i'_j}(x^{(s)}_j)^\eta}.
\]
The crucial observation is that \RWM\ maintains a product distribution over product distributions, and so we can sample each $P_{i_j}$ from $\cP$ independently.

When the underlying Bayes net is a tree, i.e. of indegree 1, we show that \RWM\ samples a random rooted spanning arborescence from a weighted complete graph. The probability to output a particular arborescence $A$ is proportional to $\prod_{e \in A} w_e$ where each $w_e$ is a weight that can be explicitly computed in terms of the observed samples and the parameters of the algorithm. It is well-known that the \emph{matrix-tree theorem} (more precisely, \emph{Tutte's theorem}) for counting weighted arborescences can be used for this purpose, and hence, we obtain an alternative to the Chow-Liu algorithm for approximately learning a tree Bayes net efficiently and sample-optimally. 

Next, we generalize our algorithm to polytree-structured Bayes nets where the underlying skeleton is acyclic. Here, we are assuming that the skeleton is given, so that the goal of the algorithm is to learn an acyclic orientation of the skeleton. For simplicity, suppose the skeleton is known to be the path. Given a particular orientation of the edges, we obtain a particular Bayes net structure. Once the structure is fixed, the conditional probability distribution corresponding to each edge parent$\rightarrow$child is set according to the empirical statistics in a separate batch of samples. This will completely specify a Bayes net $P$, which can assign probability $P(x^{(t)})$ for the sample $x^{(t)}$. Therefore, we can also compute the total loss $\ell_P=\sum_{t=1}^T \log P(x^{(t)})^{-1}$. Then, each structure $P$ will be chosen proportional to $e^{-\eta \ell_P}$ in the \RWM\ algorithm.
In order to sample a particular Bayes net among the entire class using RWM, we need to first compute the normalization constant of the \RWM\ sampler's distribution: $Z:=\sum_{P\in \cP} e^{-\eta\ell_P}$ over the class $\cP$ of all (discretized) path Bayes nets. A particular path Bayes net $P$ will be chosen with probability $e^{-\eta \ell_P}/Z$ by the RWM's sampler. We first show how to compute $Z$ efficiently using dynamic programming. 

We now show how to compute $Z$ by induction on the set of vertices of the path. 
Suppose $Z_j$ is the normalization constant obtained by only restricting to the first $j+1$ nodes in the path. That is, if $\cP_j$ is the class of Bayes nets corresponding to all orientations of the undirected path on $j+1$ nodes, then $Z_j = \sum_{P \in \cP_j} e^{-\eta \ell_P}$, where $\ell_P$ only computes the loss based on the first $j+1$ variables. For the induction, we maintain more refined information for each $j$. Let $\cP_{j, \leftarrow}$ and $\cP_{j, \rightarrow}$ be the class of all discretized Bayes nets on $j+1$ variables with a path skeleton and the last edge pointing left and right, respectively. Correspondingly, define $Z_{j, \leftarrow}$ and $Z_{j, \rightarrow}$; clearly, $Z_j = Z_{j, \leftarrow} + Z_{j, \rightarrow}$. Inductively, assume that  $Z_{j, \leftarrow}$ and $Z_{j, \rightarrow}$ are already computed.
We then need to compute $Z_{j+1, \leftarrow}$ and $Z_{j+1, \rightarrow}$.

If the $(j+1)$-th edge orients rightward, then the parents of nodes $1, \dots, j+1$ do not change, while the new node $j+2$ has parent $j+1$. We can accommodate this new edge by simply adding the negative log of the conditional probability due to this new edge to the loss restricted to the first $j+1$ variables. We can compute $Z_{j+1,\rightarrow}=(Z_{j,\leftarrow}+Z_{j,\rightarrow})e^{-\eta \Delta}$, by computing $\Delta=\sum_{t=1}^T \log P(x^{(t)}_{j+2}\mid x^{(t)}_{j+1})^{-1}$.

If the $(j+1)$-th edge orients leftward, the adjustment is slightly trickier as node $j+1$ will get a new parent $j+2$, while the new node $j+2$ has no parent. In that case, we need to first subtract out the previous sum of negative log conditional probabilities at $j+1$. Let us define:
\begin{equation*}
\begin{aligned}
\Delta_1 &=\sum_{t=1}^T \log P(x^{(t)}_{j+1} \mid x^{(t)}_{j})^{-1},\,\,\Delta_2 = \sum_{t=1}^T \log P(x^{(t)}_{j+1})^{-1},\,\,\Delta_3 = \sum_{t=1}^T\log P(x^{(t)}_{j+1} \mid x^{(t)}_{j},x^{(t)}_{j+2})^{-1},\\
\Delta_4 &= \sum_{t=1}^T\log P(x^{(t)}_{j+1} \mid x^{(t)}_{j+2})^{-1}, \Delta_5 =\sum_{t=1}^T \log P(x^{(t)}_{j+2})^{-1}.\\
\end{aligned}
\end{equation*}
If node $j$ is not a parent of node $j+1$, then node $j+1$ contributed $\Delta_2$ loss to $Z_{j, \leftarrow}$ while now it contributes $\Delta_4$ loss to $Z_{j+1, \leftarrow}$. Otherwise, it contributed $\Delta_1$ loss to $Z_{j, \rightarrow}$ while now it contributes $\Delta_3$ loss to $Z_{j+1, \leftarrow}$. The new node $j+2$ contributes $\Delta_5$ loss to $Z_{j+1, \leftarrow}$ independent of what happens to the other variables. Summarizing:
$$Z_{j+1,\leftarrow}=Z_{j,\leftarrow}\cdot e^{-\eta(\Delta_4-\Delta_2+\Delta_5)}+Z_{j,\rightarrow}\cdot e^{-\eta(\Delta_3-\Delta_1+\Delta_5)}.$$ 
It is easy to see that these updates can be performed efficiently using an appropriate dynamic programming table. Once we have computed the total sum $Z=Z_{n,\leftarrow}+Z_{n,\rightarrow}$, sampling a structure according to the sampler's distribution can simply be done by suitably unrolling the DP table.

\begin{figure}[ht!]
\caption{{\small Given a rooted polytree skeleton, for each node $v$, and for each fixed orientation of edges incident to $v$, we maintain the total weight of all consistent orientations of the subtree rooted at $v$. Above, the orientations of edges incident to $B$ and $C$ are fixed. This is needed when computing the weight for the subtree rooted at $A$, since in the first two panels, the in-degree of $C$ change from 1 to 2, while in the second two panels, $C$'s in-degree does not change.}\label{fig:polytree}}
\centering
    \subfigure{ \includegraphics[width=0.2\textwidth]{figures/ex_polytree_orient1.tikz}
    }
    \subfigure{
    \includegraphics[width=0.2\textwidth]{figures/ex_polytree_orient2.tikz}
    }
    \subfigure{
    \includegraphics[width=0.2\textwidth]{figures/ex_polytree_orient3.tikz}
    }
    \subfigure{
    \includegraphics[width=0.2\textwidth]{figures/ex_polytree_orient4.tikz}
    }
\end{figure}

The argument described above extends to learning bounded indegree polytrees and bounded indegree chordal graphs. For polytrees, the idea is illustrated in \Cref{fig:polytree}. For chordal graphs, the algorithm first builds a clique tree decomposition and uses this structure for dynamic programming. The obvious issue with chordal graphs is that some orientations may lead to cycles, unlike the case for polytrees. However, chordal graphs enjoy certain nice property (see \Cref{lem:chordal_bij_indegree}) that allows us to independently perform weighted counting/sampling of acyclic orientations in each subtree of the clique tree.

\textbf{Agnostic Learning via Maximum Likelihood Estimation} An arguably more natural approach to PAC learning in KL divergence is to maximize the empirical log-likelihood (MLE) over a suitably-discretized class of distributions (e.g., see \cite{feldman2008learning}, Theorem 17). The Chow-Liu algorithm for tree distributions can also be viewed through this lens. Note however that, despite a long history of study, Chow-Liu is not known to attain the sample complexity bound in \cref{thm:intro_tree} for learning trees in the realizable setting.

For the problem of learning polytrees and chordal-structured distributions, we can in fact adapt our algorithm to maximize likelihood and thus, get a sample complexity bound which is comparable to our \Cref{thm:bn_learn_proper} (up to log factors) for proper learning in KL. But it does not yield the near-optimal bounds (for constant failure probability) that we get for improper learning (\Cref{thm:bn_learn_imp}) in the realizable case.
The challenge of implementing the Maximum Likelihood (ML) algorithm over an exponential-sized class of distributions is \emph{efficiency} --- a naive approach would take exponential time. The dynamic programming algorithms that we develop for efficient weighted counting and sampling of DAG structures (which we use to implement \EWA \ / \RWM) can also be used to implement MLE efficiently for polytree/chordal-structured distributions given the skeleton. We give an outline of this in \Cref{app:efficient_ml}. %

\section{Related Works}\label{sec:related}
\cite{hoffgen1993learning} gave the first sample complexity bounds for agnostic learning of a
Bayes net with known structure from samples in KL divergence. This work also gave an efficient algorithm for special cases such as trees using the classical Chow-Liu algorithm. Subsequently, \cite{dasgupta1997sample} gave an efficient algorithm for learning an unknown Bayes net (discrete and Gaussian) on a fixed structure. This result was improved to a sample-optimal learning of fixed-structure Bayes nets in \cite{DBLP:journals/corr/abs-2002-05378,bhattacharyya2023near}. 

The general problem of distribution learning of Bayes networks with unknown DAG structure has remained elusive so far. It has not been shown to be NP-hard, although some related problems and specific approaches are NP-hard \citep{DBLP:conf/aistats/Chickering95,chickering2004large, DL97, DBLP:conf/soda/KargerS01}.  Many of the early approaches required \emph{faithfulness}, a condition which permits learning of the Markov equivalence class, e.g. \cite{spirtes1991algorithm, chickering2002optimal, friedman2013learning}. Finite sample complexity of such algorithms has also been studied, e.g. \cite{friedman1996sample}. Specifically for polytrees, \cite{DBLP:journals/corr/abs-1304-2736,DBLP:conf/aaai/GeigerPP90} studied recovery of the DAG for polytrees under the infinite sample regime and in the realizable setting, while \cite{choo2023learning} gave finite sample complexity for this problem. \cite{gao2021efficient} studied the more general problem of learning Bayes nets, and their sufficient conditions simplified in the setting of polytrees.

A notable prior work in the context of the current paper is the work by \cite{abbeel2006learning}, which also explores the improper learning of Bayesian networks with polynomial sample and time complexities. However, our research diverges from theirs in three critical ways: firstly, their study does not offer any proper learning algorithms; secondly, it lacks agnostic learning guarantees; and thirdly, their approach does not achieve optimal sample complexity in the realizable setting. While they do demonstrate a ``graceful degradation'' as inputs deviate from the hypothesis class, this does not equate to a definitive agnostic learning guarantee as provided in our work. On a positive note, their research does attain polynomial sample and time complexities for learning any Bayesian network with a bounded total degree in the realizable setting. It is worth noting that our results for distributions with chordal skeleton are applicable even when the total degree is unbounded, provided that the indegree remains bounded, a scenario where the findings of Abbeel, Koller, and Ng would not be applicable.

\begin{table}[h!]
\setlength{\tabcolsep}{2pt}
\renewcommand{\arraystretch}{1.5}
    \centering
    \begin{tabular}{llccl}
        \toprule
         ~ & Structure & Efficient? & Agnostic? & {Additional assumptions}\\
        \midrule
        \cite{bhattacharyya2023near} & Tree & Yes & Yes & None\\
        \midrule
        \cite{choo2023learning} & Polytree  & Yes  & No  & Known skeleton\\
        \midrule
        \cite{abbeel2006learning} & Bounded total degree~\tablefootnote{This work studies a more general notion of factor graphs.} & Yes  & No  & None\\
        \midrule
        \cite{BCD20} & Bounded in-degree  & No  & No  & None\\
        \midrule
        \multirow{2}{*}{Our results} & Tree & Yes  & Yes  & None\\
        & \makecell[l]{Chordal skeleton, \\bounded in-degree} & Yes  & Yes  & \makecell[l]{Known skeleton}\\
        \bottomrule
    \end{tabular}
    \caption{Comparison with existing works.}
    \label{tab:example}
\end{table}

\paragraph{Online Learning of Structured Distributions} 
The approach of using the online learning framework for distribution learning has been considered in the literature~\cite{catoni1997mixture,yang2000mixing,van2023high}.
These works use \EWA \ algorithm and output the mixture distribution. However, they primarily focus on minimizing the number of samples, and are not computationally efficient in general. Since we are interested in computationally efficient learning of high-dimensional distributions, their approaches do not directly translate to our context.
The closest we get is the \emph{Sparsitron} algorithm by Klivans and Meka~(\cite{DBLP:conf/focs/KlivansM17}) which learns an unknown Ising model from samples. However, Sparsitron is typical to Ising models where the conditional distribution at any component follows a logistic regression model which the Sparsitron algorithm learns.

Although not for distribution learning, a similar use of the multiplicative weights update method appears in Freund and Schapire's well-known AdaBoost algorithm (\cite{freund1997decision}) where the algorithm implicitly creates a sequence of probability measures. Later work on the hard-core lemma, such as \cite{barak2009uniform}, explicitly focus on efficient sampling from the iterates of multiplicative weights update. 

\paragraph{Robust Learning}
In the field of distribution learning, it is commonly assumed that all samples are consistently coming from an unknown distribution. However, real-world conditions often challenge this assumption, as samples may become corrupted—either removed or substituted by samples from an alternate distribution. Under such circumstances, the theoretical assurances of traditional algorithms may no longer apply. This discrepancy has spurred interest in developing robust learning algorithms capable of tolerating sample corruption. Recent years have seen notable advancements in this area, including the development of algorithms for robustly learning Bernoulli product distributions~\citep{diakonikolas2019robust}, and enhancing the robustness of learning Bayes nets~\citep{cheng2018robust}. See \cite{lai2016agnostic,diakonikolas2017being,diakonikolas2018robustly,balakrishnan2017computationally,hopkins2018mixture,kothari2018robust,diakonikolas2018robustly,diakonikolas2018learning, DBLP:conf/iclr/ChengL21,canonne2023full} and the references therein for a sample of current works in this area. These works primarily focus on guarantees with respect to the total variation distance. 

Of particular relevance is the {\em {\rm TV}-contamination model}. Here, if the distribution to be learnt is $P$, one gets samples from a `contaminated' distribution $Q$ with $\TV(P, Q)\leq \eta$. Note that this is a stronger model than {\em Huber contamination}~(\cite{huber1992robust}), where the noise is restricted to be additive, meaning that an adversary adds a limited number of noisy points to a set of uncontaminated samples from $P$. 

One can interpret our results using a {\em $\KL$-contamination model}. If the distribution to be learnt is an unknown $P$ promised to belong to a class $\cC$, the contaminated distribution $Q$ is some distribution satisfying $\kl(Q\|P)\leq \eta$. The noise is again non-additive, but the model is weaker than TV-contamination. Any $(\eta, A)$ approximation for $Q$ with respect to $\cC$ yields a distribution $\wh{P}$ such that $\kl(Q \| \wh{P}) \leq (A+1)\eta$. Therefore, we get that for Hellinger distance:
\[
\hel(P, \wh{P}) \leq \hel(P,Q)+\hel(\wh{P}, Q) \leq \sqrt{\eta} + \sqrt{(A+1)\eta} \leq \sqrt{(2A+3)\eta}.
\]
Similarly, one can also bound $\TV(P, \wh{P}) = O(\sqrt{\eta})$ for constant $A$. To the best of our knowledge, the KL-contamination model has not been explicitly considered before, but if one were to directly apply the results of \cite{cheng2018robust} with the assumption that $\kl(Q \| P) \le \eta$, one would obtain a distribution $\wh{P}$ such that $\TV(P, \wh{P}) = O(\sqrt{\eta \log 1/\eta})$, worse than ours by a $\sqrt{\log 1/\eta}$ factor which seems unavoidable using their approach \citep{diakonikolas2022optimal}.  Moreover, their results require that $\cC$ be a class of {\em balanced} Bayes nets, a technical condition which is not needed for our analysis~\footnote{A Bayes net is said to be \emph{$c$-balanced} for some $c>0$ if all conditional probability table values $\in [c,1-c]$.}. {However, we would like to note that $\kl(Q || P)$ can be large as compared to $\TV(Q,P)$, so this holds when $\kl(Q || P)$ is small.} 

\section{Open Problems}\label{sec:open_problem} 
Our work opens up several interesting research avenues.
\begin{itemize}
\item
An intriguing question is whether we can extend our result for chordal graphs of bounded indegree to general graphs of bounded treewidth and bounded indegree. Interestingly,  \cite{stanley1973acyclic} showed that counting the number of acyclic orientations reduces to the evaluation of the Tutte polynomial at the point (2,0), and the Tutte polynomial can be evaluated efficiently for bounded treewidth graphs \citep{noble1998evaluating, andrzejak1998algorithm}. This is relevant because the weights that \EWA/\RWM\ maintain are in some sense a weighted count of the number of acyclic orientations of the skeleton. However, we did not find a deletion-contraction recurrence for these weights, and so their connection to the Tutte polynomial is unclear. 
\item
Another important follow-up direction for learning Bayes nets would be to search over {\em Markov equivalence classes} rather than DAG's. A Markov equivalence class corresponds to the set of DAGs that represent the same class of Bayes nets, and they can be represented as partially directed graphs ({\em essential graphs}) that satisfy some special graphical properties. It would be interesting to explore if the structure of essential graphs can be used to speed up weighted counting and sampling; indeed, a very recent work by \cite{sagar24} gives a polynomial time algorithm for uniformly sampling an essential graph that is consistent with a given skeleton.
 
\item 
What is the role of {\em approximate sampling} in the context of distribution learning? So far, in this work, we have only used exact sampling algorithms for spanning arborescences and acyclic orientations. Can Markov chain techniques be brought to good use here? Our work further motivates settling the complexity status of approximately counting the number of acyclic orientations of an undirected graph; this question is a long-standing open problem in the counting/sampling literature. 
\item
Finally, while we have restricted ourselves to learning Bayes nets here, our framework is quite general and also applies to learning other classes of distributions, such as Ising models and factor models. We leave these questions for future work.
\end{itemize}

\paragraph*{Organization of the paper}  The rest of the paper is organized as follows. In \Cref{sec:prelim}, we present the preliminaries required for this work. 
\Cref{sec:agnosticlearningbayesnet} 
establishes the connection between regret in online learning to KL divergence in the scenario of agnostic learning of distributions. It also presents several necessary techniques from online learning along with the \EWA \ and \RWM \ algorithms that will be used later in our work. In \Cref{sec:learnchordaldist}, we present our results on learning chordal-structured distributions. In \Cref{sec:learntreedist}, we discuss our results on learning tree-structured distributions and present our alternative proper learning algorithm. In \Cref{sec:treelearnlb}, we give the lower bound of learning tree-structured distributions. In \Cref{sec:learnbayesnetboundedvc}, we design efficient learning algorithms for Bayes nets over graphs with bounded vertex cover. In \Cref{app:efficient_ml}, we outline how our algorithms can be adapted to efficiently compute maximum likelihood.

\section{Preliminaries}\label{sec:prelim}
For integers $0<m\leq n$, let $[n]$ denote the set $\{1, \dots, n\}$, and let $[m, n]$ denote the set $\{m, m+1, \dots, n\}$. For any $\eta > 0$, let $\exp_\eta(u)$ denote $e^{-\eta u}$. For concise expressions and readability, we use the asymptotic complexity notion of $\widetilde{O}(\cdot)$, where we hide poly-logarithmic dependencies of the parameters. By stating i.i.d samples from a distribution $P$, we mean independently and identically distributed samples from $P$. For a positive integer $\ell$, $\mathsf{Unif}([\ell])$ denotes the uniform distribution on the set $[\ell]$, where each $i \in [\ell]$ is chosen with equal probability of $1/\ell$.

\subsection{Probability Distributions}
We let $\Delta(\cD)$ denote the set of probability distributions over the elements of a set $\cD$.  Let $P$ and $Q$ be two such distributions over $\cD$:
\begin{itemize}
    \item[--] The {\em $\mathrm{KL}$-divergence} between $P$ and $Q$ is defined as:
$\kl(P,Q) = \sum_{x \in \cD} P(x) \log \frac{P(x)}{Q(x)}$.

\item[--] The {\em total variation (TV) distance} between $P$ and $Q$ is defined as:
$
\TV(P,Q) = \sum_{x \in \cD} |P(x)-Q(x)|.
$
\end{itemize}

\begin{lem}[Pinkser's inequality]\label{lem:pinkser}
Let $P$ and $Q$ be two probability distributions defined over the same sample space $\cD$. Then the following holds:
$$\TV(P,Q) \leq \sqrt{\frac{\kl(P||Q)}{2}}$$
\end{lem}

We study distributions over high-dimensional spaces, which can require an exponential amount of space to represent in general. So, for computational efficiency, we focus on distributions from which we can generate samples efficiently. This notion is formally defined below.

\begin{defi}[Efficiently samplable distribution]
A distribution $P$ is said to be {\em efficiently samplable} if there exists  a probabilistic Turing machine $M$ which on input $0^k$ produces a string $x$ such that $|\Pr[M(0^k) = x]-P(x)|\leq 2^{-k}$ and $M$ runs in time $\mathrm{poly}(|x|+k)$. 
\end{defi}

Next we define graphical models of interest in this work. For a directed graph $G$ on $n$ nodes, we will identify its set of nodes with $[n]$. The underlying undirected graph of $G$ is called the {\em skeleton.}
For any $i\in [n]$, let $\mathsf{pa}_G(i)$ denote the set of the parents of node $i$ and $\mathsf{nd}_G(i)$ denote the set of its non-descendants. The subscript $G$ may be removed if it is clear from the context. 

Let us start with the definition of Bayesian networks.

\begin{defi}[Bayesian networks]\label{defi:bayesnet}
A probability distribution $P$ over $n$ variables $X_1, \dots, X_n$ is said to be a {\em Bayesian network (Bayes net in short) on a directed acyclic graph $G$} with $n$ nodes if\footnote{We use the notation $X_S$ to denote $\{X_i : i \in S\}$ for a set $S \subseteq [n]$.} for every $i \in [n]$, $X_i$ is conditionally independent of $X_{\mathsf{nd}(i)}$ given $X_{\mathsf{pa}(i)}$. Equivalently, $P$ admits the factorization:
\begin{equation}\label{eqn:bnfactor}
P(x) = \Pr_{X \sim P}[X=x]= \prod_{i=1}^n \Pr_{X\sim P}[X_i = x_i \mid \forall j \in \mathsf{pa}(i), X_j = x_j] \qquad \text{for all } x.
\end{equation}
\end{defi}

It is well-known that Bayesian networks are efficiently sampleable.

Now we define tree-structured distributions, a subclass of Bayesian networks defined above.

\begin{defi}[Tree-structured distribution]
Let $T$ be a tree, and $G$ be any rooted orientation of $T$. A probability distribution $P$ is said to be {\em $T$-structured} if it is a Bayesian network on $G$. A distribution $P$ is said to be \emph{tree-structured} if it is $T$-structured for some tree $T$.
\end{defi}

Now we define the notion of polytree-structured distributions, which generalizes tree-structured distributions defined before.

\begin{defi}[Polytree-structured distribution]
A directed acyclic graph (DAG) $G$ is said to be a \emph{polytree} if the skeleton of $G$ is a forest. For a positive integer $k \in \N$, $G$ is said to be a $k$-polytree if the in-degree of each node in $G$ is at most $k$.   A distribution $P$ is said to be a \emph{($k$-)polytree-structured} if it is a Bayes net defined over some ($k$-)polytree $G$.
\end{defi}

Finally, we define the notion of Chordal-structured distributions, a class of Bayes nets that generalizes polytree-structured distributions and will be crucially used in this work. 

\begin{defi}[Chordal-structured distribution]
An undirected graph $G$ is said to be \emph{chordal} if every cycle of length at least $4$ has a chord, that is, an edge that connects two non-adjacent vertices on the cycle.  A distribution $P$ is said to be \emph{chordal-structured} if it is a Bayes net defined over a DAG whose skeleton is chordal.
\end{defi}

Clearly, any tree-structured distribution is polytree-structured, and any polytree-structured distribution is chordal-structured.

\adds{
\begin{defi}[Add-one distribution]\label{def:add_one_distribution}
For a directed acyclic graph $G$ with vertex set $[n]$ and a set of samples $S = \{x^{(1)},\ldots,x^{(|S|)}\} \subseteq [k]^n$, the \emph{add-one} or \emph{Laplace distribution} with $G$-structure given the samples $S$, denoted by $\aodist{G}{S}$, is a Bayes net on $G$ and is defined as follows
$$\aodist{G}{S}(x) = \prod_{v \in [n]} \aonodedist{S}{v}{\pa_G(v)}(x_v | x_{\pa_G(v)})$$ where all the node distributions $\aonodedist{S}{v}{\pa(v)}$ are computed as
    $$\aonodedist{S}{v}{\pa(v)}(z_v|z_{\pa(v)}) = \frac{\left|\left\{x^{(i)} \in S \,:\, x^{(i)}_v = z_v \bigwedge x^{(i)}_w = z_w \,\forall w \in \pa(v)\right\}\right| + 1}{\left|\left\{x^{(i)} \in S \,:\, x^{(i)}_w = z_w \,\forall w \in \pa(v)\right\}\right| + k}\,\,\forall\,\,z \in [k]^n.$$
\end{defi}
That is, $\aodist{G}{S}$ is the Bayes net on $G$ where all the node distributions (of $X_v \mid X_{\pa(v)}$) are defined w.r.t the (conditional) add-one or Laplace estimates computed from the samples $S = \{x^{(1)},\ldots,x^{(|S|)}\}$ for each fixing of the parents' values.

}

\subsection{PAC Distribution Learning}

A {\em distribution learning} algorithm takes as input a sequence of i.i.d.~samples generated from a probability distribution $P$ and outputs a description $\Theta$ of a distribution $\wh{P}_\Theta$ as an estimate for $P$. In the following, $\cC$ is a family of probability distributions over $[k]^n$, and $P$ is a distribution over $[k]^n$ not necessarily in $\cC$. Let us first define the notion of approximation of a distribution that will be used in this work. 

\begin{defi}[$(\eps,A)$-approximation of a distribution]
A distribution $\wh{P}$ is said to be an {\em $(\eps, A)$-approximation for $P$ with respect to $\cC$} if:
\[
\kl(P \| \wh{P}) \leq A \cdot \inf_{Q \in \cC} \kl(P \| Q) + \eps.
\]
\end{defi}
When $A=1$, $\wh{P}$ is said to be an {\em $\eps$-approximation} for $P$. Note that when $P \in \cC$, if $\wh{P}$ is an $\eps$-approximation for $P$, then $\TV(P, \wh{P}) \leq \sqrt{\frac{\eps}{2}}$ using Pinsker's inequality (\Cref{lem:pinkser}). 

Now we proceed to define the notion of PAC-learning with respect to KL divergence.

\begin{defi}[PAC-learning in KL divergence]
A distribution learning algorithm is said to\footnote{In the learning theory literature, when $A>1$, such a guarantee is also sometimes called {\em semi-agnostic learning}.} be an {\em agnostic PAC-learner for $\cC$} with sample complexity $m_\cC(n, k, \eps,\delta)$ and running time $t_\cC(n, k, \eps, \delta)$, if for all distributions $P$ over $[k]^n$ and all $\eps, \delta \in (0,1)$, given $\eps, \delta$, and a sample set of size $m = m_\cC(n, k, \eps,\delta)$ drawn i.i.d.~from $P$, the algorithm runs for time $t \leq t_\cC(n, k, \eps, \delta)$ and outputs the description $\Theta$ of a distribution $\wh{P}_\Theta$ such that with probability at least $1-\delta$, $\wh{P}_\Theta$ is an $\eps$-approximation for $P$ with respect to $\cC$ (where the probability is taken over the samples as well as the algorithm's randomness).

If $P$ is restricted  to be in $\cC$, the algorithm is said to be a {\em realizable PAC-learner for $\cC$}.  If the output $\wh{P}_\Theta$ is guaranteed to be in $\cC$, the algorithm is said to be a {\em proper PAC-learner for $\cC$}; otherwise the learner is called an {\em improper PAC-learner}. 
\end{defi}

\subsection{Online Learning}
The framework of prediction with experts is setup as follows; the formulation we follow here is based on Chapter 2 of \cite{cesa2006prediction}. The goal is to design an algorithm $\cA$ that predicts an unknown sequence $x^{(1)}, x^{(2)}, \dots$ of elements of an {\em outcome space} $\cX$. The algorithm's predictions $\wh{P}_1, \wh{P}_2, \dots$ belong to a {\em decision space} $\cD$ (which in our case will be the probability simplex on $\cX$). The algorithm $\cA$ makes its predictions sequentially, and the quality of its predictions is benchmarked against a set of reference predictions called {\em experts}. At each time instant $t$, $\cA$ has access to the set of expert predictions $\cE$, where $\cE$ is a fixed subset of $\cD$. $\cA$ makes its own prediction $\wh{P}_t$ based on the expert predictions. Finally, after the prediction $\wh{P}_t$ is made, the true outcome $x^{(t)}$ is revealed, and the algorithm $\cA$ incurs a loss $\ell(\wh{P}_t, x^{(t)})$, where $\ell: \cD \times \cX \to \R$ is a non-negative {\em loss function}. 

Note that $x^{(1)}, x^{(2)}, \dots$ form an arbitrary sequence. Hence, naturally, the loss incurred by the algorithm depends on how well the experts fit these outcomes. 

Now we are ready to define the notion of regret.
\begin{defi}[Regret]\label{defi:regret}
Given the notation above, the {\em cumulative regret} (or simply {\em regret}) over $T$ steps is defined as:
\[
\Reg_T(\cA; \cE) = \sum_{t=1}^T \ell(\wh{P}_t, x^{(t)}) - \min_{E \in \cE} \sum_{t=1}^T \ell(E, x^{(t)})
\]
Since $\cE$ is usually fixed, we often use the shorthand $\Reg_T(\cA)$. The {\em average regret} is defined as $\frac1T \Reg_T(\cA)$. The total $T$ steps are also sometimes referred to as the \emph{horizon} of the algorithm.
\end{defi}

Our work utilizes two standard online learning algorithms: Exponentially Weighted Average (\EWA) and Randomized Weighted Majority (\RWM). Both algorithms assume that the expert set is finite; the first algorithm is deterministic, while the second one is probabilistic.

\begin{fact}\label{def:ewa_rwm_weight_dist}
In both \EWA\ and \RWM\ algorithms, when run with a hyperparameter $\eta > 0$ and a set of experts $\cE = \{E_1,\ldots,E_N\}$, the \emph{weight of the $i$-th expert} ($i \in [N]$) \emph{after} time step $t$ ($t \in [T]$) is denoted by $w_{i,t}$ defined as follows:
\[w_{i,t} \triangleq \exp\left(-\eta \sum_{s=1}^{t} \ell(E_i,x^{(s)})\right).\]
\end{fact}
    \begin{algorithm}[htbp]
    \setcounter{AlgoLine}{0}
	\caption{\EWA\ forecaster}\label{alg:ewa}
	\SetKwInOut{Input}{Input}
	\Input{\,\,Experts ${\cE} = \{E_1,\ldots,E_N\}$, parameter $\eta$, horizon $T$.}
	$w_{i,0} \gets 1 \mbox{ for each } i \in [N].$ \
    
	\For{$t \gets 1$ to $T$}{
		$\widehat{P}_t \gets \frac{\sum_{i=1}^N w_{i,t-1} \, {E_i}}{\sum_{j=1}^{N} w_{j,t-1}}$. \
        
        \textbf{Output} $\wh{P}_t$. \
        
        Observe outcome $x^{(t)}$. \
        
        \For{$i \in [N]$}{
			$w_{i,t} \gets w_{i,t-1} \cdot \exp\left(-\eta\cdot\ell({E_i}, x^{(t)})\right)$.
		}
	}
\end{algorithm}
\hfill
	\begin{algorithm}[htbp]
    \setcounter{AlgoLine}{0}
	\caption{\RWM\ forecaster}\label{alg:rwm}
	\SetKwInOut{Input}{Input}
	\Input{\,\,Experts ${\cE} = \{E_1,\ldots,E_N\}$, parameter $\eta$, horizon $T$.}
	$w_{i,0} \gets 1 \mbox{ for each } i \in [N].$ \
    
	\For{$t \gets 1$ to $T$}{
		Sample $\widehat{P}_t \in \cE$ where $\Pr[\wh{P}_t = {E_i}] = \frac{w_{i,t-1}}{\sum_{j=1}^{N} w_{j,t-1}}$. \
        
		\textbf{Output} $\wh{P}_t$. \
		
        Observe outcome $x^{(t)}$. \
		
        \For{$i \in [N]$}{
			$w_{i,t} \gets w_{i,t-1} \cdot \exp\left(-\eta\cdot\ell({E_i}, x^{(t)})\right)$.
		}
	}
	\end{algorithm}

Before proceeding to give the regret bounds for \EWA\ and \RWM\ algorithms, we need the notion of exp-concave loss functions.

\begin{defi}[Exp-concave loss function]
Let $\cD$ be a convex decision space, $\cX$ be a sample space, and let $\alpha \in \R$ be a parameter. A loss function $\ell: \cD \times \cX \rightarrow \R$ is said to be \emph{$\alpha$-exp-concave}  if the function $f_x: \cD \rightarrow \R$, $f_x(P) =\exp(-\alpha \ell(P,x))$ is concave for all $x \in \cX$.
\end{defi}

The following regret bound holds for \EWA\ with {\em exp-concave} loss functions. Importantly for us, $\ell(P,x)=-\log(P(x))$ is exp-concave for $0\leq \alpha\leq 1$.

\begin{lem}[\cite{cesa2006prediction}, Theorem 3.2]\label{lem:ewa_regret_bound}
With $N = |{\cE}|$ experts and an $\alpha$-exp-concave loss function $\ell(P, x)$, for any sequence of outcomes $x^{(1)}, \dots, x^{(T)}$, the regret of the {\rm\EWA}\ forecaster with the parameter $\eta = \alpha$ is bounded as follows:
\[
\Reg_T(\mathsf{EWA} ; \cE) \leq \frac{\log N}{\eta}.
\]
\end{lem}

The regret bound below for \RWM\ assumes the loss function is bounded.
\begin{lem}[\cite{cesa2006prediction}, Lemma 4.2]\label{lem:rwmruarantee}
With $N = |\cE|$ experts and a loss function bounded in $[-L_{\max}, L_{\max}]$, for any sequence of outcomes $x^{(1)}, \dots, x^{(T)}$, the expected regret of the {\rm\RWM}\ forecaster with the parameter $\eta=\sqrt{8 (\log N)/T}$ is bounded as follows:
\[\E[\Reg_T(\mathsf{RWM}; \cE)] \leq L_{\max}\sqrt{\frac{T \log N}{2}}.
\]
Moreover, with probability at least $1-\delta$, the following holds:
\[
\Reg_T(\mathsf{RWM}; \cE)] \leq L_{\max} \left(\sqrt{\frac{T \log N}{2}} + \sqrt{\frac{T}{2} \log \frac1\delta}\right).
\]
\end{lem}

\section{Agnostic Learning of Bayesian Networks}\label{sec:agnosticlearningbayesnet}

In this section, we study the notions and our results on the agnostic learning of Bayesian networks in general. Our first result in Section~\ref{sec:reg_kl} shows that we can leverage the multiplicative weight update method to get an agnostic learning guarantee in reverse KL distance in terms of the regret of the RWM or EWA algorithm. Thereafter, we can plug in the standard regret bounds for EWA and RWM to get an agnostic learning guarantee for our output distributions with high probability. These results are given in Lemma~\ref{lem:ewa_reg_final} and~\ref{lem:rwm_reg_final} respectively. Next we apply our general learning result to the problem of learning Bayesian networks which requires us to define an appropriate set of experts. Towards this, we define one expert for each possible DAGs on $n$ nodes of in-degree at most $d$. The distribution of such an expert will be fixed to the add-1 distribution at each conditional distribution defined using a previous set of samples. A prior work by Bhattacharyya et al.~(\cite{bhattacharyya2023near}) has established that this add-1 estimator will have a near-optimal sample complexity of agnostically learning fixed-structure Bayesian networks in KL divergence. Combining everything together, we get interesting new sample-complexity upper bounds of agnostically learning Bayes nets in both improper and proper learning setting using EWA and RWM algorithms respectively in Theorems~\ref{thm:bn_learn_imp} and~\ref{thm:bn_learn_proper}.

\begin{theo}\label{thm:intro_bayes}
Let $\cC$ be the family of distributions over $[k]^n$ that can be defined as Bayes nets over DAGs with $n$ nodes and maximum in-degree $d$.
There exists a PAC-learner for $\cC$ in the realizable setting using $\tilde{O}(nk^{d+1}\eps^{-1})$ samples that returns a mixture of Bayes nets from $\cC$ which is an $\eps$-approximation of the input distribution with probability at least $2/3$.
\end{theo}

The sample complexity in \Cref{thm:intro_bayes} is {\em nearly optimal}\footnote{Note that for learning with respect to total variation distance, the corresponding sample complexity of $\widetilde{O}(nk^{d+1}\eps^{-2})$ is known (\cite{BCD20}), but the proof uses the tournament method which does not apply to KL divergence as it is not a metric.} for constant error probability. The same approach also yields polynomial sample complexity for agnostic learning in the non-realizable setting with exponentially small error probability and for proper learning with polynomially small error probability. While these bounds are no longer optimal, we will use them later to design the first polynomial time algorithms for learning Bayes nets over large classes of DAGs.

We now describe the relation between regret and KL divergence, which is crucially used throughout this work.

\subsection{Connection between Regret and KL divergence}\label{sec:reg_kl}
In the prediction with experts framework, let us fix the decision space $\cD$ to be the probability simplex on the outcome space $\cX$ and the loss function to be $\ell(P, x) = - \log P(x)$ for the distribution $P$ defined over the probability simplex on $\cX$. Then we have the following general connection between the expected average regret and the expected KL divergence when the outcomes $x^{(1)},x^{(2)},\cdots$ are i.i.d samples from a distribution $P^\ast$.

\begin{lem}\label{thm:exp_avg_regret_agnostic_learning_connection}
	If $\cA$ is a (possibly randomized) online learning algorithm that receives i.i.d samples $x^{(1)},\ldots,x^{(T)}$ from a probability distribution $P^*$ on $\cX$ and makes predictions $\wh{P}_1, \dots, \wh{P}_T \in \cD \subseteq \Delta(\cX)$ with respect to an expert  set $\cC \subseteq \cD$, 
\[\E\limits_{x^{(1)}, \ldots, x^{(T)}} \E_{t \sim \mathsf{Unif}([T])} \left[\kl(P^*\| \widehat{P}_t)\right] \leq \frac{1}{T} \E\limits_{x^{(1)}, \ldots, x^{(T)}}\left[\Reg_T(\cA \,;\, \cC)\right] + \min\limits_{P \in \cC}  \kl(P^* \| P).
\]  
\end{lem}
Note that if $\cA$ is a randomized algorithm, both the left and right hand side of the inequality in \Cref{thm:exp_avg_regret_agnostic_learning_connection} are random variables.

\begin{proof}
	From the definition of regret (\Cref{defi:regret}), we have:
	$$\Reg_T(\cA \,;\, \cC) \triangleq \sum_{t=1}^T \log \frac{1}{\widehat{P}_t(x^{(t)})} - \min_{P \in \cC} \sum_{t=1}^T \log \frac{1}{P(x^{(t)})}.$$
	Thus, the expectation over the samples of the average regret for $T$ rounds would be the following:
	\begin{equation}\label{eq:reg_kl_proof_1}
	\begin{aligned}
		\frac{1}{T} \E_{x^{(1)}, \ldots, x^{(T)} \sim P^*}\left[\Reg_T(\cA \,;\, \cC)\right] &=&\E_{x^{(1)}, \ldots, x^{(T)}} \frac{1}{T} \cdot \left(\sum_{t=1}^T \log \frac{1}{\widehat{P}_t(x^{(t)})} - \min_{P \in \cC} \sum_{t=1}^T \log \frac{1}{P(x^{(t)})}\right) \\ &\geq& \frac{1}{T} \E_{x^{(1)}, \dots, x^{(T)}}\left[\sum_{t=1}^T \log \frac{1}{\widehat{P}_t(x^{(t)})}\right] - \min_{P \in \cC} \E_{x \sim P^*}\left[\log \frac{1}{P(x)}\right]\\
&=& \frac{1}{T} \E_{x^{(1)},\dots,x^{(T)}}\left[\sum_{t=1}^T \log \frac{1}{\widehat{P}_t(x^{(t)})}\right] - \min_{P \in \cC} \kl(P^* \| P) - H(P^*)
\end{aligned}
\end{equation}
	where the second-last step is by applying Jensen's inequality, as $\min$ is a concave function, and also uses the fact that $x^{(1)},\ldots,x^{(T)} \sim P^\ast$. Also, $H(P^*)$ denotes the entropy of the distribution $P^*$.	Thus we have the following:
	\begin{eqnarray*}
		\frac{1}{T} \E_{x^{(1)}, \ldots, x^{(T)}}\left[\Reg_T(\cA \,;\, \cC)\right] &\geq& \frac{1}{T} \E_{x^{(1)}, \ldots, x^{(T)}}\left[\sum_{t=1}^T \log \frac{1}{\widehat{P}_t(x^{(t)})}\right] - \min_{P \in \cC}  \kl(P^* \| P) - H(P^*) \\
		&=& \frac{1}{T} \sum_{t=1}^T \E_{x^{(1)}, \ldots, x^{(t-1)}} \E_{x \sim P^*} \left[\log \frac{1}{\widehat{P}_t(x)} \Big| x^{(1)}, \ldots, x^{(t-1)}\right] - H(P^*) \\
  && \qquad\qquad\qquad- \min_{P \in \cC}  \kl(P^* \| P) \\
		&=& \frac{1}{T} \sum_{t} \E_{x^{(1)}, \ldots, x^{(t-1)}} \left[\kl(P^* \| \widehat{P}_t)\right] - \min_{P \in \cC}  \kl(P^* \| P) \\
		&=& \E_{x^{(1)}, \ldots, x^{(T-1)}} \left[\frac{1}{T} \sum_{t=1}^T \kl(P^* \| \widehat{P}_t)\right] - \min_{P \in \cC}  \kl(P^* \| P)%
	\end{eqnarray*}
In the second line, inside the summation over $t$, the expectation does not include $x^{(t)}$ because the prediction $\wh{P}_t$ does not depend upon $x^{(t)}$, as $\wh{P}_t$ is predicted before $x^{(t)}$ is revealed .
\end{proof}

We immediately have the following corollary:
\begin{coro}\label{cor:exp_reg}
In the same setup as \Cref{thm:exp_avg_regret_agnostic_learning_connection} above, if $\frac1T \Reg_T(\cA; \cC) \leq \rho$, 
\[
\E\limits_{x^{(1)}, \ldots, x^{(T)}} \left[\kl\left(P^*\| \frac1T \sum_{t=1}^T\widehat{P}_t\right)\right] \leq \rho + \min\limits_{P \in \cC}  \kl(P^* \| P).
\]
\end{coro}
\begin{proof}
This follows from applying the fact that KL-divergence is convex in its second argument.
\end{proof}

Now, we apply these results to the \EWA\ and \RWM\ algorithms. Let us start with the results for the \EWA\ algorithm.

\begin{lem}\label{lem:ewa_reg_final}
In the same setup as \Cref{thm:exp_avg_regret_agnostic_learning_connection} above, suppose that for every $P \in \cC$,  $\min_{x \in \cX} P(x) \geq \tau$ holds.

\begin{itemize}
\item
If $\cA$ is the {\rm\EWA}\ algorithm run with parameter $\eta = 1$, then we have the following:
\[
\E_{x^{(1)}, \dots, x^{(T)}}\left[\kl\left(P^* \| \frac1T \sum_{t=1}^T \wh{P}_t\right)\right] \leq \min_{P \in \cC} \kl(P^* \| P) + \frac{\log |\cC|}{T}.
\]

\item
For any $\eps, \delta \in (0,1)$, let $\cA$ be the {\rm\EWA}\ algorithm run with parameter $\eta = \frac{\eps}{2\sqrt{T}\log(1/\tau)\sqrt{\log(1/\delta)}}$.  Then with probability at least $1-\delta$, the following holds:
\[
\kl\left(P^* \| \frac1T \sum_{t=1}^T \wh{P}_t\right) \leq \min_{P \in \cC} \kl(P^* \| P) + \frac{2\log |\cC| \log(1/\tau)\sqrt{\log(1/\delta)}}{\eps\sqrt{T}} + \eps.
\] 
\end{itemize}
\end{lem}

\begin{proof}
Recall that our loss function $\ell(P,x) = -\log P(x)$ is $\eta$-exp-concave for all $0 \leq \eta \leq 1$. Hence, we can apply \Cref{lem:ewa_regret_bound} for any such $\eta$. The first item follows from \Cref{cor:exp_reg} by setting $\eta=1$.

For the second item, we use the method of bounded differences. Define the function $f$:
\[
f(X^{(1)}, \dots, X^{(T)}) = \kl\left(P^* \| \frac1T \sum_{t=1}^T \wh{P}_t\right)
\]

\begin{cl}
For any $X^{(1)}, X^{(2)}, \dots, X^{(T)}, {X^{(1)}}' \in \cX$, 
\[
|f(X^{(1)}, X^{(2)}, \dots, X^{(T)})-f({X^{(1)}}', X^{(2)}, \dots, X^{(T)})| \leq 2\eta \log \frac1\tau.
\]
\end{cl}

\begin{proof}
 Let $\wh{P}_t$ and $\wh{P}_t$ be the outputs from the EWA algorithm after $t$ rounds using the samples $X^{(1)}, X^{(2)}, \dots, X^{(T)}$ and ${X^{(1)}}', X^{(2)}, \dots, X^{(T)}$ respectively. Then, $\wh{P}_t=\sum_{Q \in \cC} w_{t,Q} Q$ and $\wh{P}'_t=\sum_{Q \in \cC} w_{t,Q}' Q$. 
     
By definition of the \EWA\ algorithm and the choice of our loss function, $w_{t,Q}= \frac{\left(\prod_{s\le t} Q(X^{(s)})\right)^\eta}{\sum_{R \in \cC} \left(\prod_{s\le t} R(X^{(s)})\right)^\eta}$ and $w'_{t,Q}= \frac{\left(Q({X^{(1)}}')\prod_{2\le s\le t} Q(X^{(s)})\right)^\eta}{\sum_{R \in \cC} R({X^{(1)}}') \left(\prod_{2\leq s\le t} R(X^{(s)})\right)^\eta}$. Since for every $Q \in \cC$,  $\tau\le \frac{Q(X^{(1)})}{Q({X^{(1)}}')}\le \frac{1}{\tau}$, $ \tau^{2\eta} \le \frac{w_{t,Q}}{w'_{t,Q}}\le \tau^{-2\eta}$ and hence for every $y \in \cX$, $\tau^{2\eta} \leq \frac{\wh{P}'_t(y)}{\wh{P}_t(y)}\le \tau^{-2\eta}$. Therefore:
   \begin{align*}
        \abs{f({X^{(1)}}',\dots, X^{(T)})-f(X^{(1)},\dots, X^{(T)})}
        =\abs{\sum_{y \in \cX} P^*(y)\log \paren{\frac{\sum_{t=1}^T\wh{P}'_t(y)}{\sum_{t=1}^T\wh{P}_t(y)}}}
        \le 2\eta \log \frac{1}{\tau}.
    \end{align*}

\end{proof}

Applying the McDiarmid's inequality, %
    \[
    \Pr\bracket{\abs{\kl(P^*||\wh{P})-\E\bracket{\kl(P^*||\wh{P})}}\ge \eps}\le 
     2\exp\paren{-\frac{
     2\eps^2
     }{
     4T\eta^2\log^2 \tau^{-1}
     }
     }\leq \delta, 
    \]
by our choice of $\eta$. The second item then follows from  \Cref{lem:ewa_regret_bound}, \Cref{cor:exp_reg}, and the above.
\end{proof}

Now let us discuss the guarantees of the \RWM\ algorithm.

\begin{lem}\label{lem:rwm_reg_final}
In the same setup as \Cref{thm:exp_avg_regret_agnostic_learning_connection} above, suppose that $\cC = \{P_1,\ldots,P_N\}$ and that for every $P_i \in \cC$,  $\min_{x \in \cX} P_i(x) \geq \tau$ holds.
\begin{itemize}
\item
If $\cA$ is the {\rm\RWM}\ algorithm run with the parameter $\eta = \sqrt{(8 \log N)/T}$ --- giving \emph{proper} predictions $\widehat{P}_1,\ldots,\widehat{P}_T \in \cC$ --- and $\widehat{P} \gets \widehat{P}_t \in \cC$ where $t$ is sampled uniformly from $[T]$, then we have:
\[
\E_{\cA,x^{(1)}, \dots, x^{(T)}}\left[\kl\left(P^* \| \wh{P}\right)\right] \leq \min_{P \in \cC} \kl(P^* \| P) + \frac{\log(1/\tau)\sqrt{T \log |\cC|}}{T}.
\]
\item With the same algorithm $\cA$ as above, for any $\delta \in (0, 1)$, the following bound holds with probability $\geq 1 - \delta$ over the randomness of {\rm\RWM}, the random choice of $t$ from $[T]$, and the randomness of the samples $x^{(1)},\ldots,x^{(T)}$:
\[
\kl(P^\ast \| \widehat{P}) \leq \min_{P \in \cC} \kl(P^\ast \| P) + \frac{\log(1/\tau)\left(\sqrt{\log |\cC|} + \sqrt{\log(2/\delta)}\right)}{\delta\sqrt{T}}.
\]
\end{itemize}
\end{lem}
\begin{proof}
The first item follows from \Cref{thm:exp_avg_regret_agnostic_learning_connection}, taking the expectation of both sides over the randomness of $\cA$, and the expected regret bound in \Cref{lem:rwmruarantee}. Note that we can take $L_{\max} = \log(1/\tau)$ since $P(x) \geq \tau$ for all $x \in \cX$ and $P \in \cC$, by the assumption of the lemma.

The \RWM\ algorithm (see Algorithm \ref{alg:rwm}) can equivalently be viewed as follows. Sample $U_1, \ldots, U_T \sim \mathsf{Uniform}((0,1])$ independently. At each $t \in [T]$, predict $\wh{P}_t \gets P_{i_t} \in \cC$, where $i_t \in [N]$ is the index such that $U_t \in \Bigg(\frac{\sum_{j=1}^{i_t-1} w_{j,t-1}}{\sum_{k=1}^{N} w_{k,t-1}}, \frac{\sum_{j=1}^{i_t} w_{j,t-1}}{\sum_{k=1}^{N} w_{k,t-1}}\Bigg]$. Note that modulo the choice of $U = (U_1,\ldots,U_T)$, the \RWM\ algorithm is deterministic.

From \Cref{lem:rwmruarantee} (high probability bound), with probability $\geq 1 - \delta/2$ over the random choice of $U$, $\mathrm{Regret}_T(\cA\,;\,\cC) \leq \log(1/\tau)\left(\sqrt{\frac{T \log |\cC|}{2}} + \sqrt{\frac{T}{2}\log\frac{2}{\delta}}\right)$. Conditioning on this event $\cE$, applying \Cref{thm:exp_avg_regret_agnostic_learning_connection}, we get
\begin{align*}
\E_{x^{(1)},\ldots,x^{(T)}} \E_{t \sim \mathsf{Unif}([T])} &\left[\mathcal{Z}_{X,t} \triangleq \kl(P^\ast \| \wh{P}_t) - \min_{P \in \cC} \kl(P^\ast \| P) \,\middle| \cE\right]\\
&\leq \frac{\log(1/\tau)\left(\sqrt{\log |\cC|}+ \sqrt{\log(2/\delta)}\right)}{2\sqrt{T}}
\end{align*}

Note that, since $\widehat{P}_t \in \cC$ for all $t \in [T]$, $\mathcal{Z}_{X,t} \geq 0$ for all $x^{(1)},\ldots,x^{(T)} \in \cX$ and $t \in [T]$. Thus, we can use Markov's inequality to show that, conditioned on the event $\cE$, with probability $\geq 1 - \delta/2$ over the randomness of $x^{(1)},\ldots,x^{(T)}$ and the random choice of $t \sim \mathsf{Unif}([T])$,
\begin{equation}\label{eqn:gen_regret_rwm_proper_prob_final}
\kl(P^\ast \| \widehat{P}) - \min_{P \in \cC} \kl(P^\ast \| P) \leq \frac{\log(1/\tau)\left(\sqrt{\log |\cC|} + \sqrt{\log(2/\delta)}\right)}{\delta\sqrt{T}}.
\end{equation}
Taking a union bound with event $\overline{\cE}$ and the failure event of the above bound (\ref{eqn:gen_regret_rwm_proper_prob_final}), we can see that (\ref{eqn:gen_regret_rwm_proper_prob_final}) holds with probability $\geq 1 - \delta$ over the randomness of algorithm $\cA$, the random choice of $t$, and the randomness of the samples $x^{(1)},\ldots,x^{(T)}$. This completes the proof of the lemma.
\end{proof}

\subsection{Discretization}\label{sec:discret}

\adds{
We will typically want to design agnostic PAC-learners for classes of distributions $\cC$ that are infinite, e.g., tree-structured distributions, etc. However, to apply the results of the previous section, we first need to \emph{finitize} $\cC$. We make use of the properties of the add-one distribution --- also explored in \cite{bhattacharyya2023near} --- to do this finitization more efficiently compared to using a $\varepsilon$-cover over each set of node distributions (the cover method \cite{yatracos1985rates,devroye2001density}, applied to Bayes nets). This of course means that we make use of additional samples from the ground truth distribution $P^\ast$ to construct the discretization. However, we show that this can be done without affecting our asymptotic sample complexity (up to logarithmic factors).

Suppose we are interested in the class $\cC$ of $\cG$-structured Bayes nets, where $\cG = \{G_1,\ldots,G_M\}$ is a \emph{finite set} of directed acyclic graphs (DAGs), all with indegree $\leq d$ --- $\cG$ may be the set of all rooted trees on $[n]$ (with $d = 1$), all DAGs having a particular (undirected) skeleton, all DAGs on $[n]$ with indegree $\leq d$, etc. We define the discretization $\cN^{\cG}_{\eps,\delta}$ to be the set of add-one distributions constructed from a set of $s_{\rm AO}(\varepsilon,\delta) \leq \tilde{O}\left(\frac{d^3 nk^{d+1}\log^2(nk/\eps\delta)}{\varepsilon}\right)$ samples from $P^\ast$ (see \Cref{def:add_one_distribution} and \Cref{thm:add_one_guarantee_kl_uniform}) for each DAG $G \in \cG$. Note that all the distributions in $\cN^\cG_{\eps,\delta}$ are constructed using a single set of samples from $P^\ast$.

\begin{defi}\label{def:bayes_dags_finitization}
Suppose $\cG = \{G_1,\ldots,G_M\}$ is a set of DAGs on $n$ nodes with maximum indegree $d$, and $P^\ast$ is a distribution on $[k]^n$. For $\eps > 0$ and $\delta \in (0, 1)$, the finitization $\cN^{\cG}_{\eps,\delta}$ is constructed as $\cN^{\cG}_{\eps,\delta} \triangleq \{\aoestdsamp{G_i}{S} : i \in [M]\}$ (see \Cref{thm:add_one_guarantee_kl_uniform}), where $S$ is a set of $s_{\rm{AO}}(\eps,\delta) \triangleq \Theta\left(\frac{nk^{d+1}}{\varepsilon} \log\left(\frac{n^{d+1}k^{d+1}}{\delta}\right)\log\left(\frac{nk^{d+1}}{\varepsilon}\log\left(\frac{1}{\delta}\right)\right)\right)$ i.i.d samples from $P^\ast$.
\end{defi}

For determining the number of samples $s_{\rm{AO}}(\eps,\delta)$ used to construct the finitization, we use the following refined variant ofTheorem 6.2, in \cite{bhattacharyya2023near} (where the guarantee holds uniformly for all DAGs on $[n]$ with indegree $\leq d$).

\begin{theo}\label{thm:add_one_guarantee_kl_uniform}
Let $\cG$ denote the set of all DAGs on $[n]$ with maximum indegree $\leq d$, and $P^\ast$ is a distribution on $[k]^n$. For any $G \in \cG$, let $\aoestdsamp{G}{S}$ denote the add-one distribution (see \Cref{def:add_one_distribution}) with $G$-structure constructed given a set $S$ of i.i.d samples from $P^\ast$. Then, if $|S| \geq s_{\rm AO}(\varepsilon,\delta) \triangleq \Theta\left(\frac{nk^{d+1}}{\varepsilon} \log\left(\frac{n^{d+1}k^{d+1}}{\delta}\right)\log\left(\frac{nk^{d+1}}{\varepsilon}\log\left(\frac{1}{\delta}\right)\right)\right)$, the distribution $\aoestdsamp{G}{S}$ satisfies $$\kl(P^\ast \| \aoestdsamp{G}{S}) \leq \min_{G\textrm{-structured distributions } Q} \kl(P^\ast \| Q) + \varepsilon$$ for every $G \in \cG$, with probability $\geq 1 - \delta$ over the samples $S$.
\end{theo}
\begin{proof}
As in the proof of Theorem 6.2 in \cite{bhattacharyya2023near}, for any DAG $G \in \cG$, we can write $\kl(P^\ast \| \aoestdsamp{G}{S}) - \kl(P^\ast \| P^\ast_G)$ as:
\begin{equation}\label{eqn:kl_to_add_one_decomp_bn}
 \sum_{v \in [n]} \sum_{x \in [k]^{\pa_G(v)}} P^\ast(X_{\pa_G(v)} = x) \kl\left(P^\ast(X_v | X_{\pa_G(v)} = x) \,\|\, \aoestdsamp{G}{S}(X_v | X_{\pa_G(v)} = x)\right),
\end{equation}
where $P^\ast_G$ is a distribution that minimizes $\kl(P^\ast \| P)$ among all $G$-structured distributions $P$.

Note that this decomposition depends only on the set of parents of each node $v \in [n]$. Thus to upper bound the LHS for all DAGs in $\cG$, we just need to upper bound $\sum_{x \in [k]^{U}} P^\ast(X_{U} = x) \kl\left(P^\ast(X_v | X_{U} = x) \,\|\, \aoestdsamp{G}{S}(X_v | X_{U} = x)\right)$ for all $v$ and for all subsets $U \subseteq [n]$ with $|U| \leq d$. The number of possible subsets is $\binom{n+d}{d} \leq \left(\frac{2en}{d}\right)^d \leq O((2n)^d)$ since $d \leq n-1$. Not all such configurations will give valid DAGs in $\cG$, but a bound for all $v, U$ is sufficient to upper bound the LHS for every $G \in \cG$.

Now we can proceed as in the proof of Theorem 1.4 of \cite{bhattacharyya2023near}. Let $N = |S|$ (total number of samples from $P^\ast$) and $r$ be the number of samples from $P^\ast(X_v | X_{U} = x)$ for a particular $v \in [n]$, $U \subseteq [n]$ and $x \in [k]^U$. Using the guarantee for the add-one estimator at each node (Theorem 6.1, \cite{bhattacharyya2023near}) and splitting into two cases ($\E[r] > 15 \log(1/\delta)$ and $\E[r] \leq 15 \log(1/\delta)$) as in the proof of (Theorem 1.4, \cite{bhattacharyya2023near}) gives us that
$$P^\ast(X_{U} = x) \kl\left(P^\ast(X_v | X_{U} = x) \,\|\, \aoestdsamp{G}{S}(X_v | X_{U} = x)\right) \leq O\left(\frac{k \log\left(\frac{k}{\delta}\right)\log N}{N}\right)$$
with probability $\geq 1 - \delta$ over the $N$ samples.

Rescaling $\delta$ to $\delta^\prime = \frac{\delta}{n(2kn)^d}$ and using a union bound gives the above inequality for all choices of $v \in [n]$, $U \subseteq [n]$ with $|U| \leq d$ and $x \in [k]^U$ together, with probability $\geq 1 - n(2kn)^d\delta^\prime = 1 - \delta$ over the samples.

Using this upper bound for all terms in equation (\ref{eqn:kl_to_add_one_decomp_bn}) along with the fact that $|\pa_G(v)| \leq d$ for all $G \in \cG$ gives (w.p $\geq 1 - \delta$)
$$\forall\, G \in \cG \,:\, \kl(P^\ast \| \aoestdsamp{G}{S}) - \kl(P^\ast \| P^\ast_G) \leq O\left(\frac{n k^{d+1} \log\left(\frac{n k (2kn)^d}{\delta}\right)\log N}{N}\right).$$

Choosing $N = \Theta\left(\frac{n k^{d+1}}{\eps} \log\left(\frac{n^{d+1} k^{d+1}}{\delta}\right)\log\left(\frac{n k^{d+1}}{\eps}\log\left(\frac{1}{\delta}\right)\right)\right)$ proves the theorem.
\end{proof}

This finitization of the class of $\cG$-structured Bayes nets has the following useful properties.

\begin{lem}\label{lem:bn_net_lb}
	Any $P \in \cN^\cG_{\eps,\delta}$ (where DAGs in $\cG$ have indegree $\leq d$) satisfies $\min_{x \in [k]^n} P(x) \geq \left(\frac{\eps}{C_d n k^{d+1}\log^2\left(\frac{nk}{\eps\delta}\right)}\right)^n$, for some constant $C_d > 0$ that depends on $d$. 
\end{lem}
\begin{proof}
	Each distribution $P \in \cN^{\cG}_{\eps,\delta}$ is an add-one distribution (see \Cref{def:add_one_distribution}) constructed from $s_{\rm AO}(\eps,\delta)$ samples of $P^\ast$. Each such distribution is a Bayes net on an $n$-vertex DAG and hence the joint probability is a product of the $n$ node (conditional) probabilities. Since each node distribution $X_v | X_{\pa(v)}$ is an add-one estimate, the probability value (for each $X_v = x_v | X_{\pa(v)} = x_{\pa(v)}$) is at least $\frac{1}{s_{x_\pa(v)}+k} \geq \frac{1}{s_{\rm AO}(\eps,\delta)+k}$, where $s_{x_\pa(v)}$ is the number of samples consistent with $X_{\pa(v)} = x_{\pa(v)}$. Hence,
    \begin{equation*}
    \begin{aligned}
	P(x) &\geq \prod_{v \in [n]}\frac{1}{s_{x_{\pa(v)}}+k} \geq \left(\frac{1}{s_{\rm AO}(\eps,\delta)+k}\right)^n \numrelop{1}{\geq} \Theta\left(\frac{\eps}{d^3 n k^{d+1}\log\left(\frac{nk}{\delta}\right)\log\left(\frac{nk}{\eps}\log\frac{1}{\delta}\right)}\right)^n\\
    &\geq \left(\frac{\eps}{C_d n k^{d+1}\log^2\left(\frac{nk}{\eps\delta}\right)}\right)^n,
    \end{aligned}
    \end{equation*}
 where the inequality (1) follows from $s_{\rm AO}(\eps,\delta) = \Theta\left(\frac{nk^{d+1}}{\eps}\log\left(\frac{n^{d+1}k^{d+1}}{\delta}\right)\log\left(\frac{nk^{d+1}}{\eps}\log\frac{1}{\delta}\right)\right)$ (see \Cref{def:bayes_dags_finitization}).
\end{proof}

\begin{lem}\label{lem:bn_net_kl}
	Let $\cC^\cG \subseteq \Delta([k]^n)$ be the set of all Bayes nets defined on any DAG $G \in \cG$, where $\cG$ is a finite set of DAGs on $[n]$, and all DAGs in $\cG$ have maximum indegree $d$. 
	For any $P^* \in \Delta([k]^n)$, the following holds with probability $\geq 1 - \delta$ over the samples used in the construction of $\cN^\cG_{\eps,\delta}$: 
	\[
	\min_{P \in \cN^\cG_{\eps,\delta}} \kl(P^* \| P) \leq \min_{P \in \cC^\cG} \kl(P^* \| P) + \eps.
	\]
\end{lem}

\begin{proof}
Take $\cG = \{G_1,\ldots,G_M\}$ and $\cN^\cG_{\eps,\delta} = \{\aoestdsamp{G_1}{S},\ldots,\aoestdsamp{G_M}{S}\}$ for a set of samples $S$ with $|S| \geq s_{\rm AO}(\eps, \delta)$. Applying \Cref{thm:add_one_guarantee_kl_uniform}, we have that for all DAGs $G_i \in \cG$ (where $i$ is not dependent on the samples $S$), with probability $\geq 1 - \delta$,
\begin{equation}\label{eqn:add_one_finit_close}
\kl(P^\ast \| \aoestdsamp{G_i}{S}) \leq \min_{G_i\textrm{-structured distributions } Q} \kl(P^\ast \| Q) + \varepsilon.
\end{equation}

Let $P^\prime$ denote a distribution in $\cC^\cG$ that minimizes $\kl(P^\ast \| P^\prime)$. By definition of $\cC^\cG$, $P^\prime$ will be $G_{i^\ast}$-structured for some $i^\ast \in [M]$ which is fixed as a function of $P^\ast$ (and independent of the samples $S$), which implies that,
\begin{equation*}
\begin{aligned}
\min_{P \in \cN^\cG_{\eps,\delta}} \kl(P^\ast \| P) &\leq \kl(P^\ast \| \aoestdsamp{G_{i^\ast}}{S}) \numrelop{1}{\leq} \kl(P^\ast \| P^\prime) + \varepsilon \numrelop{2}{=} \min_{P \in \cC^\cG} \kl(P^* \| P) + \eps.
\end{aligned}
\end{equation*}
Note that the equality (2) holds by definition of $P^\prime$ as the KL minimizer, and the penultimate inequality (1) holds since $P^\prime$ is $G_{i^\ast}$-structured, applying (\ref{eqn:add_one_finit_close}) with $i = i^\ast$.
\end{proof}
}

\subsection{Sample Complexity for learning Bayes Nets}
We can now collect the tools developed in Sections \ref{sec:reg_kl} and \ref{sec:discret} to prove \Cref{thm:intro_bayes} from the introduction. 
The algorithm that yields the first part of \Cref{thm:intro_bayes} (improper learning) is quite simple to describe (Algorithm \ref{alg:meta_improper} below), corresponding to \EWA\ forecaster with $\eta=1$. In the following theorem, we specify the $\cN$ and $T$ to take so as to obtain desired guarantees.

\begin{algorithm}[htbp]
\setcounter{AlgoLine}{0}
	\caption{\EWA-based learning for Bayes nets}\label{alg:meta_improper}
	\SetKwInOut{Input}{Input}
	\SetKwInOut{Output}{Output}
	\SetKwProg{myfn}{function}{}{}
    \Input{\,\, ${\cN} = \{P_1,\ldots,P_N\}, T$, hyperparameter $\eta > 0$.}
	  \Output{\,\, Sampler for $\widehat{P}$.}
	$w_{i,0} \gets 1 \mbox{ for each } i \in [N].$ \
    
	  \For{$t \gets 1$ to $T$}{%
		  $\mbox{Observe sample } x^{(t)} \sim P^\ast$. \
          
		\For{$i \in [N]$}{%
			 $w_{i,t} \gets w_{i,t-1} \cdot P_i(x^{(t)})^\eta$.
		}
	}
    
	\myfn{\rm\textsc{EWA-Sampler}()}{%
        Sample $t \gets [T]$ uniformly at random. \
		
        Sample $i \sim [N]$ with probability $\frac{w_{i,t-1}}{\sum_{j \in [N]} w_{j,t-1}}$. \
		
        \KwRet $x \sim P_i$. \
	}
	
    \KwRet \textsc{EWA-Sampler}
    \tcc{This is a sampler for $\wh{P}$.}
\end{algorithm}
\begin{algorithm}[htbp]
\setcounter{AlgoLine}{0}
	\caption{\RWM-based learning for Bayes nets}\label{alg:meta_proper}
	\SetKwInOut{Input}{Input}
	\SetKwInOut{Output}{Output}
	\Input{\,\, ${\cN} = \{P_1,\ldots,P_N\}, T$, hyperparameter $\eta > 0$.}
	\Output{\,\, $\wh{P} \in \cN$.}
	$w_{i,0} \gets 1 \mbox{ for each } i \in [N].$ \
	
    \For{$t \gets 1$ to $T$}{
        Sample $i_t \mbox{ from } [N]$ with $\Pr(i_t = i) = \frac{w_{i,t-1}}{\sum_{j \in [N]} w_{j,t-1}}$. \
		
        Observe sample $x^{(t)} \sim P^\ast$. \
		
        \For{$i \in [N]$}{
			$w_{i,t} \gets w_{i,t-1} \cdot P_i(x^{(t)})^\eta$.
		}
	}
    
    Sample $t$ uniformly from $[T]$. \
	
    \KwRet $\wh{P} \gets P_{i_t}$. \
\end{algorithm}

\adds{
\begin{theo}[Bayes net improper learning]\label{thm:bn_learn_imp}
	Let $\cC$ be the class of distributions over $[k]^n$ that can be defined as Bayes nets over any unknown DAG with $n$ nodes and in-degree $\leq d$. 
	
	\begin{itemize}
		\item
		There is an algorithm that for all $P^* \in \Delta([k]^n)$ and all $\eps\in (0,1)$, receives $m=O\left(\frac{d^3 n k^{d+1}}{\eps}\log^2 \frac{nk}{\eps}\right)$ i.i.d.~samples from $P^*$ and returns a distribution $\wh{P}$ that is with probability at least $2/3$, an $(\eps, 4)$-approximation for $P^*$ with respect to $\cC$.
		\item
		There is a realizable PAC-learner for $\cC$ with sample complexity $O\left(\frac{d^3 n k^{d+1}}{\eps \delta} \log^2\frac{nk}{\eps\delta}\right)$. 
		
		\item
		There is an agnostic PAC-learner for $\cC$ with sample complexity $$\widetilde{O}\left(\max\left\{\tfrac{d^4 n^4 \log^2(n)}{\eps^{4}} \log^2\left(\tfrac{nkd\log^2(nk/\eps\delta)}{\eps}\right) \log(1/\delta), \tfrac{d^3 n k^{d+1}}{\eps}\log\left(\tfrac{nk}{\delta}\right)\log\left(\tfrac{nk}{\eps}\log\tfrac{1}{\delta}\right)\right\}\right).$$
	\end{itemize}
	
\end{theo}
\begin{proof}
	Let $\cG$ be the collection of all DAGs on $n$ nodes that have in-degree $d$, then we can say that $|\cG| \leq (n!)^{d+1}$. Suppose $\cN = \cN^\cG_{\eps/8,\delta/2}$, using the notation of \Cref{def:bayes_dags_finitization}. By construction, we have 
	\begin{equation}\label{eqn:size}
		\log(|\cN|) \leq \log(|\cG|) \leq (d+1) \log n! \leq (d+1) n \log(n),
	\end{equation}
	and the sample complexity of the learning algorithm will be $s_{\rm AO}\left(\frac{\eps}{8}, \frac{\delta}{2}\right)$ (to construct $\cN$) + $T$ (the number of online rounds), where 
	\begin{equation}\label{eqn:net_samples}
		s_{\rm AO}(\eps/8,\delta/2) \leq O\left(\frac{d^3 n k^{d+1}}{\eps}\log\left(\tfrac{nk}{\delta}\right)\log\left(\tfrac{nk}{\eps}\log\tfrac{1}{\delta}\right)\right)\, \mbox{(using \Cref{def:bayes_dags_finitization} and \Cref{thm:add_one_guarantee_kl_uniform})}.
	\end{equation}
	
	By \Cref{lem:bn_net_kl} and the definition of $\cN$ above, 
	\begin{equation}\label{eqn:close}
		\min_{P \in \cN} \kl(P^* \| P) \leq \min_{P \in \cC} \kl(P^* \| P) + \eps/8,
	\end{equation}
	with probability $\geq 1 - \delta/2$ over the samples used to construct $\cN$.
	
	We begin with the first item. Here, we can take $\delta = \frac{1}{8}$ for the purpose of constructing $\cN$. We run \EWA-based learning (Algorithm \ref{alg:meta_improper}) using $\cN = \cN^{\cG}_{\eps/8,1/16}$ as the expert set, $\eta=1$, and loss function $\ell(P,x) = -\log P(x)$. By the first part of \Cref{lem:ewa_reg_final} and Markov's inequality, with probability at least $3/4$:
	\begin{align}
		\label{eqn:semi}
		\kl\left(P^* \| \frac1T \sum_{t=1}^T \wh{P}_t\right) \leq 4 \left(\min_{P \in \cN} \kl(P^* \| P)  + \frac{\log|\cN|}{T} \right). 
	\end{align}
	Using (\ref{eqn:size}), and combining (\ref{eqn:close}) and (\ref{eqn:semi}) with a union bound --- choosing $T = \frac{8 (d+1) n \log(n)}{\eps} \geq 8 \frac{\log(|\cN|)}{\eps}$ --- now ensures that right-hand-side in (\ref{eqn:semi}) is at most $4 \min_{P \in \cC} \kl(P^* \| P) + \eps$ with probability at least $1 - \left(\frac{1}{4} + \frac{1}{16}\right) \geq \frac{2}{3}$. The sample complexity will be dominated by $s_{\rm AO}(\frac{\eps}{8}, \frac{1}{16}) \leq O\left(\frac{d^3 n k^{d+1}}{\eps}\log^2\left(\frac{nk}{\eps}\right)\right)$ (see \Cref{def:bayes_dags_finitization}). 
	
	For the second item, let $\cN' = \cN^{\cG}_{\eps\delta/4,\delta/2}$ as the expert set for the algorithm. By \Cref{lem:bn_net_kl}, if $P^* \in \cC$, $\min_{P \in \cN'}\kl(P^*\| P) \leq \eps\delta/4$ with probability $\geq 1 - \delta/2$. We can now run the same argument as above with $\eps$ being $\eps\delta/4$ (applying Markov's inequality with error probability $\delta/2$) to obtain the result.
	
	We now prove the third item. Suppose $\eps, \delta \in (0,1)$. By \Cref{lem:bn_net_lb} and $|\cG| \leq (n!)^{d+1} \leq n^{nd}$, for every $P \in \cN$, $\min_{x \in [k]^n} P(x) \geq \Theta\left(\frac{\eps}{d^2 n k^{d+1}\log\left(\frac{nk}{\delta}\right)\log\left(\frac{nk}{\eps}\log\frac{1}{\delta}\right)}\right)^n = \tau$. Thus, $\log(1/\tau) \leq \Theta\left(n\log\left(\frac{d^2 n k^{d+1}\log^2\left(\frac{nk}{\eps\delta}\right)}{\eps}\right)\right)$. We run \EWA\ using $\cN$ as the expert set, $\eta = \frac{\eps}{4\sqrt{T}\log(1/\tau)\sqrt{\log(1/\delta)}} = \Theta\left(\frac{\eps}{4\sqrt{T}n\log\left(\frac{d^2 n k^{d+1}\log^2\left(\frac{nk}{\eps\delta}\right)}{\eps}\right)\sqrt{\log(2/\delta)}}\right)$ and the same loss function.  By the second part of \Cref{lem:ewa_reg_final}, we get that with probability at least $1-\delta/2$:
	\begin{equation}\label{eqn:full}
		\kl\left(P^* \| \frac1T \sum_{t=1}^T \wh{P}_t\right) \leq \min_{P \in \cN} \kl(P^* \| P) + \Theta\left(\tfrac{4 d n^2 \log(n) \log\left(\frac{d^2 n k^{d+1}\log^2\left(\tfrac{nk}{\eps\delta}\right)}{\eps}\right) \sqrt{\log(2/\delta)}}{\eps\sqrt{T}}\right) + \tfrac\eps2
	\end{equation}
	Using (\ref{eqn:size}) and (\ref{eqn:close}) along with a union bound, if $T = \Theta\left(d^4 n^4 \log^2(n) \eps^{-4} \log^2(\frac{nkd\log^2(nk/\eps\delta)}{\eps}) \log(1/\delta)\right)$, then the right-hand-side of (\ref{eqn:full}) is at most $ \min_{P \in \cC} \kl(P^* \| P) + \eps$ and the bound will hold with probability $\geq 1 - \delta$. Here, the sample complexity will be
	\begin{equation*}
	\begin{aligned}
	T + s_{\rm AO}(\eps/8,\delta/2) &\leq O\left(\frac{d^4 n^4 \log^2(n)}{\eps^{4}} \log^2\left(\frac{nkd\log^2(nk/\eps\delta)}{\eps}\right) \log(1/\delta)\right.\\
	&\left.+ \frac{d^3 n k^{d+1}}{\eps}\log\left(\frac{nk}{\delta}\right)\log\left(\frac{nk}{\eps}\log\frac{1}{\delta}\right)\right).
	\end{aligned}
	\end{equation*} 
\end{proof}
}
The above theorem is interesting in the context of improper learning of Bayes nets with constant in-degree $d$, because it gives a nearly-optimal sample complexity (optimal up to logarithmic factors) for realizable PAC learning with constant success probability (second part) and for getting a $(\eps,3)$-approximation with constant success probability (first part).

While the next theorem does not have significant advantages for general in-degree $d$ Bayesian networks due to the less efficient nature of Algorithm \ref{alg:meta_proper} and its higher sample complexity compared to established algorithms such as Chow-Liu for proper learning, we prove it to lay the groundwork for subclasses of Bayesian networks (such as trees, polytrees, and chordal graphs). For these subclasses, the sampling method employed in Algorithm \ref{alg:meta_proper} can be optimized to yield efficient algorithms for proper learning.

\adds{
\begin{theo}[Bayes net proper learning]\label{thm:bn_learn_proper}
	Let $\cC$ be the class of distributions over $[k]^n$ that can be defined as Bayes nets over any unknown DAG with $n$ nodes and in-degree $\leq d$. 
	There is a \emph{proper} agnostic $(\eps, \delta)$-PAC-learner for $\cC$ with sample complexity $O\left(\max\left\{\frac{d^3 n^3 \log^3\left(\frac{nkd}{\eps\delta}\right)}{\eps^2\delta^2}, \frac{d^3 n k^{d+1}}{\eps}\log^2\left(\frac{nk}{\eps\delta}\right)\right\}\right)$ for any $\delta \in (0, 1)$ and $\eps > 0$.
\end{theo}

\begin{proof}
	Define $\cN = \cN^{\cG}_{\eps/2,\delta/2}$ --- where $\cG$ is the collection of all DAGs on $n$ nodes with in-degree $\leq d$ --- the same way as in the proof of \Cref{thm:bn_learn_imp}.  We run the \RWM-based learning algorithm (Algorithm \ref{alg:meta_proper}) with $\cN$ as the expert set, $\eta = \sqrt{(8 \log |\cN|)/T}$, and the log loss function. By the second part of \Cref{lem:rwm_reg_final}, and noting that:
	\begin{enumerate}
		\item[(i)] The maximum loss at any round is at most (using \Cref{lem:bn_net_lb}; see also the third item in the proof of \Cref{thm:bn_learn_imp})
		\begin{equation*}
		\begin{aligned}
		L_{\max} &= \log(1/\tau) \leq O\left(n\log\left(\frac{d^3 n k^{d+1}\log^2\left(\frac{nk}{\eps\delta}\right)}{\eps}\right)\right).
		\end{aligned}
		\end{equation*}
		\item[(ii)] the size of the cover $\cN$ satisfies $\log |\cN| \leq \log |\cG| \leq (d+1) n \log(n)$.
	\end{enumerate}
	we have that, for any $\delta > 0$, with probability $\geq 1 - \delta/2$ over the random samples as well as the randomness of the algorithm, we have:
	\begin{equation*}
	\kl(P^\ast \| \wh{P}) \leq \min_{P \in \cN} \kl(P^\ast \| P) + \underbrace{C\left(\tfrac{\left(n\log\left(\frac{d^3 n k^{d+1}\log^2\left(\frac{nk}{\eps\delta}\right)}{\eps}\right)\right)\left(\sqrt{(d+1)n\log(n)} + \sqrt{\log(4/\delta)}\right)}{\delta\sqrt{T}}\right)}_{\triangleq \mathsf{ErrorBound}},
	\end{equation*}
	where $C$ is some constant $> 1$.
	
	Using the above inequality with a union bound, along with \Cref{lem:bn_net_kl}, the definition of $\cC$, and the fact that $\cN = \cN^{\cG}_{\eps/2,\delta/2}$, we have that
	\[
	\kl(P^\ast \| \wh{P}) \leq \min_{P \in \cC} \kl(P^\ast \| P) + \mathsf{ErrorBound} + \varepsilon/2,
	\]
	with probability $\geq 1 - \delta$ over the samples and the randomness of \RWM.
	
	Now if we choose any $T \geq \Theta\left(\frac{d^2 n^2 \log^2 \left(\frac{nkd\log^2(nk/\eps\delta)}{\eps}\right)\left(dn\log(n) + \log(4/\delta)\right)}{\eps^2\delta^2}\right)$, we get $\kl(P^\ast \| \wh{P}) \leq \varepsilon/2 + \varepsilon/2 = \varepsilon$ with probability $\geq 1 - \delta$ (this bound for $T$ uses the fact that $(\sqrt{x} + \sqrt{y})^2 \leq 2(x + y)$ by the AM-GM inequality). The total sample complexity of the algorithm will be $T +s_{\rm AO}(\eps/2, \delta/2)$ (see \Cref{thm:add_one_guarantee_kl_uniform} and \Cref{lem:bn_net_kl}). This gives us the $$O\left(\frac{d^3 n^3 \log^3\left(\frac{nkd}{\eps\delta}\right)}{\eps^2\delta^2} + \frac{d^3 n k^{d+1}}{\eps}\log^2\left(\frac{nk}{\eps\delta}\right)\right)$$ sample complexity for proper agnostic $(\eps,\delta)$-PAC learning.
\end{proof}
}

\section{Learning Chordal-structured distributions}\label{sec:learnchordaldist}

In this section, we show that {\em chordal-structured distributions} can be learnt efficiently. Before proceeding to describe our results, we first define some notions that will be used in our results and proofs. 

\subsection{Preliminaries about Chordal Graphs} 
Given an undirected graph $G=(V,E)$ and subsets $S, T \subseteq V$, we let $G[S]$ denote the induced subgraph of $G$ on $S$, $E(S)$ denote the edge set of $G[S]$, and $E(S, T)$ denote the set of edges with one endpoint in $S$ and the other in $T$. For a vertex $v$ of $G$, let $\nbr_G(v)$ denote the set of adjacent vertices of $v$ in $G$ (vertices $u$ such that $\{u,v\} \in E(G)$).

An undirected graph is {\em chordal} if every cycle of length at least 4 contains a chord, that is an edge connecting two vertices of the cycle which is not part of the cycle. 
\begin{defi}[Clique tree]
Let $G=(V,E)$ be a graph with vertex set $V$ and edge set $E$. The \emph{clique tree}, denoted by $\tg$, of $G$ is a tree that has the maximal cliques of $G$ as its vertices and for every two maximal cliques $C$ and $C'$, each clique on the path from $C$ to $C'$ in $\tg$ contains $C\cap C'$. 
$\tg$ has the \emph{induced subtree} property: for every vertex $v \in V$, the set of nodes in $\tg$ that contains $v$ forms a connected subtree of $\tg$.   The {\em treewidth} of $G$ equals one less than the size of the largest maximal clique in $G$.
\end{defi}

Let $G$ be a chordal graph, and let $\tg$ be a clique tree of $G$. Fix an arbitrary maximal clique $C_r \in V(\tg)$ and we view $\tg$ as a tree rooted at $C_r$. We denote this {\em rooted clique tree} as $\tg_{C_r}$ and as $\cT_{C_r}$ if $G$ is clear from context. For any maximal clique $C$, let $\pa_{\cT_{C_r}}(C)$ denote the parent of $C$ in the rooted tree $\cT_{C_r}$, and let $\cT_{C}$ denote  the subtree of $\cT_{C_r}$ rooted at $C$. Let $V[\cT_C]$ denote the vertex set of $G[\cT_C]$, that is, $\bigcup_{C' \in V(\mathcal{T}_C)} C'$. For notational convenience, we will use $\cT_{C}$ to denote both the subtree of $\tg$ as well the vertex set $V[\cT_C]$ when the usage will be clear from the context. Thus, $G[\mathcal{T}_C]$ denotes the subgraph of $G$ induced by the vertices in $V[\mathcal{T}_C]$.

\begin{defi}[Separator set]
Let $G$ be a chordal graph, and let $\cT_{C_r}$ be a rooted clique tree of $G$ for a maximal clique $C_r$. For $C \in V(\cT_{C_r})$, the \emph{separator} of $C$ with respect to $\cT_{C_r}$ is defined as follows:
$$\sep(C)= C \cap \mathsf{Pa}_{\cT_{C_r}}(C)$$
\end{defi}

Next we define the notion of link set, which is crucially used in our proofs.

\begin{defi}[Link set]
Let $G=(V,E)$ be a chordal graph, and $\cT_{C_r}$ be a rooted clique tree of $G$ for a maximal clique $C_r$. For $C \in V(\cT_{C_r})$, the \emph{link set of $C$} is defined as follows:
$$\lnk(C)= E(C, V[\cT_C])$$
\end{defi}

An {\em orientation} of an edge set $F$ assigns a direction to each edge in $F$; an orientation is {\em acyclic} if it does not give rise to a directed cycle. The {\em indegree} of an orientation is the maximum number of incident edges which are oriented inward at any vertex. Two acyclic orientations on edge sets $F$ and $F'$ are said to be {\em consistent} with each other if they agree on the edges in $F \cap F'$.
\begin{defi}[Indegree-Bounded Acyclic Orientations]
Let $G = (V,E)$ be an undirected graph.
For any subset of edges $F\subseteq E$ and integer $d\geq 0$, $\ao_d(F)$ denotes the set of all acyclic orientations of $F$ with indegree at most $d$.   Given an orientation $\mathcal{O}$ that assigns orientations to a subset of $E$, we let $\ao_d(F; \mathcal{O})$ denote the subset of $\ao_d(F)$ that is consistent with $\mathcal{O}$. If $F$ corresponds to the edge set of a subgraph $H$, then we also use the notations $\ao_d(H)$ and $\ao_d(H, \cO)$ respectively. 
\end{defi}
\begin{figure}[htpb]
\caption{In the left panel, a chordal graph and a clique tree decomposition with reference clique $C = \textrm{DEG}$. In the left panel, the edges of $\lnk(C)$ are in red and the vertices $V[\cT_C]$ are in bold. In the right panel, the nodes of $\cT_C$ are in green, and the separator vertices in each node of the clique tree are colored in blue. The same chordal graph and a different clique tree decomposition with reference clique $C = \textrm{ACDE}$. The colors have the same meaning as in the left.\label{fig:clique_tree_decomp}}
\subfigure{
        \label{fig:chordal_sf_1}
        \includegraphics[width=0.19\textwidth]{figures/ex_chordal_graph.tikz}
        \hspace*{5pt}
        \includegraphics[width=0.25\textwidth]{figures/ex_clique_tree_decomp.tikz}      
}\subfigure{
    \label{fig:chordal_sf_2}
    \includegraphics[width=0.19\textwidth]{figures/ex_chordal_graph2.tikz} 
    \hspace*{5pt}
    \includegraphics[width=0.25\textwidth]{figures/ex_clique_tree_decomp2.tikz}
}
\end{figure}

We will use below some standard observations about chordal graphs and acyclic orientations; for the sake of completeness, we give their proofs.

\begin{lem}[Lemma 1 of \cite{bezakova2022counting}, restated]\label{lem:chordal_disjoint_indegree}
Let $G=(V,E)$ be a chordal graph and consider a clique tree $\mathcal{T}_{C_r}$ of $G$ rooted at a node $C_r$. Let $C$ be a node in $\mathcal{T}_{C_r}$ and $C_1, \ldots, C_{\ell}$ be its children in $\mathcal{T}_{C_\ell}$. Then the edge sets of the graphs $G[\cT_{C_1} \setminus \sep(C_1)], \ldots, G[\cT_{C_\ell} \setminus \sep(C_\ell)]$ are mutually disjoint.
\end{lem}

\begin{proof}
We will prove this by contradiction. Let us assume that there exists $i \neq j$ with $i, j \in [\ell]$ such that $G[\cT_{C_i} \setminus \sep(C_i)]$ and $G[\cT_{C_j} \setminus \sep(C_j)]$ share an edge $e=\{u,v\}$. This implies that both $G[\mathcal{T}_{C_i}]$ and $G[\mathcal{T}_{C_j}]$ contain the edge $e$. However, this would imply $C \supseteq \{u,v\}$ by the \emph{inducted subtree property} (since any pair of nodes in $\cT_{C_i}$ and $\cT_{C_j}$ can only be connected in $\tg$ through $C$), and thus that $e$ will be present in both $E(\sep(C_i))$ and $E(\sep(C_j))$. This contradicts our assumption on $e$.
\end{proof}

\begin{lem}[See Theorem 11 in Chapter 4 of \cite{sun2022efficient}]\label{lem:chordal_bij_indegree}
Let $G$ be a chordal graph which is acyclically orientable with indegree $\leq d$, and consider a rooted clique tree $\mathcal{T}_{C_r}$ of $G$.
Consider a non-leaf node $C \in \cT_{C_r}$ and $C_1, \ldots, C_{\ell}$ be the set of children of $C$ in the rooted tree. Consider an acyclic orientation $\mathcal{O}_C$ of $\lnk(C)$ with indegree $\leq d$, and let $\mathcal{O}_{\sep(C_i)}$ be the orientation $\mathcal{O}_C$ restricted to $E(\sep(C_i), \cT_{C_i})$ for every $i \in [\ell]$. Then there exists a bijection between $\ao_d(\cT_C, \mathcal{O}_C)$ and $\ao_d(\cT_{C_1}, \mathcal{O}_{\sep(C_1)}) \times \ldots \times \ao_d(\cT_{C_\ell}, \mathcal{O}_{\sep(C_\ell)})$.     
\end{lem}

\begin{proof}
Consider any acyclic orientation $\mathcal{O}$ of $\cT_C$ with indegree $\leq d$ consistent with a given acyclic orientation $\mathcal{O}_C$ of $\lnk(C)$. Any subset of an acyclic orientation $\mathcal{O}$ will be acyclic as well, and indegree bounds will be preserved. Hence $\mathcal{O}$ restricted to each $\cT_{C_i}$ will give acyclic orientations $\mathcal{O}_i$ of $\cT_{C_i}$ with indegree $\leq d$. Given these consistent orientations $\mathcal{O}_1,\ldots,\mathcal{O}_\ell$ of the subtrees, and knowing that $\mathcal{O}$ is consistent with $\mathcal{O}_C$, we can reconstruct $\mathcal{O}$. This suffices to argue the injectivity.

In order to prove surjectivity, let us consider acyclic orientations $\mathcal{O}_i \in \ao_d(\cT_{C_i}, \mathcal{O}_{\sep(C_i)})$ for every $i \in [\ell]$.
Now let us consider the following orientation $\mathcal{O}$ by taking the union of $\mathcal{O}_C$ and $\bigcup_{i=1}^{\ell} \mathcal{O}_{i}$. Following \Cref{lem:chordal_disjoint_indegree}, we know that $\mathcal{O}$ is consistent with $\mathcal{O}_C$. Now we would like to prove that $\mathcal{O}$ is an acyclic orientation and has indegree $\leq d$.

We will prove acyclicity by contradiction. Assume that there is at least one cycle in $\mathcal{O}$. Consider the shortest cycle among all those cycles in $\mathcal{O}$. The cycle must contain at least two edges in $E(C)$, because if it contains one or no edges of $E(C)$, then it is contained inside some $\cT(C_i)$ which is impossible since $\cO_i$ is acyclic. Consider two vertices $u,v \in C$ which are in the cycle but not adjacent in the cycle. In $C$, there exists an edge between $u$ and $v$, with $\cO_C$ orienting it either $(u,v)$ or $(v,u)$. Whichever the case, we can shorten the cycle by taking a shortcut on the edge, which contradicts the assumption that we are considering the shortest cycle. This implies that $\mathcal{O}$ is acyclic.

To argue that the indegree of $\mathcal{O}$ is $\leq d$, note that any $v$ which is common to $\cT_{C_i}$ and $\cT_{C_j}$ for $i \neq j$ will belong to $C$ as well. Hence the orientation of all incident edges of $v$ will be fixed by $\mathcal{O}_C$ itself, and combining $\mathcal{O}_1,\ldots,\mathcal{O}_\ell$ will not increase the indegree beyond $d$ since they are all consistent with $\mathcal{O}_C$. Thus the mapping is surjective as well. Combining the above, we have the proof of the lemma. 
\end{proof}

\begin{lem}\label{lem:chordal-no-edge-cross-child}
Let $C$ and $C_i$ be two cliques of $G$ so that $C$ is a node in $\tg$ and $C_i$ is a child of $C$ in $\tg$. Then there are no edges between $V[\cT_{C_i}]\setminus C$ and $C \setminus C_i$.
\end{lem}

\begin{proof}
Assume for the sake of contradiction that $\{u,v\}$ is an edge between $u \in V[\cT_{C_i}] \setminus C$ and $v \in C \setminus C_i$. 
In particular, say $u \in K$ where $K$ is a clique in $\cT_{C_i}$. Then, by the definition of the clique tree decomposition, (i) the edge is contained neither in $C$ nor in any of the cliques in the subtree rooted at $C_i$ (since $u \not\in C$ and $v \in C \setminus C_i$), but (ii) the edge is part of some maximal-clique $C'$. Hence, $C'$ and $K$ must be separated by $C$ in the clique tree. But this is a contradiction, since $u \in C'$ and $u \in K$ but $u \not \in C$. 
\end{proof}

\subsection{\EWA\ and \RWM\ with Chordal Experts}
In this section, we consider distributions on chordal graphs which can be oriented acyclically with indegree $\leq d$. This assumption is sufficient to bound the size of each maximal clique.
\begin{rem}\label{rem:indegree_bounds_max_clique_size}
Let $G$ be an undirected chordal graph. If there exists an acyclic orientation of $G$ with indegree $\leq d$, then for any clique tree decomposition $\cT_{C_r}$ of $G$ (rooted at a maximal clique $C_r$), all the nodes of $\cT_{C_r}$ (maximal cliques of $G$) have cardinality $\leq d + 1$.
\end{rem}
\begin{proof}
	By assumption, we have an acyclic orientation with indegree $\leq d$ for each clique $C \in V(\cT_{C_r})$. If we topologically order the nodes of $C$, the last node will have indegree $|C|-1 \leq d$, which implies $|C| \leq d + 1$.
\end{proof}

We will establish the following theorem. 

\begin{theo}
Let $G$ be an undirected chordal graph, and suppose $k$ and $d$ are fixed constants. Let $\cL^G_d$ be the family of distributions over $[k]^n$ that can be defined as Bayes nets over DAGs having skeleton $G$ and with indegree $\leq d$. Given sample access from $P^* \in \Delta([k]^n)$, and a parameter $\eps > 0$,
there exist the following:
\begin{itemize}
    \item[(i)] An agnostic PAC-learner for $\cL^G_d$ using \adds{$\widetilde{O}\left(\max\left\{\frac{d^4 n^4}{\eps^4}\log^4\left(\frac{nkd}{\eps}\right)\log^2(1/\delta), \tfrac{d^3 n k^{d+1}}{\eps}\log^2\left(\frac{nk}{\eps\delta}\right)\right\}\right)$} samples that is improper and returns an efficiently-samplable mixture of distributions from $\cL^G_d$.

    \item[(ii)] An agnostic PAC-learner \learnchordal for $\cL^G_d$ using \adds{$\widetilde{O}\left(\max\left\{\frac{d^3 n^3}{\eps^2\delta^2},\frac{d^3 nk^{d+1}}{\eps}\log^2\left(\frac{nk}{\eps\delta}\right)\right\}\right)$} samples and $\mathrm{poly}(n)$ running time that is proper and returns a distribution from $\cL^G_d$.
\end{itemize}
\end{theo}

The sample complexity guarantees follow from \Cref{thm:bn_learn_imp} and from \Cref{thm:bn_learn_proper}. What remains to be justified is that the algorithms run efficiently and the distribution output by the improper learning algorithm can be sampled efficiently. Below, let $x^{(1)}, \dots, x^{(T)}$ be the samples drawn from $P^*$ \adds{for the online phase, and let $S_{\cN}$ denote the set of $s_{\rm AO}(\eps, \delta)$ samples --- for appropriate choices of $\eps, \delta$ as used in \Cref{thm:bn_learn_imp} and \Cref{thm:bn_learn_proper} --- used to construct the finite expert set $\cN$ used in the algorithms.}. 

Let $\cT^G_{C_r}$ be a clique tree decomposition of $G$ rooted at a maximal clique $C_r$. We can assume without loss of generality that $G$ has at least one acyclic orientation with indegree $\leq d$; otherwise, $\cL^G_d$ is empty and the problem is trivial. Hence, by Remark \ref{rem:indegree_bounds_max_clique_size}, every node in $\cT^G_{C_r}$ has at most $d+1$ vertices  of $G$.

\begin{cl}
For any $C \in V(\cT_{C_r})$,  $|\ao_d(\lnk(C))| \leq \binom{n+d}{d}^{d+1}$.
\end{cl}

\begin{proof}
Note that there are at most $d+1$ vertices in $C$ and each such vertex has at most $n$ incident edges, of which at most $d$ can be incoming. Combining the above, we have the claim.    
\end{proof}

\dels{
For an orientation $\mathcal{O}$ of $G[\cT_C]$ and a node $v \in V[\cT_C]$, let $\mbox{in}_C(v, \mathcal{O})$ denote the set of in-neighbors of $v$ in $G[\cT_C]$ with respect to $\mathcal{O}$. We define the \emph{weight}, with respect to the clique $C$, of a node $v  \in V[\cT_C]$, acyclic orientation $\mathcal{O}$ of $G[\cT_C]$ and time-step $t \in [T]$ as follows:
\begin{equation}\label{eqn:chordal_node_weight}
\wt_C(v, \mathcal{O}, t) \triangleq \sum_{q \in \cN_{\eps/n}^{v \mid \text{in}_{C}(v, \mathcal{O})}} \prod_{s=1}^t \exp_\eta \log q\left(x^{(s)}_{v} \mid x^{(s)}_{\text{in}_{C}(v, \mathcal{O})}\right).
\end{equation}
}
\adds{
For an orientation $\mathcal{O}$ of $G[\cT_C]$ and a node $v \in V[\cT_C]$, let $\mbox{in}_C(v, \mathcal{O})$ denote the set of in-neighbors of $v$ in $G[\cT_C]$ with respect to $\mathcal{O}$. We define the \emph{weight}, with respect to the clique $C$, of a node $v  \in V[\cT_C]$, acyclic orientation $\mathcal{O}$ of $G[\cT_C]$ and time-step $t \in [T]$ as follows:
\begin{equation}\label{eqn:chordal_node_weight}
\wt_C(v, \mathcal{O}, t) \triangleq \prod_{s=1}^t \exp_\eta \log \aonodedist{S_{\cN}}{v}{\text{in}_{C}(v,\mathcal{O})}\left(x^{(s)}_{v} \mid x^{(s)}_{\text{in}_{C}(v, \mathcal{O})}\right).
\end{equation}
where $\aonodedist{S_{\cN}}{v}{\text{in}_{C}(v,\mathcal{O})}$ is the add-one distribution for the node $v$ with nodes $\mbox{in}_C(v, \mathcal{O})$ chosen as the parents of $v$ (see \Cref{def:add_one_distribution}), computed from the first set of samples $S_{\cN}$ used to construct $\cN$.

At any point of our sampling algorithm, we only compute the add-one distribution for a node $v$ with respect to a fixed set of $\leq d$ parents for the node $v$, which are fixed (and adjusted) recursively. Note that this distribution can be computed by performing a single pass over the $\mathrm{poly}(n, k, \eps, \log(1/\delta))$ many samples (for constant $d$).
}

For all $C \in V(\tg)$ and for all $\cO_C \in \ao_d(\lnk(C))$, and for all $t \in [T]$, we will store an entry in a table $\mathsf{Table}[C, \cO_C, t]$. 

\begin{equation}\label{eqn:chordal_table_main}
\mathsf{Table}[C,\cO_C,t]
\triangleq \sum_{\mathcal{O} \in \ao_d(G[\cT_C], \cO_C)} \prod_{v \in V[\cT_C]} \wt_C(v, \mathcal{O}, t).
\end{equation}

Note that, according to this definition, $\mathsf{Table}[C_r,\emptyset,t]$ gives the total weight of all indegree $\leq d$ acyclic orientations of $G$ after observing $t$ samples. The sum is over an exponential-sized set; nevertheless, we will be able to use dynamic programming to compute it efficiently.

Let us start with the case that $C$ is a leaf node of $\tg$.
Note that, if $C$ is a leaf node of $\tg$, $\ao_d(G[\cT_C], \cO_C) = \{\cO_C\}$, where $\cO_C$ is an acyclic orientation of $C$.
\dels{So, the corresponding table entry can be directly computed for any $\cO_C$ in  $\mathsf{poly}(n, t, 1/\eps)$ time for constant $d$ and $k$.}\adds{So, the corresponding table entry can be directly computed for any $\cO_C$ in  $\mathsf{poly}(n, t,k, 1/\eps)$ time for constant $d$.} This completes the base case of our dynamic programming.

Now let us assume that $C$ has $\ell$ children $C_1, \ldots, C_{\ell}$ in $\cT^G$. For $\cO_C \in \ao_d(\lnk(C))$, we show that $\mathsf{Table}[C,\cO_C,t]$ can be inductively computed in terms of the table entries $\mathsf{Table}[C_i,\cO_{C_i},t]$ for $1\leq i \leq \ell$, where each $\cO_{C_i} \in \ao_d(\lnk({C_i}))$, so that we obtain a bottom-up dynamic programming algorithm for counting weighted acyclic orientations with indegree $\leq d$ starting from the leaf nodes of $\tg$. By \Cref{lem:chordal_bij_indegree}, we first break each acyclic orientation $\mathcal{O}$ of $\cT_C$ into $\mathcal{O}_1,\ldots,\mathcal{O}_\ell$ (acyclic orientations of the child subtrees $\cT_{C_i}$) which can be combined consistently with an orientation $\cO_C \in \ao_d(\lnk(C))$.%

The table entry $\mathsf{Table}[C, \cO_C, t]$ can be expressed as follows:
\begin{eqnarray}
\mathsf{Table}[C, \cO_C, t] &\triangleq& \sum_{\mathcal{O} \in \ao_d(G[\cT_C], \cO_C)} \prod_{v \in V[\cT_C]} \wt_C(v, \mathcal{O},t) \notag \\
&=& \sum_{\substack{\mathcal{O}_1 \in \ao_d(G[\cT_{C_1}],  \cO_C),\\[-.2em]
\vdots\\
\mathcal{O}_\ell \in \ao_d(G[\cT_{C_{\ell}}], \cO_C)}} \prod_{v \in V[\cT_C]} \wt_C(v, \cup_{j=1}^\ell \mathcal{O}_j ,t) \notag\\
&=&  \prod_{v \in C} \wt_C(v, \mathcal{O}_C, t) \sum_{\substack{\mathcal{O}_1 \in \ao_d(G[\cT_{C_1}],  \cO_C),\\[-.2em]
\vdots\\
\mathcal{O}_\ell \in \ao_d(G[\cT_{C_{\ell}}], \mathcal{O}_C)}} \left(\prod_{i=1}^{\ell} \prod_{v \in V[\cT_{C_i}] \setminus C} \wt_{C_i}(v,\mathcal{O}_i,t)\right),%
 \label{eqn:chordalcountbasic}
\end{eqnarray}
where in the last line, we used the fact that $V[\cT_{C_i}]\setminus C$ are disjoint for distinct $i$, and also that, there are no edges between $V[\cT_{C_i}]\setminus C$ and $C \setminus C_i$ by \Cref{lem:chordal-no-edge-cross-child}. 

For each $i$, let $\mathsf{Cons}_i(\cO_c) \subseteq \ao_d(\lnk({C_i}))$ be the set of acyclic orientations of $\lnk(C_i)$ of indegree at most $d$ that are consistent with $\cO_C$.  From (\ref{eqn:chordalcountbasic}), we can write $\mathsf{Table}[C, \cO_C, t]$ as:
\begin{equation}\label{eqn:chordalcount2}
\begin{split}
&\prod_{v \in C \setminus (C_1 \cup \ldots \cup C_{\ell})} \wt_C(v, \mathcal{O}_C, t)~ \times \\
&~~ \sum_{\forall i, \cO_{C_i}\in \mathsf{Cons}_i(\cO_C)}
\sum_{\forall i, \mathcal{O}_i \in \ao_d(G[\cT_{C_i}],  \cO_{C_i})}\left(
 \left(\prod_{i=1}^{\ell} \prod_{v \in V[\cT_{C_i}] \setminus \sep(C_i)} \wt_{C_i}(v,\mathcal{O}_i,t)\right)\right. \\
&~~~~~~~~~~~~~~~~~~~~~~~~~~~~~~~~~~~~~~~~~~~~~~~~~~~~~~~~~~~~\left.\cdot \left(\prod_{v \in \cup_{i=1}^\ell \sep(C_i)} \wt_C(v, \mathcal{O}_C,t)\right)\right).
\end{split}
\end{equation}

We can now state a recurrence for $\mathsf{Table}[C, \cO_C, t]$ using the notation above.

\begin{lem}\label{lem:chordalrec}
For any $C \in V(\cT_{C_r}), \cO_C \in \ao_d(\lnk(C)), t \in [T]$, if $C_1, \dots, C_\ell$ are the children of $C$ in $\cT_{C_r}$, then:
\[
\mathsf{Table}[C, \cO_C, t] = \prod_{v \in C \setminus \cup_i C_i} \mathsf{wt}_C(v, \mathcal{O}_C, t)  \sum_{\substack{\forall i, \cO_{C_i}\\ \in \mathsf{Cons}_i(\cO_C)}} \left(\prod_{i=1}^\ell \mathsf{Table}[C_i, \cO_{C_i}, t] \right) \Xi(\cO_C, \cO_{C_1}, \dots, \cO_{C_\ell}),
\]
where:
\[
\Xi(\cO_C, \cO_{C_1}, \dots, \cO_{C_\ell}) = \prod_{v \in \cup_i \sep(C_i)}  \frac{\mathsf{wt}_C(v, \cO_C, t)}{\prod\limits_{\substack{j \in [\ell] : C_j \ni v}} \mathsf{wt}_{C_j}(v, \cO_{C_j}, t)}.
\]
\end{lem}
\begin{proof}
We argue by induction. For the base case, when $C$ is a leaf, we have $\mathsf{Table}[C, \cO_C, t] = \prod_{v \in C} \mathsf{wt}_C(v, \mathcal{O}_C, t)$ which agrees with (\ref{eqn:chordal_table_main}). Otherwise, we inductively use (\ref{eqn:chordal_table_main}) on each $\mathsf{Table}[C_i, \cO_{C_i}, t]$ and check that (\ref{eqn:chordalcount2}), and therefore (\ref{eqn:chordal_table_main}), holds for $\mathsf{Table}[C, \cO_C, t]$. The only fact we need here is that if $v \in \sep(C_i)$, all its incident edges belong to $\lnk(C_i)$ and hence their orientations are fixed by $\cO_{C_i}$. 
\end{proof}

The pseudocode for counting the number of acyclic orientations with maximum indegree $d$ is described in Algorithm \ref{alg:countchordalindegree}.

\begin{algorithm}
    \setcounter{AlgoLine}{0}
    \caption{\countchordal ($G, \cT_{C_r}, \eta, S_{\cN}, \mathsf{SampleList}=[x^{(1)},\ldots,x^{(t)}], t$)}\label{alg:countchordalindegree}
	\KwIn{A known chordal skeleton $G$ which is acyclically orientable with indegree $d$, a rooted clique tree $\cT_{C_r}$ of $G$, hyperparameter $\eta$, a list of samples $[x^{(1)},\ldots,x^{(t)}]$ from $P^*$, time step $t$.}
	\KwOut{$3$-dimensional table $\mathsf{Table}$.}
	
	$\mathsf{Table} \gets \emptyset$. \
		
	$\mathsf{NumberofLevels} \gets$ the number of levels in $\tg$. \
	
	\For{every leaf node $C \in V(\cT_C)$}{
		\For{every $\cO_C \in \ao_d(C)$}{
			\For{every $v \in V(G_C)$}{
				Compute $\wt_C(v, \mathcal{O}_C, t)$ from the samples $S_{\cN}$ according to \Cref{eqn:chordal_node_weight}.
			}
			$\mathsf{Table}[C,\cO_C,t] = \prod_{v \in C} \wt_C(v, \mathcal{O}_C, t)$ according to \Cref{eqn:chordal_table_main}. 
		}
	}
	
	// The root is at level $0$, and the lowest leaf is at level $\mathsf{NumberOfLevels}$.\\
	\For{$j=\mathsf{NumberOfLevels}-1$ to $0$}{
		$\cS \gets$ the nodes at level $j$ in $\tg$.\\
		\For{every non-leaf $C \in \cS$}{
			Let the children of $C$ be $C_1,\ldots,C_\ell$, which are at level $j+1$ by definition.\\
			\For{every $\cO_C \in \ao_d(\lnk(C))$}{
				Compute $\mathsf{Table}[C,\cO_C,t]$ according to \Cref{lem:chordalrec}.\
			}
		}
	}
	
	Return $\mathsf{Table}$. \
	
\end{algorithm}

\begin{algorithm}
\setcounter{AlgoLine}{0}
\caption{\sampchordal($G,d,\eta,S_{\cN}, \mathsf{SampleList} = [x^{(1)},\ldots,x^{(t)}], t$)}
\label{alg:sampchordal2}
\KwIn{A chordal skeleton $G$ which is acyclically orientable with indegree $d$, hyperparameter $\eta>0$, a list of samples $x^{(1)},\ldots,x^{(t)}$ from distribution $P^*$ on $[k]^n$, time step $t \in [T]$}
\KwOut{A DAG with skeleton $G$.}

\LinesNumbered

Construct a clique tree $\tg$ of $G$ using the algorithm of \cite{rose1976algorithmic}.  \

$r \gets \mathsf{root}(\tg)$. \

$L \gets \mathsf{Leaves}(\tg)$. \

$\stk \gets \emptyset$. \

Call the counting algorithm \countchordal($G, \tg, \eta, S_{\cN}, \mathsf{SampleList}$) to obtain a 3-d DP table $\mathsf{Table}$. \

Push $(r, \emptyset)$ to the $\stk$.\

\While{Stack is not empty}{
Pop $(B, \mathcal{O}_p)$ from $\stk$, where $B$ is a node in $\tg$ and $\mathcal{O}_p \in \ao_d(\lnk(\pa_{\tg}(B)))$ if $B \neq r$.\

\For{every orientation $\mathcal{O} \in \ao(\lnk(B), \mathcal{O}_p)$}{
$$P_B(\mathcal{O}) \gets \frac{\mathsf{Table}[B, \mathcal{O}, t]}{\sum\limits_{\mathcal{O}' \in \ao(\lnk(B),\mathcal{O}_p)} \mathsf{Table}[B, \mathcal{O}',t]}$$
}

Sample $\mathcal{O}_B \in \ao(\lnk(B))$ from the distribution $P_B$.\

Fix the orientation of $\lnk(B)$ to be $\mathcal{O}_B$.\

\For{each child $C$ of $B$ in $\tg$}{
Push $(C, \mathcal{O}_B)$ to $\stk$.\
}

}

Return all the orientations of the edges in $G$. \

\end{algorithm}

\begin{algorithm}
\setcounter{AlgoLine}{0}
\caption{\learnchordal (RWM-based proper learning of chordal-structured distributions)}\label{alg:meta_proper_chordal_indegree}
\SetKwInOut{Input}{Input}
\SetKwInOut{Output}{Output}
\dels{\Input{\,\, Chordal skeleton $G$,  indegree parameter $d$, $\cN = \prod_{i \in [n]} \prod_{S \subseteq {\rm Nbr}_G(i)} \cN^{X_i \mid X_Ss}_{\eps/n}$, hyperparameter $\eta > 0$, $T$.}}
\adds{\Input{\,\, Sample access to $P^\ast$, Chordal skeleton $G$,  indegree parameter $d$, hyperparameter $\eta > 0$, $T$.}}
\Output{\,\, A chordal-structured distribution $\widehat{P}$.}

\adds{
Let $S_{\cN} \gets s_{\rm AO}(\eps,\delta)$ samples from $P^\ast$ (see \Cref{thm:add_one_guarantee_kl_uniform}).
}

$\mbox{Sample } t \gets \{1,\ldots,T\} \mbox{ uniformly at random.}$ \

$\mathsf{OnlineSampleList} \gets \emptyset$. \

	\For{$s \gets 1$ to $t$}{
$\mbox{Observe sample } x^{(s)} \sim P^\ast$. \

$\mathsf{OnlineSampleList} \gets \mathsf{OnlineSampleList} \cup \{x^{(s)}\}$. \

	}

Call \sampchordal ($G, d, \eta, S_{\cN}, \mathsf{OnlineSampleList}, t$) to obtain a DAG $D$. \

\dels{For each $i \in [n]$, sample $\wh{p}_i$ as described in text.\
}

\dels{
\KwRet the Bayes net $\wh{P} = (\wh{p}_1, \dots, \wh{p}_n)$ on the DAG $D$. \
}

\adds{
\KwRet the Bayes net $\wh{P} \triangleq$ the add-one distribution with structure $D$ computed from samples $S_{\cN}$ (\Cref{def:add_one_distribution}). \
}

\end{algorithm}

\color{black}

Once $\tab[\cdot,\cdot,t]$ has been constructed, we can sample an orientation of the chordal graph $G$ at time step $t$. The idea is to go down from the root of $\cT_{C_r}$ to its leaves, level-by-level, with each iteration orienting the link of a maximal clique while being consistent with the orientations sampled so far. For each clique $C \in V(\cT_{C_r})$ with parent $C_p$, we suppose that the edges in $\lnk(C_p)$ have been oriented fully. We use the table $\tab[\cdot,\cdot,t]$ to sample in a consistent way with the already oriented edges.
To argue correctness, we again appeal to \Cref{lem:chordal_bij_indegree} to see that once a parent clique has been oriented, its children can be oriented independently as long as they are consistent with the parent's orientation. Algorithm \ref{alg:sampchordal2} gives the pseudocode.

Suppose $D$ is the orientation of $G$ that is sampled by \sampchordal for a uniformly chosen value of $t \in [T]$. \dels{Then, we can readily sample a Bayes net on $D$ by following the same strategy described in the proof of \Cref{{lem:ewa_poly_tree}}. Namely, for each $i \in [n]$, we sample $\wh{p}_i(y_i, y_{\pa_D(i)})$ from $\cN_{\eps/n}^{X_i \mid X_{\pa_D(i)}}$ independently, where:
\[
\Pr[\wh{p}_i = p] \propto \prod_{s=1}^t \exp_\eta \log p\left(x^{(s)}_i, x^{(s)}_{\pa_D(i)}\right).
\]
Finally, the sampled Bayes net $\wh{P}$ is obtained by letting the conditional probability distribution at each node $i$ be $\wh{p}_i$.}
\adds{Following \Cref{def:add_one_distribution}, we return the add-one distribution $\wh{P}$, with respect to the structure $D$ and the samples $S_{\cN}$, as the Bayes net.}
The proper learning algorithm is summarized as pseudocode in Algorithm \ref{alg:meta_proper_chordal_indegree}. The improper learning algorithm is similar, but it uses \EWA\ in place of \RWM, and so, it samples a random $\wh{P}$ when generating samples instead of during the learning phase.

\section{Learning Tree-structured distributions}\label{sec:learntreedist}

Let us start with the definition of arborescence, which will be used throughout this section.

\begin{defi}[Arborescence]\label{def:arborescence}
	A directed graph $G = (V, E)$ is an \emph{out-arborescence} rooted at $v \in V$ if its skeleton (underlying undirected graph) is a tree (acyclic and connected) and there is a unique directed path from $v$ to $w$ for every $w \in V \setminus \{v\}$. An \emph{in-arborescence} rooted at $v$ is similar, but with unique directed paths from all $w \in V \setminus \{v\}$ to $v$.
\end{defi}

Given a vertex set $V = [n]$,  let $\cG_{\rm T}$ denote the set of out-arborescences rooted at node $1$. %
Note that each $G \in \cG_{\rm T}$ will have $m = n-1$ edges. Any tree-structured distribution $P$ on $[k]^n$ is $G$-structured for some $G \in \cG_{\rm T}$. Let $\cC^{\rm\textsc{Tree}}$ denote the set of all tree-structured distributions on $[k]^n$.
Now we have the following lemma.

\begin{lem}\label{lem:tree_structured_discretization}
\dels{
If we consider the finite set $\cN^{\rm\textsc{Tree}}_\varepsilon = \bigcup_{G \in \cG_{\rm T}} \cN^{G}_\varepsilon$ of distributions on $[k]^n$ (see \Cref{def:bn_net}), then for any distribution $P^\ast \in \Delta([k]^n)$,
$$\min_{Q \in \cN^{\rm\textsc{Tree}}_\varepsilon} \kl(P^\ast \| Q) \leq \min_{Q \in \cC^{\rm\textsc{Tree}}} \kl(P^\ast \| Q) + \varepsilon,$$
The size is bounded by $\log |\cN^{\rm\textsc{Tree}}_\varepsilon| \leq O(n k^2 \log (nk/\varepsilon))$.
}
\adds{
Consider the finite set $\cN^{\rm\textsc{Tree}}_{\varepsilon,\delta} = \cN^{\cG_{\rm T}}_{\eps,\delta}$ (see \Cref{def:bayes_dags_finitization}) --- where $\cG_{\rm T}$ is the set of all $1$-rooted out-arborescences on vertex set $[n]$ --- of tree-structured Bayes net distributions on $[k]^n$ (see \Cref{defi:bayesnet}). Then, for any distribution $P^\ast \in \Delta([k]^n)$,
$$\min_{Q \in \cN^{\rm\textsc{Tree}}_{\varepsilon,\delta}} \kl(P^\ast \| Q) \leq \min_{Q \in \cC^{\rm\textsc{Tree}}} \kl(P^\ast \| Q) + \varepsilon,$$
with probability $\geq 1 - \delta$ over the samples used to construct $\cN^{\rm\textsc{Tree}}_{\eps,\delta}$. The size of $\cN^{\rm\textsc{Tree}}_{\eps,\delta}$ is bounded by $\log |\cN^{\rm\textsc{Tree}}_{\eps,\delta}| \leq \log |\cG_{\rm T}| \leq n\log(n)$.
}
\end{lem}

\begin{proof}

\adds{
By the definition of tree-structured distributions, every tree-structured distribution will be $G$-structured for some $G \in \cG_{\rm T}$. Thus, by the construction of $\cN^{\cG_{\rm T}}_{\eps,\delta}$ (see \Cref{def:bayes_dags_finitization} and \Cref{lem:bn_net_kl}), we will have that $\min_{P \in \cN^{\cG_{\rm T}}_{\eps,\delta}} \kl(P^\ast \| P) \leq \min_{P \in \cC^{\rm\textsc{Tree}}} \kl(P^\ast \| P) + \eps$ w.p $\geq 1 - \delta$.
We also have $|\cN^{\cG_{\rm \textsc{Tree}}}_{\eps,\delta}| \leq |\cG_{\rm T}|$ by construction, and the final bound follows from Cayley's formula.
}
\end{proof}

\subsection{Sampling from the EWA/RWM Distribution}
In this section, we describe our approach for sampling tree-structured distribution using \EWA\ and \RWM\ algorithms.

Following \Cref{defi:bayesnet}, each tree-structured distribution $P$ in \dels{$\cN^{\rm\textsc{Tree}}_\varepsilon$}\adds{$\cN^{\rm\textsc{Tree}}_{\varepsilon,\delta}$} can be factored as $(G, V_P)$ where $G$ is the underlying rooted arborescence and $V_P = (p_e : e \in E(G), p_{1})$ is the vector of functions where $p_e(x_i,x_j)$ for $e = (i,j)$ corresponds to $\Pr_{X \sim P}(X_j = x_j \mid X_i = x_i)$ and $p_1(x_1)$ corresponds to $\Pr_{X \sim P}(X_1 = x_1)$, where $X_1$ denotes the root node of $G$. \dels{Using the notation of \Cref{def:cond_net}, $p_1 \in \cN^{X_1}_{\eps/n}$ and for $e=(i,j)$, $p_e \in \cN^{X_j \mid X_i}_{\eps/n}$. For simplicity, we will call these sets $\cN^1$ and $\cN^e$ respectively in this section. }\adds{By construction of $\cN^{\rm\textsc{Tree}}_{\varepsilon,\delta}$ (\Cref{lem:tree_structured_discretization}), each $p_i$ will be the add-one distribution (see \Cref{def:add_one_distribution} and \Cref{def:bayes_dags_finitization}) of $X_i | X_{\pa(i)}$, computed using $s_{\rm AO}(\eps,\delta)$ samples from $P^\ast$ (see \Cref{def:bayes_dags_finitization}).}

Let us start with our result for \EWA\ algorithm.

\begin{lem}\label{lem:ewa_poly_tree}
\adds{Let $\wh{P}_t$ be the output of the {\rm \EWA}\ algorithm run with the expert set $\cN^{\rm\textsc{Tree}}_{\eps,\delta}$, parameter $\eta$, and horizon $T \geq t$ for time step $t$. For any sequence of observed samples $x^{(1)}, \dots, x^{(T)} \in [k]^n$, one can generate a sample from $\wh{P}_t$ in polynomial time (w.r.t $n, k, 1/\eps$ and $\log(1/\delta)$).}
\end{lem}

\begin{proof}

\adds{
In the \EWA\ algorithm, $\wh{P}_t$ is a mixture of the distributions in $\cN^{\rm\textsc{Tree}}_{\eps,\delta}$. For a distribution $P \in \cN^{\rm\textsc{Tree}}_{\eps,\delta}$ that factors into $(G, (p_e: e \in E(G), p_1))$ as above (where $p_1$ and all $p_e : e \in E(G)$ are add-one distributions (see \Cref{def:add_one_distribution} and \Cref{def:bayes_dags_finitization})), let the weight of $P$ in $\wh{P}_t$ be $\omega(P)$. We have the following (note that, in the analysis below, we use $q_1$, $q_e : e \in E(H)$ for the add-one distributions w.r.t DAG $H$ to distinguish them from $p_1$, $p_e : e \in E(G)$ (w.r.t DAG $G$)):
\begin{align}
&\omega(P) \notag\\
 &=\frac{\prod_{s<t} \exp_\eta \log\left( p_1(x^{(s)}_1) \cdot \prod_{e \in E(G)} p_e(x^{(s)}_e)\right)}{
 	\sum_{H \in \cG_{\rm T}} \prod_{s<t}\exp_{\eta}\log \left(q_1(x^{(s)}_1) \cdot \prod_{e \in E(H)} q_e(x^{(s)}_e)\right)} \notag\\
&= \frac{\prod_{s<t} \exp_\eta \log p_1(x^{(s)}_1) \cdot \prod_{e \in E(G)} \prod_{s<t}\exp_\eta \log p_e(x^{(s)}_e)}
{\prod_{s<t}\exp_{\eta}\log {q_{1}(x^{(s)}_1)} \cdot \sum_{H \in \cG_{\rm T}} \prod_{e \in E(H)} \prod_{s<t} \exp_{\eta}\log q_{e}(x^{(s)}_e)} \notag \\
&= \frac{\cancel{\prod_{s<t} \exp_\eta \log p_1(x^{(s)}_1)}}{\cancel{\prod_{s<t}\exp_{\eta}\log q_{1}(x^{(s)}_1)}} \cdot
\prod_{e \in E(G)}\frac{\prod_{s<t}\exp_\eta \log q_e(x^{(s)}_e)}{\sum_{H \in \cG_{\rm T}} \prod_{e \in E(H)} \prod_{s<t} \exp_{\eta}\log q_{e}(x^{(s)}_e)},\notag \\
\label{eqn:treetotalweight}
\end{align}
where we note that $p_1 = q_1$ always (both are add-one distributions from the same set of samples, and node $1$ is the root node in both $G$ and $H$ by definition of $\cG_{\rm T}$.

We can interpret (\ref{eqn:treetotalweight}) as follows: the whole product gives the probability of sampling structure $G$ (and hence a Bayes net distribution $P$ on $[k]^n$ when $G$ is combined with the appropriate add-one distribution) as a product distribution over $G$'s edges. Hence, we can sample $G$ edge-wise, recursively following the above product distribution, as long as we can sample from the edge-probability distributions efficiently. The difficulty lies in the fact that the normalization factors involve exponentially many terms (going over each $H \in \cG_{\rm T}$).

$G$ is sampled from a weighted mixture of spanning arborescences of the complete directed graph on $n$ nodes, where the weight of a spanning arborescence $H$ is proportional to $\prod\limits_{e \in E(H)} \mathbf{w}(e)$ and $\mathbf{w}(e) = \prod_{s<t}\exp_{\eta}\log q_{e}(x^{(s)}_e)$ for any $e = (i,j) \in [n]^2$ ($q_e$ is an add-one distribution). Sampling weighted arborescences is known to be in polynomial time (\cite{borchardt1860ueber, chaiken1982combinatorial,  tutte2001graph, de2020elementary}), using Tutte's theorem; for completeness, we describe the algorithm in \Cref{sec:learntreedist:sampling}.
Once $G$ is sampled, we can construct $P$ from a sufficiently large set $S_\cN$ of samples from $P^\ast$ ($|S_\cN| = s_{\rm AO}(\eps,\delta)$, see \cref{def:bayes_dags_finitization}). First, we fix the distribution $p_1$ (of $X_1$) and then we follow the edge structure of $G$ to fix the distribution of each $X_i | X_{\pa(i)}$ as the appropriate add-one distribution (see \Cref{def:add_one_distribution}). 

Once $P$ is constructed, a sample from $P$ can be generated in polynomial time (in $n, k$) as $P$ is tree-structured. This completes the proof of the lemma.
}
\end{proof}

Now let us state our result for \RWM\ algorithm.

\begin{lem}\label{lem:rwm_poly_tree}
\adds{
If the {\rm \RWM}\ algorithm is run with the expert set $\cN^{\rm\textsc{Tree}}_{\eps,\delta}$, parameter $\eta$, and horizon $T$, the prediction $\wh{P}_t$ at time step $t \leq T$ can be computed in polynomial time, for any sequence of observed samples $x^{(1)}, \dots, x^{(T)} \in [k]^n$.
}
\end{lem}
\begin{proof}
This follows directly from the argument used to prove \Cref{lem:ewa_poly_tree}, since in \RWM, $\wh{P}_t$ is sampled during the execution of the learning algorithm rather than during the sampling stage as in \EWA. 
\end{proof}

\subsection{Learning Guarantees}

In this section, we will state the learning guarantees for tree-structured distributions for both the improper and proper settings. Let us start by describing our result for the improper setting.

\begin{theo}
\adds{
Let $P^*$ be an unknown distribution over $[k]^n$, and $\eps, \delta \in (0,1)$ be parameters. Given sample access to $P^*$, there exist algorithms outputting efficiently-sampleable distributions (which are mixtures of trees) giving the following guarantees:
\begin{enumerate}
\item[(i)]
With probability at least $2/3$, the the output distribution $Q$ is an $(\eps, 4)$-approximation of $P^*$; the sample complexity is $\tilde{O}\left(\frac{n k^2 \log^2(nk/\varepsilon)}{\varepsilon}\right)$.
\item[(ii)]
A realizable $(\eps,\delta)$-PAC learner for $\cC^{\textsc{Tree}}$ with sample complexity $\tilde{O}\left(\frac{n k^2 \log^2(nk/\delta\varepsilon)}{\delta\varepsilon}\right)$.
\item[(iii)]
An agnostic $(\eps,\delta)$-PAC learner for $\cC^{\textsc{Tree}}$ with sample complexity $$\tilde{O}\left(n^4 \eps^{-4} \log^4(nk/\eps\delta) + nk^2\eps^{-1}\log^2(nk/\eps\delta)\right).$$
\end{enumerate}
}
\end{theo}
\begin{proof}
The algorithm in all cases is the \EWA-based improper learning algorithm (Algorithm \ref{alg:meta_improper}) instantiated with appropriate $\cN = \cN^{\rm\textsc{Tree}}_{\eps,\delta}$ as in the proof of \Cref{thm:bn_learn_imp}. The correctness and sample complexities in each case follow from \Cref{thm:bn_learn_imp}. The fact that the output distribution is efficiently sampleable is shown in \Cref{lem:ewa_poly_tree}. 
\end{proof}

In the realizable case (where $P^\ast$ is a tree-structured distribution), this sample complexity bound matches the following lower-bound in \citep[Appendix A.2]{canonne2017testing} up to log-factors for constant error probability. 

\begin{lem}[Restatement of lower bound result in Appendix A.2 of \cite{canonne2017testing}]
Let $P^\ast$ be an unknown tree-structured distribution defined over $\{0,1\}^n$, and $\eps \in (0,1)$ be a parameter. Any algorithm that outputs a distribution $Q$ which is an $\eps$-approximation of $P^\ast$ with constant probability requires $\Omega(\frac{n}{\eps})$ samples from $P^\ast$\footnote{The lower bound stated in \cite{canonne2017testing} holds for general Bayes net distributions of indegree $\leq d$. For tree-structured distributions, we have $d=1$. Moreover, their lower bound is with respect to TV-distance, which translates to KL-divergence due to Pinkser's inequality.}.      
\end{lem}

Now we will state the guarantees for proper learning of the tree-structured distributions.

\begin{theo}
\adds{
Let $P^*$ be an unknown distribution over $[k]^n$, and $\eps, \delta \in (0, 1)$ be parameters. Given sample access to $P^*$, there exists a proper agnostic $(\eps,\delta)$-PAC-learner for $\cC^{\textsc{Tree}}$ that takes $\widetilde{O}\left(\frac{n^3}{\eps^2 \delta^2} \log^3\left(\frac{nk}{\eps\delta}\right) + \frac{nk^2}{\eps}\log^2\left(\frac{nk}{\eps\delta}\right)\right)$ samples from $P^\ast$ and runs in time $O(\mathsf{poly}(n,k,1/\eps,1/\delta))$.
}
\end{theo}

\begin{proof}
The algorithm is the \RWM-based proper learning algorithm (Algorithm \ref{alg:meta_proper}) applied with $\cN = \cN^{\textsc{Tree}}_{\varepsilon/2,\delta/2}$, $\eta = \sqrt{(8 \log |\cN|)/T}$ and the log loss function. The algorithm, by construction, will output a tree-structured distribution in $\cN$. The correctness (agnostic PAC guarantee) and sample complexity follow from exactly the same argument as in \Cref{thm:bn_learn_proper}, taking $d = 1$ since tree-structured distributions are Bayes nets with in-degree $1$. The fact that the \RWM-computation can be done in $\mathsf{poly}(n,k,\frac{1}{\eps},\frac{1}{\delta})$ time, in spite of having exponentially many distributions in $\cN$, follows from the fact that the sampling from the weight distribution, for each time step $t$, can be done efficiently when $\cN = \cN^{\textsc{Tree}}_{\eps/2,\delta/2}$ (see \Cref{lem:rwm_poly_tree}).
\end{proof}

\subsection{Proof of \texorpdfstring{\Cref{lem:ewa_poly_tree}}{Lemma \ref{lem:ewa_poly_tree}}}\label{sec:learntreedist:sampling}

Given a vertex set $V = [n]$, $\cG_{\rm T}$ denotes the set of all spanning out-arborescences rooted at node $1$. $\cC^{\rm\textsc{Tree}}$ denotes the set of all tree-structured distributions on $[k]^n$. $\cN^{\rm\textsc{Tree}}_{\varepsilon,\delta}$ denotes the finite set of distributions (see \Cref{lem:tree_structured_discretization} and \Cref{def:bayes_dags_finitization}) that \emph{finitizes} $\cC^{\rm\textsc{Tree}}$ with error $\varepsilon$ in $\kl$ and failure probability $\delta$. Let $S_{\cN}$ denote the set of samples used to construct $\cN^{\rm\textsc{Tree}}_{\eps,\delta}$. By construction, each tree-structured distribution $P$ in $\cN^{\rm\textsc{Tree}}_{\varepsilon,\delta}$ can be factored as $(G, V_P)$ where $G$ is the underlying rooted arborescence and $V_P = (p_e : e \in E(G), p_{1})$ is the vector of functions where $p_e(x_i,x_j)$ for $e = (i,j)$ corresponds to $\Pr_{X \sim \aonodedist{S_{\cN}}{j}{\{i\}}}(X_j = x_j \mid X_i = x_i)$ and $p_1(x_1)$ corresponds to $\Pr_{X \sim \aonodedist{S_\cN}{1}{\emptyset}}(X_1 = x_1)$ (here, as in \Cref{def:add_one_distribution}, $\aonodedist{S_\cN}{v}{\{\pa(v)\}}$ denotes the add-one distribution computed from $S_\cN$ for node $v$ when the parent set is $\{\pa(v)\}$).%

In the proof of \Cref{lem:ewa_poly_tree}, we derive the following expression for $\omega_t(P)$, the weight assigned to a tree-structured distribution $P \cong (G, (p_e : e \in E(G), p_1))$ in the discretized set $\cN = \cN^{\rm\textsc{Tree}}_{\eps,\delta}$ by the \EWA/\RWM\ algorithms at step $t$.

\begin{align}\label{eqn:app_tree_ewa_dist_factor}
	\omega_t(P)
	&=
	\prod_{e \in E(G)}\frac{\prod_{s<t}\exp_\eta \log p_e(x^{(s)}_e)}{\sum_{H \in \cG_{\rm T}} \prod_{e \in E(H)} \prod_{s<t} \exp_{\eta}\log q_{e}(x^{(s)}_e)},
\end{align}
where $p_e : e \in E(G)$ denotes the add-one-distributions w.r.t structure $G$, and $q_e : e \in E(H)$ denotes the add-one distributions w.r.t structure $H$ (all distributions computed from the samples $S_\cN$ with $|S_\cN| = s_{\rm AO}(\eps,\delta) \leq \tilde{O}\left(\frac{n k^{2}\log^2(nk/\eps\delta)}{\eps}\right)$, see \Cref{def:bayes_dags_finitization}).
 
This gives the following high-level algorithm for sampling from the \EWA/\RWM\ distribution. 

\begin{itemize}
	\item[(i)]
	Sample a spanning arboresence $G \in \cG_{\rm T}$ (rooted at node $1$), where $\Pr(G) \propto \prod\limits_{e \in E(G)} \mathbf{w}(e)$ and $\mathbf{w}(e) \triangleq \left(\prod_{s<t} \aonodedist{S_\cN}{j}{\{i\}}(x^{(s)}_e)\right)^\eta$ for any $e = (i,j) \in [n]^2$. This can be done in polynomial time (even though there are exponentially many such arborescences) by applying Tutte's matrix-tree theorem to weighted digraphs (\cite{borchardt1860ueber, chaiken1982combinatorial,  tutte2001graph, de2020elementary}). We describe this approach in the algorithm \textsc{SamplingArborescence}, which we describe and analyze later. We invoke
	\[
	G \leftarrow \textsc{SamplingArborescence}(G_0 \leftarrow K_n, w \leftarrow \mathbf{w}, r \leftarrow 1),
	\]
	where $K_n$ is the complete graph on the vertex set $[n]$ viewed as a digraph, which returns $G \in \cG_{\rm T}$ (rooted at $1$) with the required sampling probability.
	\item[(ii)] For each $e = (i, j) \in E(G)$ (fixed in the preceding step), set the node distribution of node $j$, $\widehat{p}_j(x_j | x_i)$, to be the add-one distribution $\aonodedist{S_\cN}{j}{\{i\}}$. Also let $\widehat{p}_1$ be the add-one distribution $\aonodedist{S_\cN}{1}{\emptyset}$. Finally, let $P$ be the Bayes net $(G, (\widehat{p}_1, \ldots, \widehat{p}_n))$.
\end{itemize}

Once $P$ is sampled, a sample from $P$ can be generated in polynomial time as $P$ is tree-structured.

\paragraph*{Correctness} The correctness of the above algorithm follows from considering each factor of \Cref{eqn:app_tree_ewa_dist_factor} and noting that the steps (i)-(ii) described above are independent.
	
\paragraph*{Running Time} The \textsc{SamplingArborescence} call in step (i), in this case (invoked with $K_n$), requires $O(n^5)$ time, and the weights required (for all the edges in $K_n$) can be computed in $O(n^2 k \times |s_{\rm AO}(\eps,\delta)|)$ time. The add-one distributions used in step (ii) can again be computed in $O(nk\cdot s_{\rm AO}(\eps, \delta))$ time (one pass over the samples, with fixed $G$ giving the parent, for each node).  which gives an overall time complexity which is $\tilde{O}(n^8 k^3/\eps\delta)$.

\subsection*{The SamplingArborescence algorithm}

\paragraph*{Preliminaries} Let $G = ([n], E, w)$ be a connected weighted directed graph on $n$ vertices $\{1,\ldots,n\}$ with $m$ edges $e_1, \ldots,e_m$. $w: E \rightarrow \R_{>0}$ is a positive weight function. Let $A_w(G)$ denote the $n \times n$ weighted vertex adjacency matrix of $G$. For every $i, j \in [n]$, $A_w(G)$ is defined as follows:
\[
[A_w(G)]_{i,j} = 
\begin{cases}
	w(e_k), & \mbox{if $e_k$ is the directed edge from $i$ to $j$} \\
	0, & \mbox{if there is no directed edge from $i$ to $j$}
\end{cases}
\]

Let $N_{\mathrm{in},w}(G)$ be the weighted \emph{inward} incidence matrix of order $m \times n$ defined as follows: For every $i \in [n]$ and $k \in [m]$, we have:
\[
[N_{\mathrm{in},w}(G)]_{k,i} =
\begin{cases}
	\sqrt{w(e_k)}, & \text{if directed edge $e_k$ points to vertex $i$}\\
	0, & \text{otherwise}
\end{cases}
\]

Similarly, we can define the weighted \emph{outward} incidence matrix $M_{\mathrm{out},w}(G)$ of order $n \times m$ as follows, for $i \in [n]$ and $k \in [m]$: 
\[
[M_{\mathrm{out},w}(G)]_{i,k} =
\begin{cases}
	\sqrt{w(e_k)}, & \text{if directed edge $e_k$ points from vertex $i$}\\
	0, & \text{otherwise}
\end{cases}
\]

The \emph{indegree} matrix $D_{\mathrm{in},w}(G)$ is a $n \times n$ diagonal matrix such that, for all $i \in [n]$, $[D_{\mathrm{in},w}(G)]_{i,i}$ is equal to
the sum of the weights of all incoming edges to vertex $i$. Similarly, the \emph{outdegree} matrix $D_{\mathrm{out},w}(G)$ is a $n \times n$ diagonal matrix such that for all $i \in [n]$, $[D_{\mathrm{out},w}(G)]_{i,i}$ is equal to the sum of the weights of all outgoing edges from vertex $i$.

Now we can define the \emph{Laplacian matrices} $L_{1,w}(G)$ and $L_{2,w}(G)$ associated with digraph $G$ as follows:

\begin{equation}
	L_{1,w}(G) \triangleq D_{\mathrm{in},w}(G) - A_{w}(G), \  \mbox{and} \ L_{2,w}(G) = D_{\mathrm{out},w}(G) - A^\top_{w}(G).    
\end{equation}

For all these associated matrices, we omit $G$ and denote the matrix $A_w(G)$ as $A_w$ etc. when $G$ is clear from the context. Now we have the following relationships between the Laplacian and incidence matrices:

\begin{cl}[Equation ($10$) of \cite{de2020elementary}]
	\begin{itemize}
		\item[(i)] $L_{1,w} = (N^\top_{\mathrm{in},w} - M_{\mathrm{out},w})N_{\mathrm{in},w}$.
		
		\item[(ii)] $L_{2,w} = (M_{\mathrm{out},w} - N^\top_{\mathrm{in},w})M^\top_{\mathrm{out},w}$.
	\end{itemize}	
\end{cl}

For a vertex $r \in [n]$, let $L^r_{1,w}$ be the $(n-1) \times (n-1)$ matrix obtained by removing the $r$-th row and $r$-th column of $L_{1,w}$. Similarly, we also define $N^r_{\mathrm{in},w}$ and $M^r_{\mathrm{out},w}$. Then we have the following:
\begin{cl}[Equation ($11$) of \cite{de2020elementary}]
	\begin{itemize}
		\item[(i)] $L^{r}_{1,w} = ((N^r_{\mathrm{in},w})^\top - M^{r}_{\mathrm{out},w})N^{r}_{\mathrm{in},w}$.
		\item[(ii)] $L^{r}_{2,w} = ((M^{r}_{\mathrm{out},w} - (N_{\mathrm{in},w}^{r})^\top))(M_{\mathrm{out},w}^{r})^\top$.
	\end{itemize}
\end{cl}

\begin{defi}[Weight of a graph]\label{def:graph_prod_wt}
	Let $G = (V, E, w)$ be a weighted (directed or undirected) graph, where $w: E \rightarrow \R_{>0}$. Let $H$ be any subgraph of $G$. The weight of graph $H$ is defined as the \emph{product} of the weights of its edges, i.e.
	$$w(H) = \prod_{e \in E(H)} w(e).$$    
\end{defi}

\begin{defi}[Contraction and deletion]\label{def:simple_contraction_deletion}
If $G = (V, E, w)$ is a weighted undirected graph (not necessarily simple) and $e = \{i,j\} \in E$, the graph $G/e$ obtained by \emph{contracting} edge $e$ is a weighted graph on $V \setminus \{i, j\} \cup \{\langle i,j \rangle\}$ where
\begin{itemize}
	\item Vertices $i$ and $j$ are removed, and a new vertex $\langle i,j \rangle$ is added.
	\item \emph{All} edges $\{i,j\}$ (in case of parallel edges) are removed in $G/e$.
	\item Edges $\{u,v\}$ where $u, v \in V \setminus \{i,j\}$ are preserved in $G/e$ with the same weight.
	\item Every edge $\{u,v\} \in E$ where $u \in \{i,j\}$ and $v \in V \setminus \{i,j\}$ becomes an edge $\{\langle i, j \rangle, v\}$ in $G/e$ with the same weight.
\end{itemize}

Similarly the graph $G \setminus e$ obtained by \emph{deleting} edge $e$ is just $(G \setminus e) = (V, E \setminus e, w\mid_{E \setminus e})$ where the edge $e$ is deleted (including all parallel edges) from $G$.
\end{defi}

\begin{defi}[Contraction edge mapping]\label{def:contraction_edge_mapping}
With the above definition of contraction (with parallel edges kept), if $G$ is a graph and $G^\prime = G \setminus e$ for $e = \{i,j\} \in E(G)$, we can map each edge $e \in E(G^\prime)$ \emph{injectively} to an edge $e \in E(G)$ that caused its inclusion in $G^\prime$. We denote this mapping by $f_{G^\prime:G}: E(G^\prime) \rightarrow E(G)$.
\end{defi}

The following theorem is a generalization of Tutte's matrix-tree theorem to weighted graphs.

\begin{lem}[Theorem 3 of \cite{de2020elementary}]\label{lem:weighted_directed_mt}
	Let $G = (V,E, w)$ be a weighted directed graph on $n$ vertices $[n]$ with positive weight function $w: E \rightarrow \R_{>0}$. Then, for any $r \in [n]$, the sum of the weights of the weighted outgoing (incoming) spanning arborescences rooted at vertex $r$ is equal to $\det(L^r_{1,w})$  ($\det(L^r_{
		2,w})$, respectively).
\end{lem}

Now the next corollary follows as any undirected graph can be viewed as a directed graph with edge $\{i,j\}$ corresponding to a pair of directed edges $(i,j)$ and $(j,i)$.
 
\begin{coro}\label{cor:mt_thm_application}
If $G = K_n$ (the complete graph on $n$ vertices) and $w: [n]^2 \rightarrow \R_{>0}$ is any positive weight function,
$$\det(L^1_{1,w}) = \sum_{G \in \cG_{\rm T}} \prod_{e \in E(G)} w(e),$$
where $\cG_{\rm T}$, as defined earlier, is the set of all spanning out-arborescences on $[n]$ rooted at vertex $1$.
\end{coro}

\begin{algorithm}[ht]
    \setcounter{AlgoLine}{0}
	\caption{The \textsc{SamplingArborescence} algorithm}\label{alg:samparboracences}
	\KwIn{$G_0$ a directed graph, $w: E(G_0) \rightarrow \R_{\geq 0}$ weight function, $r \in V(G_0)$ root node.}
	\KwOut{Arborescence $G_T \in \cA_{G_0,r}$ with $\Pr(\text{output } G_T) \propto \prod_{e \in E(G_T)} w(e)$.}
	Let $G_T \gets \emptyset$. \
    
	Choose an arbitrary ordering of the $m$ edges of $G_0$, and let $\mathsf{EdgeList} \gets [e_1,\ldots,e_m]$. \
    
	Let $\mathsf{ContractionMapping}[e_i] \gets e_i$ for each $i \in [m]$. \
    
	Let $j \leftarrow 0$. \
    
	\While{$r$ has an outgoing edge in $G_j$}{
		Let $e = (r, x)$ be the first remaining outgoing edge. \
        
		Compute $p_e \gets \frac{\det(L^r_{1,w}(G_j \setminus e))}{\det(L^r_{1,w}(G_j))}$.\
        
		Sample $u \sim \mathsf{Unif}((0, 1])$. \
        
		\If{$u \leq p_e$ (with probability $p_e$)} {
			$G_{j+1} \gets G_j \setminus e$ (delete edge $e$). \
            
			$j \gets j + 1$. \
            
		}
		\Else {
			$G_{j+1} \gets G_{j}/e$ (contract edge $e$). \
            
			Add $\mathsf{ContractionMapping}[e]$ to $G_T$ (this will be an edge of $G_0$). \
    
			\tcc{Reconstruct \textsf{ContractionMapping} for mapping $E(G_{j+1})$ to $E(G_0)$.}        
			$\mathsf{CM}^\prime \leftarrow \textrm{empty dictionary}$. \
            
			\For{each edge $e \in E(G_{j+1})$}{
				$\mathsf{CM}^\prime[e] \gets \mathsf{ContractionMapping}[f_{G_{j+1}:G_j}(e)]$ (see \Cref{def:contraction_edge_mapping}). \
			}
            
			$\mathsf{ContractionMapping} \gets \mathsf{CM}^\prime$ \
            
			$r \gets \langle r, x \rangle \in V(G_{j+1})$ (the newly-contracted vertex). \ 
            
			$j \gets j+1$. \
		}
	}
	\KwRet{$G_T$}. \
\end{algorithm}

\paragraph*{The algorithm}

If $G$ is a graph and $v \in V(G)$ be a vertex of $G$, let $\cA_{G,v}$ denote the set of out-arborescences on $V(G)$ which are subgraphs of $G$ (viewed as a digraph), rooted at $v$, and span the set of nodes reachable from $v$. Note that this set will not be empty unless $v$ is an isolated vertex.

\begin{defi}[Product arborescence distribution]\label{def:prod_arbo_dist}
	Let $G = (V,E,w)$ be a weighted graph where $w: E \rightarrow \R_{\geq 0}$ is a weight function on the edges. For any $v \in V(G)$, a distribution $\mu$ on $\cA_{G,v}$ is said to be the \emph{product arborescence distribution} if
	\[
\Pr_{G \sim \mu}(G) \propto \prod_{e \in G} w(e) \mbox{ for all $G \in \cA_{G,v}$}.
	\]
\end{defi}

\begin{cl}\label{cl:prod_arbo_dist_edge_prob}
	Suppose $G = (V, E, w)$ is a weighted digraph with non-negative weights, $r \in V(G)$ and $\mu$ is a product arborescence distribution on $\cA_{G,r}$.
	Then
	\[
	\Pr_{T \sim \mu}(e \in T) = 1 - \frac{\det(L^r_{1,w}(G \setminus e))}{\det(L^r_{1,w}(G))}.
	\]
\end{cl}
\begin{proof}
	Since $\Pr_\mu(T) \propto \prod_{e \in T} w(e)$, we have
	\begin{align*}
		\Pr_{T \sim \mu}(e \in T) = \frac{\sum_{T \in \cA_{G,r} \,:\, T \ni e} w(T)}{\sum_{T \in \cA_{G,r}} w(T)} = 1 - \frac{\sum_{T \in \cA_{G,r} \,:\, T \not\ni e} w(T)}{\sum_{T \in \cA_{G,r}} w(T)} = 1 - \frac{\det(L^r_{1,w}(G \setminus e))}{\det(L^{r}_{1,w}(G))},
	\end{align*}
	where the first equality follows from \Cref{def:prod_arbo_dist} and \Cref{def:graph_prod_wt}, and the third equality follows from applying \Cref{lem:weighted_directed_mt} to $G \setminus e$ and $G$.
\end{proof}

Now we are ready to prove the correctness of the algorithm \textsc{SamplingArborescence}.

\begin{cl}\label{cl:sample_arborescence_correctness}
	There is an algorithm \textsc{SamplingArborescence} that, if invoked with input $(G, w, r)$ for a digraph $G$, $w: E(G) \rightarrow \R_{\geq 0}$, and $r \in V(G)$, returns a random arborescence $G_T$ sampled from the product arborescence distribution on $\cA_{G,r}$.
	The algorithm runs in time $O(|E||V|^3)$  in the real RAM model.
\end{cl}
\begin{proof}
	The pseudocode for the algorithm is given as Algorithm \ref{alg:samparboracences}. Observe that, if $G_T \in \cA_{G,r}$, then it is actually sampled from the product arborescence distribution (\Cref{def:prod_arbo_dist}) by \Cref{cl:prod_arbo_dist_edge_prob}, since each edge $e$ is added to $G_T$ with the correct probability $1 - p_e$.
	
	It remains to ensure that $G_T$ is indeed a spanning arborescence. Let $r_j$ denote the value of $r$ used in iteration $j$ of the while loop (lines 5-20). $r_j \in V(G_j)$ by construction. Let $\cS(r_j) \subseteq V(G_0)$ denote $\{r_j\}$ if $V(G_j) = V(G_0)$ (no contractions) and the set of nodes in $G_0$ corresponding to the contracted-vertex in $G_j$ otherwise.
	
	By construction, when $|\cS(r_j)| > 1$, the algorithm maintains (in $G_T$) an out-arborescence rooted at $r = r_0$ that spans $\cS(r_j)$. This is because when the algorithm selects an edge $e = (r_{j}, x) \in E(G_{j})$ and takes $G_{j+1} = G_{j}/(r_{j},x)$, this edge will correspond to an edge $(v, x)$ in the original graph $G_0$ where $v \in \cS(r_{j})$ and $x \not\in \cS(r_{j})$. This will ensure that there are no (undirected) cycles formed when adding $(v, x)$ to $G_T$ and the rooted arborescence invariant will be maintained in iteration $j+1$ with $\cS(r_{j+1}) = \cS(r_{j}) \sqcup \{x\}$. Note that when an edge $e$ is selected for contraction in iteration $j$, the algorithm maintains a mapping $\mathsf{ContractionMapping}$ from the edges in $G_j$ to the edges in $G_0$ and adds the original edge (in $G_0$) to $G_T$. The other property of the algorithm is that it does not \emph{disconnect} the graph --- an edge $e$ which is a \emph{cut edge} will be present in all arborescences and hence will have $p_e = 0$, ensuring that it will never be selected for deletion (in line 9). So, the algorithm will terminate when $\cS(r_j)$ is the set of nodes reachable from $r = r_0$, and $G_T$ will have the required spanning arborescence.
	
	For the time complexity bound, note that the while loop (lines 5-20) will run at most $|E|$ times, while the expensive step inside the while loop will be the two determinant computations in line 7 (to compute $p_e$). These can be done in $O(|V|^3)$ time in the real RAM model, by Gaussian elimination etc. Contraction, deletion etc. require only $O(|V| + |E|) \leq O(|V|^2)$ time. This gives us the $O(|E||V|^3)$ bound.
\end{proof}

\section{Lower bound for learning tree-structured distributions}\label{sec:treelearnlb}

In this section, we show a lower bound of $\Omega(\frac{n}{\eps})$ holds for our tree structure learning results. This follows from \cite{bhattacharyya2023near}. We are including it here for completeness.

\begin{theo}\label{theo:kltreelb}
Given sample access to an unknown tree-structured discrete distribution defined over $[k]^n$, and $\eps >0$ be a parameter. With probability at least $9/10$, $\Omega(\frac{n}{\eps})$ samples are necessary for any realizable PAC learner for learning $\cC^{\textsc{Tree}}$.   
\end{theo}

We will prove the above result for $k=2$, that is, for distributions defined over $\{0,1\}^n$. The lower bound immediately extends to distributions over $[k]^n$.

We will first prove the above result for $n=3$ nodes. Then we will extend the result for $n=3\ell$, for some positive integer $\ell$. 

\begin{defi}[Hellinger distance]
Let $P$ and $Q$ be two probability distributions defined over the same sample space $\cD$. The Hellinger distance between $P$ and $Q$ is defined as follows:
$$\hel(P,Q)=\sqrt{\frac{1}{2} \sum_{x \in \cD}\left(\sqrt{P(x)- \sqrt{Q(x)}}\right)^2}$$
\end{defi}

\begin{fact}\label{fact:helklrelation}
Let $P$ and $Q$ be two probability distributions defined over the same sample space $D$. Then the following holds:
$$\hel(P,Q) \leq \sqrt{\frac{1}{2}\kl(P||Q)}$$  
\end{fact}

We will also need the following folklore result.

\begin{fact}\label{fact:hellingerlb}
Let $P$ be an unknown discrete distribution defined over $\{0,1\}^3$, and $\eps>0$ be a parameter. In order to output a distribution $\widehat{P}$ defined over $\{0,1\}^3$ such that $\hel(P, \widehat{P}) \leq \sqrt{\eps}$, $\Omega(\frac{1}{\eps})$ samples from $P$ are necessary.  
\end{fact}

\begin{lem}\label{lem:kllb3nodes}
There exist three tree-structured distributions $P_1, P_2,P_3$ defined over $\{0,1\}^3$, such that any algorithm that that can learn any distribution $P_i$ in up to $\eps$-KL-distance, that is, can output a distribution $\widehat{P}_i$ such that $\kl(P_i||\widehat{P}_i) \leq \eps$, requires $\Omega(\frac{1}{\eps})$ samples from $P_i$, for any $i \in [3]$.  
\end{lem}

$P_1$ has three nodes $X_1,Y_1, Z_1$, $P_2$ has three nodes $X_2,Y_2, Z_2$ and $P_3$ has three nodes $X_3,Y_3, Z_3$. In $P_1$, $(Y_1,Z_1)$ takes values uniformly at random between $(0,0)$ and $(1,1)$ and $X_1$ copies $Y_1$ and $Z_1$ with probability $(1- \eps)$, and with the remaining probability, takes a value from $\mathsf{Ber}(\frac{1}{2})$. $P_2$ and $P_3$ are defined in a similar fashion where we set $X_2=Z_2$ and $X_3=Y_3$, respectively.

We will use the following results from \citep{bhattacharyya2023near}.

\begin{cl}[See Fact 7.4 (i) and Lemma 7.5 (i)  of \cite{bhattacharyya2023near}]\label{cl:kllb3nodesapproxhel}
\begin{itemize}
    \item[(i)] $\hel(P_i,P_j) \geq 10 \sqrt{\eps}$ for any $i \neq j \in [3]$.

    \item[(ii)] Every tree $\cT$ on $3$ vertices is not $\Theta(\eps \log \frac{1}{\eps})$-approximate for any of $P_1, P_2$ and $P_3$.
\end{itemize}
\end{cl}

\begin{proof}[Proof of \Cref{lem:kllb3nodes}]
We will prove this by contradiction. Suppose there exists an algorithm $\cA$ that can learn any distribution up to $\eps$-KL-distance using $o(\frac{1}{\eps})$ samples, by outputting the pmf of $\widehat{P}$. In that case, we will argue that using $\cA$, we can distinguish between $P_1, P_2$ and $P_3$.

Let us denote the unknown distribution be $P$, and suppose given sample access to $P$, $\cA$ takes $o(1/\eps)$ samples from $P$ and outputs a distribution $\widehat{P}$ such that $\hel(P,\widehat{P}) \leq \sqrt{\eps}$. This implies that given sample access to either of $P_i$ with $i \in [3]$, $\cA$ can output such a $\widehat{P}$ such that $\hel(P_i, \widehat{P}) \leq \sqrt{\eps}$, which contradicts \Cref{fact:hellingerlb}. Using \Cref{fact:helklrelation}, we are done with the proof of the lemma.
\end{proof}

Now we are ready to prove \Cref{theo:kltreelb}.

\begin{proof}[Proof of \Cref{theo:kltreelb}]

Let us assume that $n=3 \ell$ for some positive integer $\ell$. We will divide the $n$ variables into $\ell$ blocks, each of size $3$ nodes.
We will prove this by contradiction.
We will define a distribution $P$ over $\{0,1\}^n$, where the $i$-th block is chosen uniformly to be either $P_1$ or $P_2$.

Suppose there exists an algorithm that can take $o(\frac{n}{\eps})$ samples and outputs an $\eps$-approximate tree $T$ of $P$, with probability at least 9/10. Since we have chosen each block independently, $T$ is a disjoint union of $T_1, \ldots, T_{\ell}$. Using \Cref{cl:kllb3nodesapproxhel}, and setting $\eps$ as $\eps/\ell$, we can say that each $T_i$ is not $\Theta(\eps/\ell)$-approximate with probability at least $2/3$. Thus, using Chernoff bound, we can say that at least $\frac{\ell}{100}$ trees are not $\frac{C\eps}{\ell}$ approximate. Thus, $\kl(P||\widehat{P}) \geq \eps$ for a suitable constant $C$ for any tree-structured distribution $\widehat{P}$. This completes the proof of the theorem.
\end{proof}

\section{Learning Bayesian networks with bounded vertex cover}\label{sec:learnbayesnetboundedvc}
In this section, we show that bayesian networks can be learned efficiently if the size of the vertex cover of the associated moralized graph is bounded. Before proceeding to present our result, we need some notations.

Let $G= (V,E)$ be DAG with the vertex set $V$ and the edge set $E$. Here each vertex $v \in V$ corresponds to a variable and the edges in $E$ encode the conditional independence relations between the nodes of $V$. Let $\cD$ denote the set of all possible DAG on the vertex set $V$. Moreover, for every DAG $G$, let us associate it with a non-negative weight function $f: G \rightarrow \R^+ \cup \{0\}$ which encodes how well $G$ fits a given dataset. Now we are ready to define the notion of modular weight function, which will be crucially used in our proofs.

\begin{defi}[Modular weight function]
Let $G(V,E)$ be a DAG with vertex set $V$ and edge set $E$. Moreover, let $f: G \rightarrow \R^+ \cup \{0\}$ denotes the weight function associated with $G$. The weight function $f$ is said to be a \emph{modular} weight function if it has the following form:
\begin{equation}
    f(G) = \prod_{v \in V} f_v(G_v)
\end{equation}    
\end{defi}

Now we define the notion of the vertex cover number of a DAG below.

\begin{defi}[Vertex cover number of a DAG]
Given a DAG $G$, let $G_M$ denote the undirected \emph{moralized graph} corresponding to $G$. The \emph{vertex cover number} $\tau(G)$ of $G$ is the size of the smallest vertex cover of $G_M$. Moreover, for some integer $\ell$, the set of DAGs $G$ such that $\tau(G) \leq \ell$ is denoted as $\ldag$.  
\end{defi}

Now we are ready to state the main result that we will be proving in this section.

\begin{theo}\label{theo:boundedvcbayesnet}
Let $P^*$ be an unknown discrete distribution on $[k]^n$ defined over a DAG $G$ such that the size of the vertex cover of the moralized graph $G_M$ corresponding to $G$ is bounded by an integer $\ell$. Moreover, let $\ldag$ be the family of distributions over $[k]^n$ that can be defined as Bayes nets over DAGs whose vertex cover size of the associated moralized graph is bounded by $\ell$. Given sample access from $P^*$, and a parameter $\eps > 0$,
there exist the following:
\begin{itemize}
    \item[(i)] An agnostic PAC-learner for $\ldag$ which uses $O\left(\frac{\ell^4 n^4}{\eps^{4}} \log^4\left(\tfrac{nk\ell}{\eps}\right) \log\left(\tfrac{1}{\delta}\right) + \frac{\ell^3 n k^{\ell+1}}{\eps} \log^2\left(\frac{nk}{\eps\delta}\right)\right)$ samples and runs in time $O(\exp(\ell)\mathrm{poly}(n))$ that is improper and returns an efficiently samplable mixture of distributions from $\ldag$.

    \item[(ii)] An agnostic PAC-learner for $\ldag$ which uses $O\left(\frac{\ell^3 n^3}{\delta^{2}\eps^{2}}\log^2\left(\frac{nk}{\eps\delta}\right) + \frac{\ell^3 n k^{\ell+1}}{\eps} \log^2\left(\frac{nk}{\eps\delta}\right)\right)$  samples and runs in time $O(\exp(\ell)\mathrm{poly}(n))$  that is proper and returns a distribution from $\ldag$.
\end{itemize}
\end{theo}

In order to prove the above theorem, we will be using the following result from \cite{harviainen2023revisiting} which states that $\ldag$ can be sampled efficiently.

\begin{theo}[Theorem $13$ of \cite{harviainen2023revisiting} restated]\label{theo:dagsamp}
There exists a randomized algorithm that can sample from $\ldag$ with weight proportional to $f(G)$ efficiently in expected sampling time $4^{\ell} n^{O(1)}$.    
\end{theo}

Note that the above result gives a polynomial time sampling algorithm for $\ldag$ when $\ell$ is bounded by a constant.

Interestingly, our loss function is also a modular weight function.

\begin{cl}\label{cl:modularweightfunction}
Our loss function $\exp_{\beta} \log\left(\prod_{s=1}^t P(x^{(i)})\right) $ is a modular weight function.    
\end{cl}

This follows from the fact that the distribution $P^*$ defined over the DAG $G$ factorizes: see \Cref{eqn:bnfactor} in \Cref{defi:bayesnet}. Now we are ready to present the proof of our main result.

\begin{proof}[Proof of \Cref{theo:boundedvcbayesnet}]
First note that a vertex cover of the moralized graph $G_M$ having size $\leq \ell$ implies that the indegree of $G$ is $\leq \ell$; if the indegree of $G$ is $> d$, there exists a \emph{clique} in the moralized graph of size $> d+1$, which implies that any vertex cover of the moralized graph must have size $> d$.
\begin{itemize}
\item[(i)] We will be using the \EWA \ algorithm (Algorithm~\ref{alg:ewa}) for this purpose. In each round of the algorithm, we will be using the algorithm from \Cref{theo:dagsamp}. The sample complexity (aka. number of rounds of \EWA\ algorithm) follows from the guarantee of \EWA\ algorithm for learning Bayesian networks with bounded indegree (see \Cref{thm:bn_learn_imp}). Since in each round, we will be calling the algorithm corresponding to \Cref{theo:dagsamp}, the running time of our algorithm will be $O(\exp(\ell)\mathrm{poly}(n))$.

\item[(ii)] We will be calling \RWM \ algorithm (Algorithm~\ref{alg:rwm}) here. Similar to the above, in each round, we will be using the algorithm from \Cref{theo:dagsamp}. 
The sample complexity (aka. number of rounds of \RWM\ algorithm) follows from the guarantee of \RWM\ algorithm for learning Bayesian networks (see \Cref{thm:bn_learn_proper}). Since in each round, we will be calling the algorithm corresponding to \Cref{theo:dagsamp}, the running time of our algorithm will be $O(\exp(\ell)\mathrm{poly}(n))$.    
\end{itemize}
\end{proof}

\section{Efficient Maximum Likelihood Estimation}\label{app:efficient_ml}
It is well-known that maximizing the likelihood $\E_{x \sim P^\ast} P(x)$ is equivalent to minimizing the KL divergence $\kl(P^\ast \| P)$. This equivalence, in expectation, follows from the definition of KL divergence as $\kl(P^\ast \| P) \triangleq \E_{x \sim P^\ast} \log\left(\frac{P^\ast(x)}{P(x)}\right)$. From a finite-sample PAC distribution learning perspective, we can still say that if $\widehat{P} \in \mathcal{C}$ is a distribution that maximizes the empirical log-likelihood $\frac{1}{T} \sum_{t=1}^{T} \log P(x^{(t)})$ --- computed using a sufficiently large set of samples $\{x^{(1)},\ldots,x^{(T)}\} \sim (P^*)^{\otimes T}$, $T \geq \textsf{Sample-Complexity}(\mathcal{C}, \varepsilon, \delta)$ --- over all $P \in \mathcal{C}$, then with probability $\geq 1 - \delta$, we will have $\kl(P^\ast \| P) \leq \min_{Q \in \mathcal{C}} \kl(P^\ast \| Q) + \varepsilon$. For instance, see Theorem 17 (Appendix G) of \cite{feldman2008learning}. However, if $\mathcal{N}$ is a class of discretized Bayes nets $\{(G_1, P_1),\ldots,(G_N, P_N)\}$, even with each DAG structure $G_i$ endowed with add-one conditional probabilities $P_i$ as we do in our EWA/RWM based learning-via-sampling approach, \emph{efficiently} finding a $(G_{i^\ast}, P_{i^\ast})$ that maximizes the empirical log-likelihood over such a class is not a trivial task.

By the Bayes net factorization property, the empirical log-likelihood of each $G_i$ does decompose nicely; as $\sum_{t=1}^{T} \log P_i(x^{(t)}) = \sum_{v=1}^{n} \sum_{t=1}^{T}\log P_{i,v}(x^{(t)}_v|x^{(t)}_{\pa_{G_i}(v)})$. However, maximizing $\sum_t \log P_{i,v}(x_v^{(t)} | x_S^{(t)})$ over all $S \subseteq N(v)$ with $|S| \leq d$, for each node $v$ (independently), does not suffice since this may result in the formation of cycles in the final digraph. Note that this is not an issue with tree-skeletons (tree-structured or polytree-structured distributions).

Another issue with independent maximization is that the final structure $G$ needs to be consistent with the skeleton and with the choice of parents for each node (each edge of the skeleton can only be oriented in one direction), which would need to be kept track of in the maximization algorithm.

We claim that the dynamic programming approaches developed for weighted-counting/sampling of exponentially-large classes of polytree and chordal-structured DAGs can be adapted to efficiently compute the maximum likelihood as well. For polytree distributions, we can root the skeleton $G$ arbitrarily and recursively compute $T[v, P, t]$ for each ``subtree'' $T_v$ rooted at vertex $v$, where $T[v, P, t]$ gives the maximum empirical likelihood ($t$-sample) over all possible DAG structures (with add-one probabilities) of $T_v$ such that vertex $v$ has fixed parents $P$, for all $P \subseteq N(v)$ with $|P| \leq d$. In the base case (only single edge), we can brute-force. We can then use the recurrence $T[v, P, t] = \sum_{\substack{v_1,\ldots,v_k\\ \text{``children'' of $v$}}} \sum_{\substack{P_1,\ldots,P_k\\ \text{consistent with } P \text{ and } G}}T[v_i, P_i, t]$ to compute a maximum-likelihood acyclic orientation of $G$. This is similar to our proposed weighted-counting DP, replacing the weighted sum with maximum empirical likelihood. For chordal-structured Bayes nets, we can do a similar adaptation of the \textsf{CountChordalDist} algorithm (\Cref{alg:countchordalindegree}); still recursively computing the DP table bottom-up from the clique-tree decomposition, orienting the ``\emph{link}'' of each clique to prevent cycles while getting ``independence'' of orientations of child-subtrees, etc.

Finally, we remark on the sample complexity of this approach. From e.g., \cite[Theorem 17]{feldman2008learning}, which uses a Hoeffding bound on the difference between the empirical and the true log-likelihood, we obtain a sample complexity of $\tilde{O}\left(\frac{\log^2(1/\tau)\log |\mathcal{C}|}{\eps^2}\right)$ for proper realizable learning in KL divergence (with constant error probability) via maximum likelihood, and the analysis can be adapted to agnostic learning as well. This matches our sample complexity bounds via RWM-regret up to log factors. For improper learning in the realizable case, our EWA-regret approach gives a better, near-optimal, sample complexity, as has been discussed before.

\section{Discussion}

\paragraph*{Conclusion} In this work, we established a novel connection between distribution learning and graphical structure sampling algorithms via the framework of online learning. Leveraging this connection, we designed efficient algorithms for agnostically learning bounded indegree chordal-structured distributions, with polynomial sample complexity. These algorithms only require knowledge of the distribution's skeleton, without needing information on the edge directions. Since polytree-structured distributions are a subset of chordal-structured distributions, our result also gives new results on the well-studied problem of learning polytree-structured distributions.  Interestingly, our method also leads to a new algorithm for learning tree-structured distributions, which is significantly different from the extremely well studied Chow-Liu algorithm. Finally, we also give an improper learning algorithm that, with probability $2/3$, gives an $(\eps, 3)$-approximation with respect to tree-structured distributions, which has a quadratic sample complexity advantage over Chow-Liu. Our work opens up several interesting research avenues, which we discuss in detail in \Cref{sec:open_problem}.

\begin{ack}
The authors would like to thank the anonymous reviewers for their comments which improved the presentation of the paper.
AB and PGJ's research were supported by the National Research Foundation, Prime Minister’s Office, Singapore under its Campus for Research Excellence and Technological Enterprise (CREATE) programme. SS's research is supported by the NRF Investigatorship award (NRF-NRFI10-2024-0006)
and CQT Young Researcher Career Development Grant (25-YRCDG-SS). AB and SS were also supported by National Research Foundation Singapore under its NRF Fellowship Programme (NRF-NRFFAI1-2019-0002). AB was additionally supported by an Amazon Research Award and a Google South/Southeast Asia Research Award.  SG's work is partially supported by the SERB CRG Award CRG/2022/007985. NVV's work was supported in part by NSF CCF grants 2130608 and 2342244 and a UNL Grand Challenges Catalyst Competition Grant. 

We would like to thank Debojyoti Dey, a Ph.D. student at IIT Kanpur, for discussions regarding robust learning algorithms in high dimensions.  AB would also like to thank Daniel Beaglehole for a short meeting which seeded the idea for this work.

\end{ack}

\bibliography{reference}

\end{document}